\documentclass{article}

\usepackage{times}
\usepackage{graphicx} 
\usepackage{subfigure}

\usepackage{natbib}

\usepackage{algorithm}
\usepackage{algorithmic}
\usepackage[algo2e,linesnumbered,lined,boxed,vlined]{algorithm2e}

\usepackage{amsfonts}
\usepackage{amssymb}

\usepackage{hyperref}

\usepackage[accepted]{icml2016}

\usepackage{comment}
\usepackage{amsthm}
\usepackage{mathtools}
\usepackage{tabularx}
\usepackage{multirow}

\newtheorem{lemma}{Lemma}
\newtheorem{proposition}{Proposition}

\newtheorem{definition}{Definition}

\newtheorem{remark}{Remark}

\def\b{\ensuremath\boldsymbol}

\usepackage{lastpage}
\usepackage{fancyhdr}
\usepackage{url}
\icmltitlerunning{Spectral, Probabilistic, and Deep Metric Learning: Tutorial and Survey}

\begin{document}

\AddToShipoutPictureBG*{%
  \AtPageUpperLeft{%
    \setlength\unitlength{1in}%
    \hspace*{\dimexpr0.5\paperwidth\relax}
    \makebox(0,-0.75)[c]{\normalsize {\color{black} To appear as a part of an upcoming textbook on dimensionality reduction and manifold learning.}}
    }}

\twocolumn[
\icmltitle{Spectral, Probabilistic, and Deep Metric Learning: Tutorial and Survey}

\icmlauthor{Benyamin Ghojogh}{bghojogh@uwaterloo.ca}
\icmladdress{Department of Electrical and Computer Engineering, 
\\Machine Learning Laboratory, University of Waterloo, Waterloo, ON, Canada}
\icmlauthor{Ali Ghodsi}{ali.ghodsi@uwaterloo.ca}
\icmladdress{Department of Statistics and Actuarial Science \& David R. Cheriton School of Computer Science, 
\\Data Analytics Laboratory, University of Waterloo, Waterloo, ON, Canada}
\icmlauthor{Fakhri Karray}{karray@uwaterloo.ca}
\icmladdress{Department of Electrical and Computer Engineering, 
\\Centre for Pattern Analysis and Machine Intelligence, University of Waterloo, Waterloo, ON, Canada}
\icmlauthor{Mark Crowley}{mcrowley@uwaterloo.ca}
\icmladdress{Department of Electrical and Computer Engineering, 
\\Machine Learning Laboratory, University of Waterloo, Waterloo, ON, Canada}

\icmlkeywords{Tutorial}

\vskip 0.3in
]

\begin{abstract}
This is a tutorial and survey paper on metric learning. Algorithms are divided into spectral, probabilistic, and deep metric learning. We first start with the definition of distance metric, Mahalanobis distance, and generalized Mahalanobis distance. In spectral methods, we start with methods using scatters of data, including the first spectral metric learning, relevant methods to Fisher discriminant analysis, Relevant Component Analysis (RCA), Discriminant Component Analysis (DCA), and the Fisher-HSIC method. Then, large-margin metric learning, imbalanced metric learning, locally linear metric adaptation, and adversarial metric learning are covered. We also explain several kernel spectral methods for metric learning in the feature space. We also introduce geometric metric learning methods on the Riemannian manifolds. In probabilistic methods, we start with collapsing classes in both input and feature spaces and then explain the neighborhood component analysis methods, Bayesian metric learning, information theoretic methods, and empirical risk minimization in metric learning. In deep learning methods, we first introduce reconstruction autoencoders and supervised loss functions for metric learning. Then, Siamese networks and its various loss functions, triplet mining, and triplet sampling are explained. Deep discriminant analysis methods, based on Fisher discriminant analysis, are also reviewed. Finally, we introduce multi-modal deep metric learning, geometric metric learning by neural networks, and few-shot metric learning. 
\end{abstract}






\section{Introduction}

Dimensionality reduction and manifold learning are used for feature extraction from raw data. 
A family of dimensionality reduction methods is metric learning which learns a distance metric or an embedding space for separation of dissimilar points and closeness of similar points. In supervised metric learning, we aim to discriminate classes by learning an appropriate metric. 
Dimensionality reduction methods can be divided into spectral, probabilistic, and deep methods \cite{ghojogh2021data}. 
Spectral methods have a geometrical approach and usually are reduced to generalized eigenvalue problems \cite{ghojogh2019eigenvalue}. Probabilistic methods are based on probability distributions. Deep methods use neural network for learning.  
In each of these categories, there exist several metric learning methods. In this paper, we review and introduce the most important metric learning algorithms in these categories. 
Note that there exist some other surveys on metric learning such as \cite{yang2006distance,yang2007overview,kulis2013metric,bellet2013survey,wang2015survey,suarez2021tutorial}.
A survey specific to deep metric learning is \cite{kaya2019deep}.
A book on metric learning is \cite{bellet2015metric}.
Finally, some Python toolboxes for metric learning are \cite{suarez2020pydml,de2020metric,musgrave2020pytorch}. 
The remainder of this paper is organized as follows. Section \ref{section_generalized_Mahalanobis_distance} defines distance metric and the generalized Mahalanobis distance. Sections \ref{section_spectral_metric_learning}, \ref{section_probabilistic_metric_learning}, and \ref{section_deep_metric_learning} introduce and discuss spectral, probabilistic, and deep metric learning methods, respectively. Finally, section \ref{section_conclusion} concludes the paper. The table of contents can be found at the end of paper. 

\section*{Required Background for the Reader}

This paper assumes that the reader has general knowledge of calculus, probability, linear algebra, and basics of optimization.

\section{Generalized Mahalanobis Distance Metric}\label{section_generalized_Mahalanobis_distance}

\subsection{Distance Metric}

\begin{definition}[Distance metric]\label{definition_distance_metric}
Consider a metric space $\mathcal{X}$. A distance metric is a mapping $d: \mathcal{X} \times \mathcal{X} \rightarrow [0, \infty)$ which satisfies the following properties:
\begin{itemize}\setlength\itemsep{0em}
\item non-negativity: $d(\b{x}_i, \b{x}_j) \geq 0$
\item identity: $d(\b{x}_i, \b{x}_j) = 0 \iff \b{x}_i = \b{x}_j$
\item symmetry: $d(\b{x}_i, \b{x}_j) = d(\b{x}_j, \b{x}_i)$
\item triangle inequality: $d(\b{x}_i, \b{x}_j) \leq d(\b{x}_i, \b{x}_k) + d(\b{x}_k, \b{x}_j)$
\end{itemize}
where $\b{x}_i, \b{x}_j, \b{x}_k \in \mathcal{X}$. 
\end{definition}
An example of distance metric is the Euclidean distance:
\begin{align}\label{equation_Euclidean_distance}
\|\b{x}_i - \b{x}_j\|_2 := \sqrt{(\b{x}_i - \b{x}_j)^\top (\b{x}_i - \b{x}_j)}.
\end{align}

\subsection{Mahalanobis Distance}

The Mahalanobis distance is another distance metric which was originally proposed in \cite{mahalanobis1930tests}.
\begin{definition}[Mahalanobis distance \cite{mahalanobis1930tests}]
Consider a $d$-dimensional metric space $\mathcal{X}$. 
Let two clouds or sets of points $\mathcal{X}_1$ and $\mathcal{X}_2$ be in the data, i.e., $\mathcal{X}_1, \mathcal{X}_2 \in \mathcal{X}$. 
A point is considered in each set, i.e., $\b{x}_i \in \mathcal{X}_1$ and $\b{x}_j \in \mathcal{X}_2$. The Mahalanobis distance between the two points is:
\begin{align}\label{equation_Mahalanobis_distance}
\|\b{x}_i - \b{x}_j\|_{\b{\Sigma}} := \sqrt{(\b{x}_i - \b{x}_j)^\top \b{\Sigma}^{-1} (\b{x}_i - \b{x}_j)},
\end{align}
where $\b{\Sigma} \in \mathbb{R}^{d \times d}$ is the covariance matrix of data in the two sets $\mathcal{X}_1$ and $\mathcal{X}_2$.

If the points $\b{x}_i$ and $\b{x}_j$ are the means of the sets $\mathcal{X}_1$ and $\mathcal{X}_2$, respectively, as the representatives of the sets, this Mahalanobis distance is a good measure of distance of the sets \cite{mclachlan1999mahalanobis}:
\begin{align}
\|\b{\mu}_1 - \b{\mu}_2\|_{\b{\Sigma}} := \sqrt{(\b{\mu}_1 - \b{\mu}_2)^\top \b{\Sigma}^{-1} (\b{\mu}_1 - \b{\mu}_2)},
\end{align}
where $\b{\mu}_1$ and $\b{\mu}_2$ are the means of the sets $\mathcal{X}_1$ and $\mathcal{X}_2$, respectively. 

Let $\mathcal{X}_1 := \{\b{x}_{1,i}\}_{i=1}^{n_1}$ and $\mathcal{X}_2 := \{\b{x}_{2,i}\}_{i=1}^{n_2}$.
The unbiased sample covariance matrices of these two sets are:
\begin{align*}
\b{\Sigma}_1 := \frac{1}{n_1-1} \sum_{i=1}^{n_1} (\b{x}_{1,i} - \b{\mu}_1) (\b{x}_{1,i} - \b{\mu}_1)^\top,
\end{align*}
and $\b{\Sigma}_2$ similarly. 
The covariance matrix $\b{\Sigma}$ can be an unbiased sample covariance matrix \cite{mclachlan1999mahalanobis}:
\begin{align*}
\b{\Sigma} := \frac{1}{n_1 + n_2 - 2} \Big( (n_1 - 1) \b{\Sigma}_1 + (n_2 - 1) \b{\Sigma}_2 \Big).
\end{align*}

The Mahalanobis distance can also be defined between a point $\b{x}$ and a cloud or set of points $\mathcal{X}$ \cite{de2000mahalanobis}. Let $\b{\mu}$ and $\b{\Sigma}$ be the mean and the (sample) covariance matrix of the set $\mathcal{X}$. The Mahalanobis distance of $\b{x}$ and $\mathcal{X}$ is:
\begin{align}
\|\b{x} - \b{\mu}\|_{\b{\Sigma}} := \sqrt{(\b{x} - \b{\mu})^\top \b{\Sigma}^{-1} (\b{x} - \b{\mu})}.
\end{align}
\end{definition}

\begin{figure}[!t]
\centering
\includegraphics[width=3in]{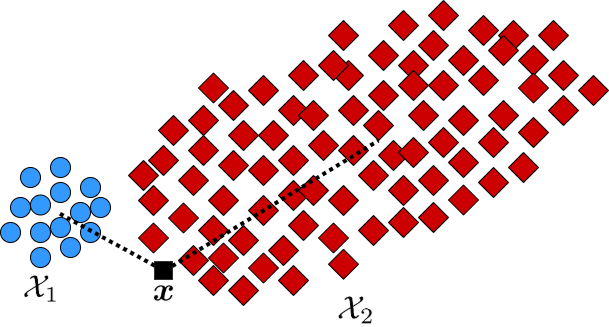}
\caption{An example for comparison of the Euclidean and Mahalanobis distances.}
\label{figure_Mahalanobis}
\end{figure}

\begin{remark}[Justification of the Mahalanobis distance \cite{de2000mahalanobis}]
Consider two clouds of data, $\mathcal{X}_1$ and $\mathcal{X}_2$, depicted in Fig. \ref{figure_Mahalanobis}. 
We want to compute the distance of a point $\b{x}$ from these two data clouds to see which cloud this point is closer to. The Euclidean distance ignores the scatter/variance of clouds and only measures the distances of the point from the means of clouds. Hence, in this example, it says that $\b{x}$ belongs to $\mathcal{X}_1$ because it is closer to the mean of $\mathcal{X}_1$ compared to $\mathcal{X}_2$. However, the Mahalanobis distance takes the variance of clouds into account and says that $\b{x}$ belongs to $\mathcal{X}_2$ because it is closer to its scatter compared to $\mathcal{X}_1$. Visually, human also says $\b{x}$ belongs to $\mathcal{X}_2$; hence, the Mahalanobis distance has performed better than the Euclidean distance by considering the variances of data.
\end{remark}

\subsection{Generalized Mahalanobis Distance}

\begin{definition}[Generalized Mahalanobis distance]
In Mahalanobis distance, i.e. Eq. (\ref{equation_Mahalanobis_distance}), the covariance matrix $\b{\Sigma}$ and its inverse $\b{\Sigma}^{-1}$ are positive semi-definite. We can replace $\b{\Sigma}^{-1}$ with a positive semi-definite weight matrix $\b{W} \succeq \b{0}$ in the squared Mahalanobis distance. We name this distance a generalized Mahalanobis distance:
\begin{equation}\label{equation_generalized_Mahalanobis_distance}
\begin{aligned}
& \|\b{x}_i - \b{x}_j\|_{\b{W}} := \sqrt{(\b{x}_i - \b{x}_j)^\top \b{W} (\b{x}_i - \b{x}_j)}. \\
&\therefore\,\,\,\, \|\b{x}_i - \b{x}_j\|_{\b{W}}^2 := (\b{x}_i - \b{x}_j)^\top \b{W} (\b{x}_i - \b{x}_j).
\end{aligned}
\end{equation}
We define the generalized Mahalanobis norm as:
\begin{align}
\|\b{x}\|_{\b{W}} := \sqrt{\b{x}^\top \b{W} \b{x}}.
\end{align}
\end{definition}

\begin{lemma}[Triangle inequality of norm]
Let $\|.\|$ be a norm. Using the Cauchy-Schwarz inequality, it satisfies the triangle inequality:
\begin{align}\label{equation_xi_xj_triangle_inequality}
\|\b{x}_i + \b{x}_j\| \leq \|\b{x}_i\| + \|\b{x}_j\|.
\end{align}
\end{lemma}
\begin{proof}
\begin{align*}
\|\b{x}_i + \b{x}_j\|^2 &= (\b{x}_i + \b{x}_j)^\top (\b{x}_i + \b{x}_j) \\
&= \|\b{x}_i\|^2 + \|\b{x}_j\|^2 + 2\b{x}_i^\top \b{x}_j \\
&\overset{(a)}{\leq} \|\b{x}_i\|^2 + \|\b{x}_j\|^2 + 2\|\b{x}_i\| \|\b{x}_j\| \\
&= (\|\b{x}_i\| + \|\b{x}_j\|)^2,
\end{align*}
where $(a)$ is because of the Cauchy-Schwarz inequality, i.e., $\b{x}_i^\top \b{x}_j \leq \|\b{x}_i\| \|\b{x}_j\|$.
Taking second root from the sides gives Eq. (\ref{equation_xi_xj_triangle_inequality}). Q.E.D.
\end{proof}

\begin{proposition}
The generalized Mahalanobis distance is a valid distance metric. 
\end{proposition}
\begin{proof}
We show that the characteristics in Definition \ref{definition_distance_metric} are satisfied:
\begin{itemize}
\item As $\b{W} \succeq \b{0}$, Eq. (\ref{equation_generalized_Mahalanobis_distance}) is non-negative. 
\item identity: if $\|\b{x}_i - \b{x}_j\|_{\b{W}} = 0$, according to Eq. (\ref{equation_generalized_Mahalanobis_distance}), we have $\b{x}_i - \b{x}_j = 0 \implies \b{x}_i = \b{x}_j$. If $\b{x}_i = \b{x}_j$, we have $\|\b{x}_i - \b{x}_j\|_{\b{W}} = 0$ according to Eq. (\ref{equation_generalized_Mahalanobis_distance}).
\item symmetry: \\$\|\b{x}_i - \b{x}_j\|_{\b{W}} = \sqrt{(\b{x}_i - \b{x}_j)^\top \b{W} (\b{x}_i - \b{x}_j)} = \sqrt{(\b{x}_j - \b{x}_i)^\top \b{W} (\b{x}_j - \b{x}_i)} = \|\b{x}_j - \b{x}_i\|_{\b{W}}$.
\item triangle inequality: $\|\b{x}_i - \b{x}_j\|_{\b{W}} = \|\b{x}_i - \b{x}_k + \b{x}_k - \b{x}_j\|_{\b{W}} \overset{(\ref{equation_xi_xj_triangle_inequality})}{\leq} \|\b{x}_i - \b{x}_k\|_{\b{W}} + \|\b{x}_k - \b{x}_j\|_{\b{W}}$.
\end{itemize}
\end{proof}

\begin{remark}
It is noteworthy that $\b{W} \succeq \b{0}$ is required so that the generalized Mahalanobis distance is convex and satisfies the triangle inequality.
\end{remark}


\begin{remark}
The weight matrix $\b{W}$ in Eq. (\ref{equation_generalized_Mahalanobis_distance}) weights the dimensions and determines some correlation between dimensions of data points. In other words, it changes the space in a way that the scatters of clouds are considered. 
\end{remark}

\begin{remark}
The Euclidean distance is a special case of the Mahalanobis distance where the weight matrix is the identity matrix, i.e., $\b{W} = \b{I}$ (cf. Eqs. (\ref{equation_Euclidean_distance}) and (\ref{equation_generalized_Mahalanobis_distance})). In other words, the Euclidean distance does not change the space for computing the distance. 
\end{remark}

\begin{proposition}[Projection in metric learning]\label{proposition_metric_learning_projection}
Consider the eigenvalue decomposition of the weight matrix $\b{W}$ in the generalized Mahalanobis distance with $\b{V}$ and $\b{\Lambda}$ as the matrix of eigenvectors and the diagonal matrix of eigenvalues of the weight, respectively.  
Let $\b{U} := \b{V} \b{\Lambda}^{(1/2)}$.
The generalized Mahalanobis distance can be seen as the Euclidean distance after applying a linear projection onto the column space of $\b{U}$: 
\begin{equation}\label{equation_metric_learning_projection}
\begin{aligned}
\|\b{x}_i - \b{x}_j\|_{\b{W}}^2 &= (\b{U}^\top \b{x}_i - \b{U}^\top \b{x}_j)^\top (\b{U}^\top \b{x}_i - \b{U}^\top \b{x}_j) \\
&= \|\b{U}^\top \b{x}_i - \b{U}^\top \b{x}_j\|_2^2. 
\end{aligned}
\end{equation}
\end{proposition}
If $\b{U} \in \mathbb{R}^{d \times p}$ with $p \leq d$, the column space of the projection matrix $\b{U}$ is a $p$-dimensional subspace. 
\begin{proof}
By the eigenvalue decomposition of $\b{W}$, we have:
\begin{align}\label{equation_W_U_UT}
\b{W} = \b{V} \b{\Lambda} \b{V}^\top \overset{(a)}{=} \b{V} \b{\Lambda}^{(1/2)} \b{\Lambda}^{(1/2)} \b{V}^\top \overset{(b)}{=} \b{U} \b{U}^\top,
\end{align}
where $(a)$ is because $\b{W}$ is positive semi-definite so all its eigenvalues are non-negative and can be written as multiplication of its second roots. Also, $(b)$ is because we define $\b{U} := \b{V} \b{\Lambda}^{(1/2)}$. 
Substituting Eq. (\ref{equation_W_U_UT}) in Eq. (\ref{equation_generalized_Mahalanobis_distance}) gives:
\begin{align*}
\|\b{x}_i - \b{x}_j\|_{\b{W}}^2 &= (\b{x}_i - \b{x}_j)^\top \b{U} \b{U}^\top (\b{x}_i - \b{x}_j) \\
&= (\b{U}^\top \b{x}_i - \b{U}^\top \b{x}_j)^\top (\b{U}^\top \b{x}_i - \b{U}^\top \b{x}_j) \\
&= \|\b{U}^\top \b{x}_i - \b{U}^\top \b{x}_j\|_2^2. 
\end{align*}
Q.E.D. 
It is noteworthy that Eq. (\ref{equation_W_U_UT}) can also be obtained using singular value decomposition rather than eigenvalue decomposition. In that case, the matrices of right and left singular vectors are equal because of symmetry of $\b{W}$. 
\end{proof}

\subsection{The Main Idea of Metric Learning}

Consider a $d$-dimensional dataset $\{\b{x}_i\}_{i=1}^n \subset \mathbb{R}^d$ of size $n$. 
Assume some data points are similar in some sense. For example, they have similar pattern or the same characteristics. Hence, we have a set of similar pair points, denotes by $\mathcal{S}$. In contrast, we can have dissimilar points which are different in pattern or characteristics. Let the set of dissimilar pair points be denoted by $\mathcal{D}$. In summary:
\begin{equation}
\begin{aligned}
& (\b{x}_i, \b{x}_j) \in \mathcal{S} \text{ if } \b{x}_i \text{ and } \b{x}_j \text{ are similar}, \\
& (\b{x}_i, \b{x}_j) \in \mathcal{D} \text{ if } \b{x}_i \text{ and } \b{x}_j \text{ are dissimilar}.
\end{aligned}
\end{equation}
The measure of similarity and dissimilarity can be belonging to the same or different classes, if class labels are available for dataset. In this case, we have:
\begin{equation}
\begin{aligned}
& (\b{x}_i, \b{x}_j) \in \mathcal{S} \text{ if } \b{x}_i \text{ and } \b{x}_j \text{ are in the same class}, \\
& (\b{x}_i, \b{x}_j) \in \mathcal{D} \text{ if } \b{x}_i \text{ and } \b{x}_j \text{ are in different classes}.
\end{aligned}
\end{equation}
In metric learning, we learn the weight matrix so that the distances of similar points become smaller and the distances of dissimilar points become larger. 
In this way, the variance of similar and dissimilar points get smaller and larger, respectively. A 2D visualization of metric learning is depicted in Fig. \ref{figure_Metric_learning}.
If the class labels are available, metric learning tries to make the intra-class and inter-class variances smaller and larger, respectively. This is the same idea as the idea of Fisher Discriminant Analysis (FDA) \cite{fisher1936use,ghojogh2019fisher}.

\begin{figure}[!t]
\centering
\includegraphics[width=3in]{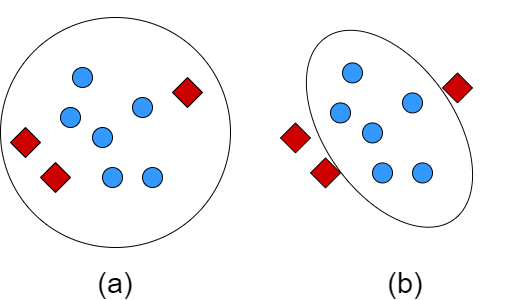}
\caption{Visualizing metric learning in 2D: (a) the contour of Euclidean distance which does not properly discriminate classes, and (b) the contour of Euclidean distance which is better in discrimination of classes.}
\label{figure_Metric_learning}
\end{figure}

\section{Spectral Metric Learning}\label{section_spectral_metric_learning}

\subsection{Spectral Methods Using Scatters}

\subsubsection{The First Spectral Method}

The first metric learning method was proposed in \cite{xing2002distance}. In this method, we minimize the distances of the similar points by the weight matrix $\b{W}$ where this matrix is positive semi-definite:
\begin{equation*}
\begin{aligned}
& \underset{\b{W}}{\text{minimize}}
& & \sum_{(\b{x}_i, \b{x}_j) \in \mathcal{S}} \|\b{x}_i - \b{x}_j\|_{\b{W}}^2 \\
& \text{subject to}
& & \b{W} \succeq \b{0}. 
\end{aligned}
\end{equation*}
However, the solution of this optimization problem is trivial, i.e., $\b{W} = \b{0}$. Hence, we add a constraint on the dissimilar points to have distances larger than some positive amount:
\begin{equation}\label{equation_ML_optimization_spectral_first_method}
\begin{aligned}
& \underset{\b{W}}{\text{minimize}}
& & \sum_{(\b{x}_i, \b{x}_j) \in \mathcal{S}} \|\b{x}_i - \b{x}_j\|_{\b{W}}^2 \\
& \text{subject to}
& & \sum_{(\b{x}_i, \b{x}_j) \in \mathcal{D}} \|\b{x}_i - \b{x}_j\|_{\b{W}} \geq \alpha, \\
& & & \b{W} \succeq \b{0}. 
\end{aligned}
\end{equation}
where $\alpha > 0$ is some positive number such as $\alpha = 1$. 

\begin{lemma}[\cite{xing2002distance}]
If the constraint in Eq. (\ref{equation_ML_optimization_spectral_first_method}) is squared, i.e., $\sum_{(\b{x}_i, \b{x}_j) \in \mathcal{D}} \|\b{x}_i - \b{x}_j\|_{\b{W}}^2 \geq \alpha$, the solution of optimization will have rank $1$. Hence, we are using a non-squared constraint in the optimization problem. 
\end{lemma}
\begin{proof}
If the constraint in Eq. (\ref{equation_ML_optimization_spectral_first_method}) is squared, the problem is equivalent to (see {\citep[Appendix B]{ghojogh2019fisher}} for proof):
\begin{equation*}
\begin{aligned}
& \underset{\b{W}}{\text{maximize}}
& & \frac{\sum_{(\b{x}_i, \b{x}_j) \in \mathcal{D}} \|\b{x}_i - \b{x}_j\|_{\b{W}}^2}{\sum_{(\b{x}_i, \b{x}_j) \in \mathcal{S}} \|\b{x}_i - \b{x}_j\|_{\b{W}}^2},
\end{aligned}
\end{equation*}
which is a Rayleigh-Ritz quotient \cite{ghojogh2019eigenvalue}. 
We can restate $\|\b{x}_i - \b{x}_j\|_{\b{W}}^2$ as:
\begin{equation}\label{equation_spectral_ML_first_method_trace_W_Sigma_S}
\begin{aligned}
\sum_{(\b{x}_i, \b{x}_j) \in \mathcal{S}} \|\b{x}_i - \b{x}_j\|_{\b{W}}^2 = \textbf{tr}(\b{W} \b{\Sigma}_{\mathcal{S}}), \\
\sum_{(\b{x}_i, \b{x}_j) \in \mathcal{D}} \|\b{x}_i - \b{x}_j\|_{\b{W}}^2 = \textbf{tr}(\b{W} \b{\Sigma}_{\mathcal{D}}),
\end{aligned}
\end{equation}
where $\textbf{tr}(.)$ denotes the trace of matrix and:
\begin{equation}\label{equation_spectral_ML_first_method_Sigma_S}
\begin{aligned}
\b{\Sigma}_{\mathcal{S}} := \sum_{(\b{x}_i, \b{x}_j) \in \mathcal{S}} (\b{x}_i - \b{x}_j) (\b{x}_i - \b{x}_j)^\top, \\
\b{\Sigma}_{\mathcal{D}} := \sum_{(\b{x}_i, \b{x}_j) \in \mathcal{D}} (\b{x}_i - \b{x}_j) (\b{x}_i - \b{x}_j)^\top.
\end{aligned} 
\end{equation}
Hence, we have:
\begin{equation*}
\begin{aligned}
& \frac{\sum_{(\b{x}_i, \b{x}_j) \in \mathcal{D}} \|\b{x}_i - \b{x}_j\|_{\b{W}}^2}{\sum_{(\b{x}_i, \b{x}_j) \in \mathcal{S}} \|\b{x}_i - \b{x}_j\|_{\b{W}}^2} = \frac{\textbf{tr}(\b{W} \b{\Sigma}_{\mathcal{D}})}{\textbf{tr}(\b{W} \b{\Sigma}_{\mathcal{S}})} \overset{(\ref{equation_W_U_UT})}{=} \frac{\textbf{tr}(\b{U} \b{U}^\top \b{\Sigma}_{\mathcal{D}})}{\textbf{tr}(\b{U} \b{U}^\top \b{\Sigma}_{\mathcal{S}})} \\
& \overset{(a)}{=} \frac{\textbf{tr}(\b{U}^\top \b{\Sigma}_{\mathcal{D}} \b{U})}{\textbf{tr}(\b{U}^\top \b{\Sigma}_{\mathcal{S}} \b{U})} = \frac{\sum_{i=1}^d \b{u}^\top \b{\Sigma}_{\mathcal{D}} \b{u}}{\sum_{i=1}^d \b{u}^\top \b{\Sigma}_{\mathcal{S}} \b{u}}, 
\end{aligned}
\end{equation*}
where $(a)$ is because of the cyclic property of trace and $(b)$ is because $\b{U} = [\b{u}_1, \dots, \b{u}_d]$. 
Maximizing this Rayleigh-Ritz quotient results in the following generalized eigenvalue problem \cite{ghojogh2019eigenvalue}: 
\begin{align*}
\b{\Sigma}_{\mathcal{D}} \b{u}_1 = \lambda \b{\Sigma}_{\mathcal{S}} \b{u}_1,
\end{align*}
where $\b{u}_1$ is the eigenvector with largest eigenvalue and the other eigenvectors $\b{u}_2, \dots, \b{u}_d$ are zero vectors. Q.E.D.
\end{proof}

The Eq. (\ref{equation_ML_optimization_spectral_first_method}) can be restated as a maximization problem:
\begin{equation}\label{equation_ML_optimization_spectral_first_method_2}
\begin{aligned}
& \underset{\b{W}}{\text{maximize}}
& & \sum_{(\b{x}_i, \b{x}_j) \in \mathcal{D}} \|\b{x}_i - \b{x}_j\|_{\b{W}} \\
& \text{subject to}
& & \sum_{(\b{x}_i, \b{x}_j) \in \mathcal{S}} \|\b{x}_i - \b{x}_j\|_{\b{W}}^2 \leq \alpha, \\
& & & \b{W} \succeq \b{0}. 
\end{aligned}
\end{equation}
We can solve this problem using projected gradient method \cite{ghojogh2021kkt} where a step of gradient ascent is followed by projection onto the two constraint sets:
\begin{align*}
& \b{W} := \b{W} + \eta \frac{\partial}{\partial \b{W}} \Big( \sum_{(\b{x}_i, \b{x}_j) \in \mathcal{D}} \|\b{x}_i - \b{x}_j\|_{\b{W}} \Big), \\
& \b{W} := \arg \min_{\b{Q}} \Big(\|\b{Q} - \b{W}\|_F^2 \,\text{ s.t.} \\
&~~~~~~~~~~~~~~~~~~~~~~~~~~~ \sum_{(\b{x}_i, \b{x}_j) \in \mathcal{S}} \|\b{x}_i - \b{x}_j\|_{\b{Q}}^2 \leq \alpha\Big), \\
& \b{W} := \b{V}\, \textbf{diag}(\max(\lambda_1, 0), \dots, \max(\lambda_d, 0))\, \b{V}^\top,
\end{align*}
where $\eta>0$ is the learning rate and $\b{V}$ and $\b{\Lambda} = \textbf{diag}(\lambda_1, \dots, \lambda_d)$ are the eigenvectors and eigenvalues of $\b{W}$, respectively (see Eq. (\ref{equation_W_U_UT})). 

\subsubsection{Formulating as Semidefinite Programming}

Another metric learning method is \cite{ghodsi2007improving} which minimizes the distances of similar points and maximizes the distances of dissimilar points. For this, we minimize the distances of similar points and the negation of distances of dissimilar points.  
The weight matrix should be positive semi-definite to satisfy the triangle inequality and convexity. The trace of weight matrix is also set to a constant to eliminate the trivial solution $\b{W} = \b{0}$. 
The optimization problem is:
\begin{equation}\label{equation_ML_optimization_spectral_method2}
\begin{aligned}
& \underset{\b{W}}{\text{minimize}}
& & \frac{1}{|\mathcal{S}|} \sum_{(\b{x}_i, \b{x}_j) \in \mathcal{S}} \|\b{x}_i - \b{x}_j\|_{\b{W}}^2 \\
& & & - \frac{1}{|\mathcal{D}|} \sum_{(\b{x}_i, \b{x}_j) \in \mathcal{D}} \|\b{x}_i - \b{x}_j\|_{\b{W}}^2 \\
& \text{subject to} 
& & \b{W} \succeq \b{0}, \\
& & & \textbf{tr}(\b{W}) = 1,
\end{aligned}
\end{equation}
where $|.|$ denotes the cardinality of set.

\begin{lemma}[\cite{ghodsi2007improving}]\label{lemma_sepctral_ML_method2_objective}
The objective function can be simplified as:
\begin{equation}
\begin{aligned}
&\frac{1}{|\mathcal{S}|} \sum_{(\b{x}_i, \b{x}_j) \in \mathcal{S}} \|\b{x}_i - \b{x}_j\|_{\b{W}}^2 - \frac{1}{|\mathcal{D}|} \sum_{(\b{x}_i, \b{x}_j) \in \mathcal{D}} \|\b{x}_i - \b{x}_j\|_{\b{W}} \\
&~~ = \textbf{vec}(\b{W})^\top 
\Big( \frac{1}{|\mathcal{S}|} \sum_{(\b{x}_i, \b{x}_j) \in \mathcal{S}} \textbf{vec}\big((\b{x}_i - \b{x}_j) (\b{x}_i - \b{x}_j)^\top\big) \\
&~~~~~~~~~~~~~~~ - \frac{1}{|\mathcal{D}|} \sum_{(\b{x}_i, \b{x}_j) \in \mathcal{D}} \textbf{vec}\big((\b{x}_i - \b{x}_j) (\b{x}_i - \b{x}_j)^\top\big) \Big),
\end{aligned}
\end{equation}
where $\textbf{vec}(.)$ vectorizes the matrix to a vector \cite{ghojogh2021kkt}. 
\end{lemma}
\begin{proof}
See {\citep[Section 2.1]{ghodsi2007improving}} for proof. 
\end{proof}
According to Lemma \ref{lemma_sepctral_ML_method2_objective}, Eq. (\ref{equation_ML_optimization_spectral_method2}) is a Semidefinite Programming (SDP) problem. It can be solved iteratively using the interior-point method \cite{ghojogh2021kkt}. 

\subsubsection{Relevant to Fisher Discriminant Analysis}\label{section_relation_to_FDA}

Another metric learning method is \cite{alipanahi2008distance} which has two approaches, introduced in the following. The relation of metric learning with Fisher discriminant analysis \cite{fisher1936use,ghojogh2019fisher} was discussed in this paper \cite{alipanahi2008distance}. 

\hfill\break
\textbf{-- Approach 1: }
As $\b{W} \succeq \b{0}$, the weight matrix can be decomposed as in Eq. (\ref{equation_W_U_UT}), i.e., $\b{W} = \b{U} \b{U}^\top$. Hence, we have:
\begin{align}
\|\b{x}_i - \b{x}_j\|_{\b{W}}^2 &\overset{(\ref{equation_generalized_Mahalanobis_distance})}{=} (\b{x}_i - \b{x}_j)^\top \b{W} (\b{x}_i - \b{x}_j) \nonumber\\
&\overset{(a)}{=} \textbf{tr}\big((\b{x}_i - \b{x}_j)^\top \b{W} (\b{x}_i - \b{x}_j)\big) \nonumber\\
&\overset{(\ref{equation_W_U_UT})}{=} \textbf{tr}\big((\b{x}_i - \b{x}_j)^\top \b{U} \b{U}^\top (\b{x}_i - \b{x}_j)\big) \nonumber\\
&\overset{(b)}{=} \textbf{tr}\big(\b{U}^\top (\b{x}_i - \b{x}_j) (\b{x}_i - \b{x}_j)^\top \b{U}\big), \label{equation_ML_spectral_method3_trace_UT_x_xT_U}
\end{align}
where $(a)$ is because a scalar is equal to its trace and $(b)$ is because of the cyclic property of trace. 
We can substitute Eq. (\ref{equation_ML_spectral_method3_trace_UT_x_xT_U}) in Eq. (\ref{equation_ML_optimization_spectral_method2}) to obtain an optimization problem:
\begin{equation}\label{equation_ML_optimization_spectral_method3}
\begin{aligned}
& \underset{\b{U}}{\text{minimize}}
& & \frac{1}{|\mathcal{S}|} \sum_{(\b{x}_i, \b{x}_j) \in \mathcal{S}} \textbf{tr}\big(\b{U}^\top (\b{x}_i - \b{x}_j) (\b{x}_i - \b{x}_j)^\top \b{U}\big) \\
& & & \!\!\!\!\!\!\!- \frac{1}{|\mathcal{D}|} \sum_{(\b{x}_i, \b{x}_j) \in \mathcal{D}} \textbf{tr}\big(\b{U}^\top (\b{x}_i - \b{x}_j) (\b{x}_i - \b{x}_j)^\top \b{U}\big) \\
& \text{subject to} 
& & \textbf{tr}(\b{U} \b{U}^\top) = 1,
\end{aligned}
\end{equation}
whose objective variable is $\b{U}$. Note that the constraint $\b{W} \succeq \b{0}$ is implicitly satisfied because of the decomposition $\b{W} = \b{U} \b{U}^\top$.
We define:
\begin{equation}\label{equation_Sigma_S_prime_Sigma_D_prime}
\begin{aligned}
& \b{\Sigma}'_{\mathcal{S}} := \frac{1}{|\mathcal{S}|} \sum_{(\b{x}_i, \b{x}_j) \in \mathcal{S}} (\b{x}_i - \b{x}_j) (\b{x}_i - \b{x}_j)^\top \overset{(\ref{equation_spectral_ML_first_method_trace_W_Sigma_S})}{=} \frac{1}{|\mathcal{S}|} \b{\Sigma}_{\mathcal{S}}, \\
& \b{\Sigma}'_{\mathcal{D}} := \frac{1}{|\mathcal{D}|} \sum_{(\b{x}_i, \b{x}_j) \in \mathcal{D}} (\b{x}_i - \b{x}_j) (\b{x}_i - \b{x}_j)^\top \overset{(\ref{equation_spectral_ML_first_method_trace_W_Sigma_S})}{=} \frac{1}{|\mathcal{D}|} \b{\Sigma}_{\mathcal{D}}.
\end{aligned}
\end{equation}
Hence, Eq. (\ref{equation_ML_optimization_spectral_method3}) can be restated as:
\begin{equation}\label{equation_ML_optimization_spectral_method3_2}
\begin{aligned}
& \underset{\b{U}}{\text{minimize}}
& & \textbf{tr}(\b{U}^\top (\b{\Sigma}'_{\mathcal{S}} - \b{\Sigma}'_{\mathcal{D}}) \b{U}) \\
& \text{subject to} 
& & \textbf{tr}(\b{U} \b{U}^\top) = 1,
\end{aligned}
\end{equation}
whose Lagrangian is \cite{ghojogh2021kkt}:
\begin{align*}
\mathcal{L} = \textbf{tr}(\b{U}^\top (\b{\Sigma}'_{\mathcal{S}} - \b{\Sigma}'_{\mathcal{D}}) \b{U}) - \lambda (\textbf{tr}(\b{U} \b{U}^\top) - 1).
\end{align*}
Taking derivative of the Lagrangian and setting it to zero gives:
\begin{align}
&\frac{\partial \mathcal{L}}{\partial \b{U}} = 2(\b{\Sigma}'_{\mathcal{S}} - \b{\Sigma}'_{\mathcal{D}}) \b{U} - 2\lambda \b{U} \overset{\text{set}}{=} \b{0} \nonumber\\
&\implies (\b{\Sigma}'_{\mathcal{S}} - \b{\Sigma}'_{\mathcal{D}}) \b{U} = \lambda \b{U}, \label{equation_spectral_ML_method3_eig_problem}
\end{align}
which is the eigenvalue problem for $(\b{\Sigma}'_{\mathcal{S}} - \b{\Sigma}'_{\mathcal{D}})$ \cite{ghojogh2019eigenvalue}. Hence, $\b{U}$ is the eigenvector of $(\b{\Sigma}'_{\mathcal{S}} - \b{\Sigma}'_{\mathcal{D}})$ with the smallest eigenvalue because Eq. (\ref{equation_ML_optimization_spectral_method3}) is a minimization problem. 

\hfill\break
\textbf{-- Approach 2: }
We can change the constraint in Eq. (\ref{equation_ML_optimization_spectral_method3_2}) to have orthogonal projection matrix, i.e., $\b{U}^\top \b{U} = \b{I}$.
Rather, we can make the rotation of the projection matrix by the matrix $\b{\Sigma}'_{\mathcal{S}}$ be orthogonal, i.e., $\b{U}^\top \b{\Sigma}'_{\mathcal{S}} \b{U} = \b{I}$. Hence, the optimization problem becomes:
\begin{equation}\label{equation_ML_optimization_spectral_method3_approach2}
\begin{aligned}
& \underset{\b{U}}{\text{minimize}}
& & \textbf{tr}(\b{U}^\top (\b{\Sigma}'_{\mathcal{S}} - \b{\Sigma}'_{\mathcal{D}}) \b{U}) \\
& \text{subject to} 
& & \b{U}^\top \b{\Sigma}'_{\mathcal{S}}\, \b{U} = \b{I},
\end{aligned}
\end{equation}
whose Lagrangian is \cite{ghojogh2021kkt}:
\begin{align}
&\mathcal{L} = \textbf{tr}(\b{U}^\top (\b{\Sigma}'_{\mathcal{S}} - \b{\Sigma}'_{\mathcal{D}}) \b{U}) - \textbf{tr}(\b{\Lambda}^\top (\b{U}^\top \b{\Sigma}'_{\mathcal{S}}\, \b{U} - \b{I})). \nonumber\\
&\frac{\partial \mathcal{L}}{\partial \b{U}} = 2(\b{\Sigma}'_{\mathcal{S}} - \b{\Sigma}'_{\mathcal{D}}) \b{U} - 2 \b{\Sigma}'_{\mathcal{S}}\,  \b{U} \b{\Lambda} \overset{\text{set}}{=} \b{0} \nonumber\\
&\implies (\b{\Sigma}'_{\mathcal{S}} - \b{\Sigma}'_{\mathcal{D}}) \b{U} = \b{\Sigma}'_{\mathcal{S}}\,  \b{U} \b{\Lambda}, 
\end{align}
which is the generalized eigenvalue problem for $(\b{\Sigma}'_{\mathcal{S}} - \b{\Sigma}'_{\mathcal{D}}, \b{\Sigma}'_{\mathcal{S}})$ \cite{ghojogh2019eigenvalue}. Hence, $\b{U}$ is a matrix whose columns are the eigenvectors sorted from the smallest to largest eigenvalues. 

The optimization problem is similar to the optimization of Fisher discriminant analysis (FDA) \cite{fisher1936use,ghojogh2019fisher} where $\b{\Sigma}'_{\mathcal{S}}$ and $\b{\Sigma}'_{\mathcal{D}}$ are replaced with the intra-class and inter-class covariance matrices of data, respectively. This shows the relation of this method with FDA. It makes sense because both metric learning and FDA have the same goal and that is decreasing and increasing the variances of similar and dissimilar points, respectively. 

\subsubsection{Relevant Component Analysis (RCA)}\label{section_RCA}


Suppose the $n$ data points can be divided into $c$ clusters, or so-called chunklets. If class labels are available, classes are the chunklets. If $\mathcal{X}_l$ denotes the data of the $l$-th cluster and $\b{\mu}_l$ is the mean of $\mathcal{X}_l$, the summation of intra-cluster scatters is:
\begin{align}\label{equation_intra_cluster_scatter}
\mathbb{R}^{d \times d} \ni \b{S}_w := \frac{1}{n} \sum_{l=1}^c \sum_{\b{x}_i \in \mathcal{X}_l} (\b{x}_i - \b{\mu}_l) (\b{x}_i - \b{\mu}_l)^\top. 
\end{align}
Relevant Component Analysis (RCA) \cite{shental2002adjustment} is a metric learning method. In this method, we first apply Principal Component Analysis (PCA) \cite{ghojogh2019unsupervised} on data using the total scatter of data. Let the projection matrix of PCA be denoted by $\b{U}$. After projection onto the PCA subspace, the summation of intra-cluster scatters is $\widehat{\b{S}}_w := \b{U}^\top \b{S}_w \b{U}$ because of the quadratic characteristic of covariance. 
RCA uses $\widehat{\b{S}}_w$ as the covariance matrix in the Mahalanobis distance, i.e., Eq. (\ref{equation_Mahalanobis_distance}). According to Eq. (\ref{equation_metric_learning_projection}), the subspace of RDA is obtained by the eigenvalue (or singular value) decomposition of $\widehat{\b{S}}_w^{-1}$ (see Eq. (\ref{equation_W_U_UT})). 

\subsubsection{Discriminative Component Analysis (DCA)}\label{section_DCA}

Discriminative Component Analysis (DCA) \cite{hoi2006learning} is another spectral metric learning method based on scatters of clusters/classes. 
Consider the $c$ clusters, chunklets, or classes of data. The intra-class scatter is as in Eq. (\ref{equation_intra_cluster_scatter}). The inter-class scatter is:
\begin{equation}\label{equation_inter_cluster_scatter}
\begin{aligned}
&\mathbb{R}^{d \times d} \ni \b{S}_b := \frac{1}{n} \sum_{l=1}^c \sum_{j=1}^c (\b{\mu}_l - \b{\mu}_j) (\b{\mu}_l - \b{\mu}_j)^\top, \text{ or } \\
&\mathbb{R}^{d \times d} \ni \b{S}_b := \frac{1}{n} \sum_{l=1}^c (\b{\mu}_l - \b{\mu}) (\b{\mu}_l - \b{\mu})^\top,
\end{aligned}
\end{equation}
where $\b{\mu}_l$ is the mean of the $l$-th class and $\b{\mu}$ is the total mean of data.
According to Proposition \ref{proposition_metric_learning_projection}, metric learning can be seen as Euclidean distance after projection onto the column space of a projection matrix $\b{U}$ where $\b{W} = \b{U} \b{U}^\top$. 
Similar to Fisher discriminant analysis \cite{fisher1936use,ghojogh2019fisher}, DCA maximizes the inter-class variance and minimizes the intra-class variance after projection. Hence, its optimization is:
\begin{equation}\label{equation_optimization_DCA}
\begin{aligned}
& \underset{\b{U}}{\text{maximize}}
& & \frac{\textbf{tr}(\b{U}^\top \b{S}_b \b{U})}{\textbf{tr}(\b{U}^\top \b{S}_w \b{U})},
\end{aligned}
\end{equation}
which is a generalized Rayleigh-Ritz quotient. The solution $\b{U}$ to this optimization problem is the generalized eigenvalue problem $(\b{S}_b, \b{S}_w)$ \cite{ghojogh2019eigenvalue}. 
According to Eq. (\ref{equation_W_U_UT}), we can set the weight matrix of the generalized Mahalanobis distance as $\b{W} = \b{U} \b{U}^\top$ where $\b{U}$ is the matrix of eigenvectors. 

\subsubsection{High Dimensional Discriminative Component Analysis}

Another spectral method for metric learning is \cite{xiang2008learning} which minimizes and maximizes the intra-class and inter-class variances, respectively, by the the same optimization problem as Eq. (\ref{equation_optimization_DCA}) with an additional constraint on the orthogonality of the projection matrix, i.e., $\b{U}^\top \b{U} = \b{I}$. This problem can be restated by posing penalty on the denominator:
\begin{equation}\label{equation_optimization_Sb_lambda_Sw}
\begin{aligned}
& \underset{\b{U}}{\text{maximize}}
& & \textbf{tr}(\b{U}^\top (\b{S}_b - \lambda \b{S}_w) \b{U}) \\
& \text{subject to} 
& & \b{U}^\top \b{U} = \b{I},
\end{aligned}
\end{equation}
where $\lambda>0$ is the regularization parameter. 
The solution to this problem is the eigenvalue problem for $\b{S}_b - \lambda \b{S}_w$. 
The eigenvectors are the columns of $\b{U}$ and the weight matrix of the generalized Mahalanobis is obtained using Eq. (\ref{equation_W_U_UT}). 

If the dimensionality of data is large, computing the eigenvectors of $(\b{S}_b - \lambda \b{S}_w) \in \mathbb{R}^{d \times d}$ is very time-consuming. 
According to {\citep[Theorem 3]{xiang2008learning}}, the optimization problem (\ref{equation_optimization_Sb_lambda_Sw}) can be solved in the orthogonal complement space of the null space of $\b{S}_b + \b{S}_w$ without loss of any information (see {\citep[Appendix A]{xiang2008learning}} for proof). 
Hence, if $d \gg 1$, we find $\b{U}$ as follows. 
Let $\b{X} := [\b{x}_1, \dots, \b{x}_n] \in \mathbb{R}^{d \times n}$ be the matrix of data. 
Let $\b{A}_w$ and $\b{A}_b$ be the adjacency matrices for the sets $\mathcal{S}$ and $\mathcal{D}$, respectively. For example, if $(\b{x}_i, \b{x}_j) \in \mathcal{S}$, then $\b{A}_w(i,j) = 1$; otherwise, $\b{A}_w(i,j) = 0$. 
If $\b{L}_w$ and $\b{L}_b$ are the Laplacian matrices of $\b{A}_w$ and $\b{A}_b$, respectively, we have $\b{S}_w = 0.5 \b{X} \b{L}_w \b{X}^\top$ and $\b{S}_b = 0.5 \b{X} \b{L}_b \b{X}^\top$ (see \cite{belkin2002laplacian,ghojogh2021laplacian} for proof). 
We have $\textbf{tr}(\b{S}_w + \b{S}_b) = \textbf{tr}(\b{X} (0.5\b{L}_w + 0.5\b{L}_b) \b{X}^\top) = \textbf{tr}(\b{X}^\top \b{X} (0.5\b{L}_w + 0.5\b{L}_b))$ because of the cyclic property of trace. If the rank of $\b{L} := \b{X}^\top \b{X} (0.5\b{L}_w + 0.5\b{L}_b) \in \mathbb{R}^{n \times n}$ is $r \leq n$, it has $r$ non-zero eigenvalues which we compute its corresponding eigenvectors. We stack these eigenvectors to have $\b{V} \in \mathbb{R}^{d \times r}$. The projected intra-class and inter-class variances after projection onto the column space of $\b{V}$ are $\b{S}'_w := \b{V}^\top \b{S}_w \b{V}$ and $\b{S}'_b := \b{V}^\top \b{S}_b \b{V}$, respectively. Then, we use $\b{S}'_w$ and $\b{S}'_b$ in Eq. (\ref{equation_optimization_Sb_lambda_Sw}) and the weight matrix of the generalized Mahalanobis is obtained using Eq. (\ref{equation_W_U_UT}). 

\subsubsection{Regularization by Locally Linear Embedding}\label{section_spectral_ML_regularization_by_LLE}

The spectral metric learning methods using scatters can be modeled as maximization of the following Rayleigh–Ritz quotient \cite{baghshah2009semi}:
\begin{equation}\label{equation_optimization_ML_LLE_1}
\begin{aligned}
& \underset{\b{U}}{\text{maximize}}
& & \frac{\sum_{(\b{x}_i, \b{x}_j) \in \mathcal{S}} \|\b{x}_i - \b{x}_j\|_{\b{W}}}{\sum_{(\b{x}_i, \b{x}_j) \in \mathcal{D}} \|\b{x}_i - \b{x}_j\|_{\b{W}} + \lambda \Omega(\b{U})}, \\ 
& \text{subject to}
& & \b{U}^\top \b{U} = \b{I},
\end{aligned}
\end{equation}
where $\b{W} = \b{U}\b{U}^\top$ (see Eq. (\ref{equation_W_U_UT})), $\lambda>0$ is the regularization parameter, and $\Omega(\b{U})$ is a penalty or regularization term on the projection matrix $\b{U}$.
This optimization maximizes and minimizes the distances of the similar and dissimilar points, respectively. 
According to Section \ref{section_relation_to_FDA}, Eq. (\ref{equation_optimization_ML_LLE_1}) can be restated as:
\begin{equation}\label{equation_optimization_ML_LLE_2}
\begin{aligned}
& \underset{\b{U}}{\text{maximize}}
& & \frac{\textbf{tr}(\b{U}^\top \b{S}_b \b{U})}{\textbf{tr}(\b{U}^\top \b{S}_w \b{U}) + \lambda \Omega(\b{U})}, \\ 
& \text{subject to}
& & \b{U}^\top \b{U} = \b{I}.
\end{aligned}
\end{equation}

As was discussed in Proposition \ref{proposition_metric_learning_projection}, metric learning can be seen as projection onto a subspace. 
The regularization term can be linear reconstruction of every projected point by its $k$ Nearest Neighbors ($k$NN) using the same reconstruction weights as before projection \cite{baghshah2009semi}. 
The weights for linear reconstruction in the input space can be found as in locally linear embedding \cite{roweis2000nonlinear,ghojogh2020locally}. If $s_{ij}$ denotes the weight of $\b{x}_j$ in reconstruction of $\b{x}_i$ and $\mathcal{N}(\b{x}_i)$ is the set of $k$NN for $\b{x}_i$, we have:
\begin{align*}
& \underset{s_{ij}}{\text{minimize}}
& & \sum_{i=1}^n \Big\|\b{x}_i - \sum_{\b{x}_j \in \mathcal{N}(\b{x}_i)} s_{ij} \b{x}_j\Big\|_2^2, \\ 
& \text{subject to}
& & \sum_{\b{x}_j \in \mathcal{N}(\b{x}_i)} s_{ij} = 1.
\end{align*}
The solution of this optimization is \cite{ghojogh2020locally}:
\begin{align*}
s_{ij}^* = \frac{\b{G}_i^{-1} \b{1}}{\b{1}^\top \b{G}_i^{-1} \b{1}},
\end{align*}
where $\b{G}_i := (\b{x}_i \b{1}^\top - \b{X}_i)^\top (\b{x}_i \b{1}^\top - \b{X}_i)$ in which $\b{X}_i \in \mathbb{R}^{d \times k}$ denotes the stack of $k$NN for $\b{x}_i$. 
We define $\b{S}^* := [s^*_{ij}] \in \mathbb{R}^{n \times n}$.
The regularization term can be reconstruction in the subspace using the same reconstruction weights as in the input space \cite{baghshah2009semi}:
\begin{align}
\Omega(\b{U}) &:= \sum_{i=1}^n \Big\|\b{U}^\top \b{x} - \sum_{\b{x}_j \in \mathcal{N}(\b{x}_i)} s^*_{ij} \b{U}^\top \b{x}_j \Big\|_2^2 \nonumber \\
&= \textbf{tr}(\b{U}^\top\b{X}\b{E}\b{X}^\top \b{U}), \label{equation_optimization_ML_LLE_penalty}
\end{align}
where $\b{X} = [\b{x}_1, \dots, \b{x}_n] \in \mathbb{R}^{d \times n}$ and $\mathbb{R}^{n \times n} \ni \b{E} := (\b{I} - \b{S}^*)^\top (\b{I} - \b{S}^*)$. 
Putting Eq. (\ref{equation_optimization_ML_LLE_penalty}) in Eq. (\ref{equation_optimization_ML_LLE_2}) gives:
\begin{equation}\label{equation_optimization_ML_LLE_3}
\begin{aligned}
& \underset{\b{U}}{\text{maximize}}
& & \frac{\textbf{tr}(\b{U}^\top \b{S}_b \b{U})}{\textbf{tr}\big(\b{U}^\top (\b{S}_w + \lambda \b{X}\b{E}\b{X}^\top) \b{U}\big)}, \\ 
& \text{subject to}
& & \b{U}^\top \b{U} = \b{I}.
\end{aligned}
\end{equation}
The solution to this optimization problem is the generalized eigenvalue problem $(\b{S}_b, \b{S}_w + \lambda \b{X}\b{E}\b{X}^\top)$ where $\b{U}$ has the eigenvectors as its columns \cite{ghojogh2019eigenvalue}. 
According to Eq. (\ref{equation_W_U_UT}), the weight matrix of metric is $\b{W} = \b{U} \b{U}^\top$. 

\subsubsection{Fisher-HSIC Multi-view Metric Learning (FISH-MML)}

Fisher-HSIC Multi-view Metric Learning (FISH-MML) \cite{zhang2018fish} is a metric learning method for multi-view data. In multi-view data, we have different types of features for every data point. For example, an image dataset, which has a descriptive caption for every image, is multi-view. 
Let $\b{X}^{(r)} := \{\b{x}_i^{(r)}\}_{i=1}^n$ be the features of data points in the $r$-th view, $c$ be the number of classes/clusters, and $v$ be the number of views. 
According to Proposition \ref{proposition_metric_learning_projection}, metric learning is the Euclidean distance after projection with $\b{U}$. 
The inter-class scatter of data, in the $r$-th view, is denoted by $\b{S}_b^{(r)}$ and calculated using Eqs. (\ref{equation_inter_cluster_scatter}).
The total scatter of data, in the $r$-th view, is denoted by $\b{S}_t^{(r)}$ and is the covariance of data in that view. 

Inspired by Fisher discriminant analysis \cite{fisher1936use,ghojogh2019fisher}, we maximize the inter-class variances of projected data, $\sum_{r=1}^v \textbf{tr}(\b{U}^\top \b{S}_b^{(r)} \b{U})$, to discriminate the classes after projection. Also, inspired by principal component analysis \cite{ghojogh2019unsupervised}, we maximize the total scatter of projected data, $\sum_{r=1}^v \textbf{tr}(\b{U}^\top \b{S}_t^{(r)} \b{U})$, for expressiveness. 
Moreover, we maximize the dependence of the projected data in all views because various views of a point should be related. A measure of dependence between two random variables $X$ and $Y$ is the Hilbert-Schmidt Independence Criterion (HSIC) \cite{gretton2005measuring} whose empirical estimation is:
\begin{align}\label{equation_HSIC}
\text{HSIC}(X,Y) = \frac{1}{(n-1)^2} \textbf{tr}(\b{K}_x \b{H} \b{K}_y \b{H}),
\end{align}
where $\b{K}_x$ and $\b{K}_y$ are kernel matrices over $X$ and $Y$ variables, respectively, and $\b{H} := \b{I} - (1/n)\b{1}\b{1}^\top$ is the centering matrix. 
The HSIC between projection of two views $\b{X}^{(r)}$ and $\b{X}^{(w)}$ is:
\begin{align*}
&\text{HSIC}(\b{U}^\top \b{X}^{(r)}, \b{U}^\top \b{X}^{(w)}) \overset{(\ref{equation_HSIC})}{\propto} \textbf{tr}(\b{K}^{(r)} \b{H} \b{K}^{(w)} \b{H}) \\
&\overset{(a)}{=} \textbf{tr}(\b{X}^{(r)\top} \b{U} \b{U}^\top \b{X}^{(r)} \b{H} \b{K}^{(w)} \b{H}) \\
&\overset{(b)}{=} \textbf{tr}(\b{U}^\top \b{X}^{(r)} \b{H} \b{K}^{(w)} \b{H} \b{X}^{(r)\top} \b{U})
\end{align*}
where $(a)$ is because we use the linear kernel for $\b{U}^\top \b{X}^{(r)}$, i.e., $\b{K}^{(r)} := (\b{U}^\top \b{X}^{(r)})^\top \b{U}^\top \b{X}^{(r)}$ and $(b)$ is because of the cyclic property of trace. 

In summary, we maximize the summation of inter-class scatter, total scatter, and the dependence of views, which is:
\begin{align*}
&\sum_{r=1}^v \big( \textbf{tr}(\b{U}^\top \b{S}_b^{(r)} \b{U}) + \lambda_1 \textbf{tr}(\b{U}^\top \b{S}_t^{(r)} \b{U}) \\
&~~~~~~~~ + \lambda_2 \textbf{tr}(\b{U}^\top \b{X}^{(r)} \b{H} \b{K}^{(w)} \b{H} \b{X}^{(r)\top} \b{U}) \big) \\
&= \sum_{r=1}^v \textbf{tr}\big(\b{U}^\top (\b{S}_b^{(r)} + \lambda_1 \b{S}_t^{(r)} \\
&~~~~~~~~ + \lambda_2 \b{X}^{(r)} \b{H} \b{K}^{(w)} \b{H} \b{X}^{(r)\top} ) \b{U}\big),
\end{align*}
where $\lambda_1, \lambda_2 >0$ are the regularization parameters. 
The optimization problem is:
\begin{equation}
\begin{aligned}
& \underset{\b{U}}{\text{maximize}}
& & \sum_{r=1}^v \textbf{tr}\big(\b{U}^\top (\b{S}_b^{(r)} + \lambda_1 \b{S}_t^{(r)} \\
& & &~~~~~~ + \lambda_2 \b{X}^{(r)} \b{H} \b{K}^{(w)} \b{H} \b{X}^{(r)\top} ) \b{U}\big) \\
& \text{subject to} 
& & \b{U}^\top \b{U} = \b{I},
\end{aligned}
\end{equation}
whose solution is the eigenvalue problem for $\b{S}_b^{(r)} + \lambda_1 \b{S}_t^{(r)} + \lambda_2 \b{X}^{(r)} \b{H} \b{K}^{(w)} \b{H} \b{X}^{(r)\top}$ where $\b{U}$ has the eigenvectors as its columns \cite{ghojogh2019eigenvalue}. 

\subsection{Spectral Methods Using Hinge Loss}

\subsubsection{Large-Margin Metric Learning}\label{section_large_margin_metric_learning}

$k$-Nearest Neighbors ($k$NN) classification is highly impacted by the metric used for measuring distances between points. Hence, we can use metric learning for improving the performance of $k$NN classification \cite{weinberger2006distance,weinberger2009distance}. 
Let $y_{ij}=1$ if $(\b{x}_i, \b{x}_j) \in \mathcal{S}$ and $y_{ij}=0$ if $(\b{x}_i, \b{x}_j) \in \mathcal{D}$. Moreover, we consider $k$NN for similar points where we find the nearest neighbors of every point among the similar points to that point. Let $\eta_{ij} = 1$ if $(\b{x}_i, \b{x}_j) \in \mathcal{S}$ and $\b{x}_j$ is among $k$NN of $\b{x}_i$. Otherwise, $\eta_{ij} = 0$. 
The optimization problem for finding the best weigh matrix in the metric can be \cite{weinberger2006distance,weinberger2009distance}:
\begin{equation}\label{equation_optimization_largeMarginMetricLearning}
\begin{aligned}
& \underset{\b{W}}{\text{minimize}}
& & \sum_{i=1}^n \sum_{j=1}^n \eta_{ij} \|\b{x}_i - \b{x}_j\|_{\b{W}}^2 \\
& & &+ \lambda \sum_{i=1}^n \sum_{j=1}^n \sum_{l=1}^n \eta_{ij} (1 - y_{il})\Big[1 \\
& & &~~~~~~~~~~~ + \|\b{x}_i - \b{x}_j\|_{\b{W}}^2 - \|\b{x}_i - \b{x}_l\|_{\b{W}}^2\Big]_+, \\
& \text{subject to} 
& & \b{W} \succeq \b{0}, 
\end{aligned}
\end{equation}
where $\lambda>0$ is the regularization parameter, and $[.]_+ := \max(.,0)$ is the standard Hinge loss. 

The first term in Eq. (\ref{equation_optimization_largeMarginMetricLearning}) pushes the similar neighbors close to each other. The second term in this equation is the triplet loss \cite{schroff2015facenet} which pushes the similar neighbors to each other and pulls the dissimilar points away from one another. This is because minimizing $\|\b{x}_i - \b{x}_j\|_{\b{W}}^2$ for $\eta_{ij}=1$ decreases the distances of similar neighbors. Moreover, minimizing $- \|\b{x}_i - \b{x}_l\|_{\b{W}}^2$ for $1-y_{il}=1$ (i.e., $y_{il} = 0$) is equivalent to maximizing $\|\b{x}_i - \b{x}_l\|_{\b{W}}^2$ which maximizes the distances of dissimilar points. 
Minimizing the whole second term forces the distances of dissimilar points to be at least greater that the distances of similar points up to a threshold (or margin) of one. We can change the margin by changing $1$ in this term with some other positive number. 
In this sense, this loss is closely related to the triplet loss for neural networks \cite{schroff2015facenet} (see Section \ref{section_triplet_loss}).

Eq. (\ref{equation_optimization_largeMarginMetricLearning}) can be restated using slack variables $\xi_{ijl}, \forall i,j,l \in \{1, \dots, n\}$. 
The Hinge loss in term $[1 + \|\b{x}_i - \b{x}_j\|_{\b{W}}^2 - \|\b{x}_i - \b{x}_l\|_{\b{W}}^2]_+$ requires to have:
\begin{align*}
&1 + \|\b{x}_i - \b{x}_j\|_{\b{W}}^2 - \|\b{x}_i - \b{x}_l\|_{\b{W}}^2 \geq 0 \\
&\implies \|\b{x}_i - \b{x}_l\|_{\b{W}}^2 - \|\b{x}_i - \b{x}_j\|_{\b{W}}^2 \leq 1.
\end{align*}
If $\xi_{ijl} \geq 0$, we can have sandwich the term $\|\b{x}_i - \b{x}_l\|_{\b{W}}^2 - \|\b{x}_i - \b{x}_j\|_{\b{W}}^2$ in order to minimize it:
\begin{align*}
& 1 - \xi_{ijl} \leq \|\b{x}_i - \b{x}_l\|_{\b{W}}^2 - \|\b{x}_i - \b{x}_j\|_{\b{W}}^2 \leq 1.
\end{align*}
Hence, we can replace the term of Hinge loss with the slack variable.
Therefore, Eq. (\ref{equation_optimization_largeMarginMetricLearning}) can be restated as \cite{weinberger2006distance,weinberger2009distance}:
\begin{equation}\label{equation_optimization_largeMarginMetricLearning_2}
\begin{aligned}
& \underset{\b{W},\, \{\xi_{ijl}\}}{\text{minimize}}
& & \sum_{i=1}^n \sum_{j=1}^n \eta_{ij} \|\b{x}_i - \b{x}_j\|_{\b{W}}^2 \\
& & &+ \lambda \sum_{i=1}^n \sum_{j=1}^n \sum_{l=1}^n \eta_{ij} (1 - y_{il})\, \xi_{ijl} \\
& \text{subject to} 
& & \|\b{x}_i - \b{x}_l\|_{\b{W}}^2 - \|\b{x}_i - \b{x}_j\|_{\b{W}}^2 \geq 1 - \xi_{ijl}, \\
& & &~~~~~~~~~ \forall (\b{x}_i, \b{x}_j) \in \mathcal{S}, \eta_{ij}=1, (\b{x}_i, \b{x}_l) \in \mathcal{D}, \\
& & & \xi_{ijl} \geq 0, \\
& & & \b{W} \succeq \b{0}.
\end{aligned}
\end{equation}
This optimization problem is a semidefinite programming which can be solved iteratively using interior-point method \cite{ghojogh2021kkt}. 

This problem uses triplets of similar and dissimilar points, i.e., $\{\b{x}_i, \b{x}_j, \b{x}_l\}$ where $(\b{x}_i, \b{x}_j) \in \mathcal{S}$, $\eta_{ij}=1$, $(\b{x}_i, \b{x}_l) \in \mathcal{D}$. 
Hence, triplets should be extracted randomly from the dataset for this metric learning. 
Solving semidefinite programming is usually slow and time-consuming especially for large datasets.
Triplet minimizing can be used for finding the best triplets for learning \cite{poorheravi2020acceleration}. For example, the similar and dissimilar points with smallest and/or largest distances can be used to limit the number of triplets \cite{sikaroudi2020offline}. 
The reader can also refer to for Lipschitz analysis in large margin metric learning \cite{dong2019metric}. 

\subsubsection{Imbalanced Metric Learning (IML)}

Imbalanced Metric Learning (IML) \cite{gautheron2019metric} is a spectral metric learning method which handles imbalanced classes by further decomposition of the similar set $\mathcal{S}$ and dissimilar set $\mathcal{D}$. 
Suppose the dataset is composed of two classes $c_0$ and $c_1$. Let $\mathcal{S}_0$ and $\mathcal{S}_1$ denote the similarity sets for classes $c_0$ and $c_1$, respectively. We define pairs of points taken randomly from these sets to have similarity and dissimilarity sets \cite{gautheron2019metric}:
\begin{align*}
& \text{Sim}_0 \subseteq \mathcal{S}_0 \times \mathcal{S}_0, \quad \text{Sim}_1 \subseteq \mathcal{S}_1 \times \mathcal{S}_1, \\
& \text{Dis}_0 \subseteq \mathcal{S}_0 \times \mathcal{S}_1, \quad \text{Dis}_1 \subseteq \mathcal{S}_1 \times \mathcal{S}_0.
\end{align*}
The optimization problem of IML is:
\begin{align}
& \underset{\b{W}}{\text{minimize}}
~~~~~~ \frac{\lambda}{4|\text{Sim}_0|} \sum_{(\b{x}_i, \b{x}_j) \in \text{Sim}_0} \big[\|\b{x}_i - \b{x}_j\|_{\b{W}}^2 - 1\big]_+ \nonumber\\
&~~~~~~~~~ + \frac{\lambda}{4|\text{Sim}_1|} \sum_{(\b{x}_i, \b{x}_j) \in \text{Sim}_1} \big[\|\b{x}_i - \b{x}_j\|_{\b{W}}^2 - 1\big]_+ \nonumber
\end{align}
\begin{align}
&~~~~~~~~~ + \frac{1-\lambda}{4|\text{Dis}_0|} \sum_{(\b{x}_i, \b{x}_j) \in \text{Dis}_0} \big[\!-\|\b{x}_i - \b{x}_j\|_{\b{W}}^2 +1+m \big]_+ \nonumber \\
&~~~~~~~~~ + \frac{1-\lambda}{4|\text{Dis}_1|} \sum_{(\b{x}_i, \b{x}_j) \in \text{Dis}_1} \big[\!-\|\b{x}_i - \b{x}_j\|_{\b{W}}^2 +1+m \big]_+ \nonumber \\
&~~~~~~~~~ + \gamma \|\b{W} - \b{I}\|_F^2 \nonumber \\
& \text{subject to} ~~~~ \b{W} \succeq \b{0},
\end{align}
where $|.|$ denotes the cardinality of set, $[.]_+ := \max(.,0)$ is the standard Hinge loss, $m>0$ is the desired margin between classes, and $\lambda \in [0,1]$ and $\gamma>0$ are the regularization parameters.
This optimization pulls the similar points to have distance less than $1$ and pushes the dissimilar points away to have distance more than $m+1$. Also, the regularization term $\|\b{W} - \b{I}\|_F^2$ tries to make the weight matrix is the generalized Mahalanobis distance close to identity for simplicity of metric. In this way, the metric becomes close to the Euclidean distance, preventing overfitting, while satisfying the desired margins in distances.

\subsection{Locally Linear Metric Adaptation (LLMA)}

Another method for metric learning is Locally Linear Metric Adaptation (LLMA) \cite{chang2004locally}. 
LLMA performs nonlinear and linear transformations globally and locally, respectively. 
For every point $\b{x}_l$, we consider its $k$ nearest (similar) neighbors. 
The local linear transformation for every point $\b{x}_l$ is:
\begin{align}
\mathbb{R}^d \ni \b{y}_l := \b{x}_l + \b{B} \b{\pi}_i,
\end{align}
where $\b{B} \in \mathbb{R}^{d \times k}$ is the matrix of biases, $\mathbb{R}^k \ni \b{\pi}_i = [\pi_{i1}, \dots, \pi_{ik}]^\top$, and $\pi_{ij} := \exp(-\|\b{x}_i - \b{x}_j\|_2^2 / 2 w^2)$ is a Gaussian measure of similarity between $\b{x}_i$ and $\b{x}_j$. The variables $\b{B}$ and $w$ are found by optimization. 

In this method, we minimize the distances between the linearly transformed similar points while the distances of similar points are tried to be preserved after the transformation:
\begin{equation}\label{equation_ML_optimization_spectral_LLMA}
\begin{aligned}
& \underset{\{\b{y}_i\}_{i=1}^n, \b{B}, w, \sigma}{\text{minimize}}
& & \sum_{(\b{y}_i, \b{y}_j) \in \mathcal{S}} \|\b{y}_i - \b{y}_j\|_2^2 \\
& & & + \lambda\, \sum_{i=1}^n \sum_{j=1}^n (q_{ij} - d_{ij})^2 \exp(\frac{-d_{ij}^2}{\sigma^2}),
\end{aligned}
\end{equation}
where $\lambda>0$ is the regularization parameter, $\sigma_2^2$ is the variance to be optimized, and $d_{ij} := \|\b{x}_i - \b{x}_j\|_2$ and $q_{ij} := \|\b{y}_i - \b{y}_j\|_2$. 
This objective function is optimized iteratively until convergence. 

\subsection{Relevant to Support Vector Machine}\label{section_relation_to_SVM}

Inspired by $\nu$-Support Vector Machine ($\nu$-SVM) \cite{scholkopf2000new}, the weight matrix in the generalized Mahalanobis distance can be obtained as \cite{tsang2003distance}: 
\begin{equation}
\begin{aligned}
& \underset{\b{W}, \gamma, \{\xi_{il}\}}{\text{minimize}}
& & \frac{1}{2} \|\b{W}\|_2^2 + \frac{\lambda_1}{|\mathcal{S}|} \sum_{(\b{x}_i, \b{x}_j) \in \mathcal{S}} \|\b{x}_i - \b{x}_j\|_{\b{W}}^2 \\
& & & + \lambda_2 \Big( \nu \gamma + \frac{1}{|\mathcal{D}|} \sum_{(\b{x}_i, \b{x}_l) \in \mathcal{D}} \xi_{il} \Big) \\
& \text{subject to} 
& & \b{W} \succeq \b{0}, \\
& & & \gamma \geq 0, \\
& & & \|\b{x}_i - \b{x}_j\|_{\b{W}}^2 - \|\b{x}_i - \b{x}_l\|_{\b{W}}^2 \geq \gamma - \xi_{il}, \\
& & &~~~~~~~~~~~~~~~~~~ \forall (\b{x}_i, \b{x}_j) \in \mathcal{S}, (\b{x}_i, \b{x}_l) \in \mathcal{D}, \\
& & & \xi_{il} \geq 0, \quad \forall (\b{x}_i, \b{x}_l) \in \mathcal{D},
\end{aligned}
\end{equation}
where $\lambda_1, \lambda_2 > 0$ are regularization parameters. 
Using KKT conditions and Lagrange multipliers \cite{ghojogh2021kkt}, the dual optimization problem is (see \cite{tsang2003distance} for derivation):
\begin{equation}\label{equation_relation_to_SVM_dual_optimization}
\begin{aligned}
& \underset{\{\alpha_{ij}\}}{\text{maximize}}
~~~~~ \sum_{(\b{x}_i, \b{x}_j) \in \mathcal{D}} \alpha_{ij} (\b{x}_i - \b{x}_j)^\top \b{W} (\b{x}_i - \b{x}_j) \\
& -\frac{1}{2} \sum_{(\b{x}_i, \b{x}_j) \in \mathcal{D}} \sum_{(\b{x}_k, \b{x}_l) \in \mathcal{D}} \alpha_{ij} \alpha_{kl} ((\b{x}_i - \b{x}_j)^\top (\b{x}_k - \b{x}_l))^2 \\
& + \frac{\lambda_1}{|\mathcal{S}|} \sum_{(\b{x}_i, \b{x}_j) \in \mathcal{D}} \sum_{(\b{x}_k, \b{x}_l) \in \mathcal{S}} \alpha_{ij} ((\b{x}_i - \b{x}_j)^\top (\b{x}_k - \b{x}_l))^2 \\
& \text{subject to} 
~~~~~~~~~~ \frac{1}{\lambda_2} \sum_{(\b{x}_i, \b{x}_j) \in \mathcal{D}} \alpha_{ij} \geq \nu, \\
& ~~~~~~~~~~~~~~~~~~~~~~~~~~ \alpha_{ij} \in [0, \frac{\lambda_2}{|\mathcal{D}|}],
\end{aligned}
\end{equation}
where $\{\alpha_{ij}\}$ are the dual variables. This problem is a quadratic programming problem and can be solved using optimization solvers. 

\subsection{Relevant to Multidimensional Scaling}

Multidimensional Scaling (MDS) tries to preserve the distance after projection onto its subspace \cite{cox2008multidimensional,ghojogh2020multidimensional}. We saw in Proposition \ref{proposition_metric_learning_projection} that metric learning can be seen as projection onto the column space of $\b{U}$ where $\b{W} = \b{U} \b{U}^\top$. Inspired by MDS, we can learn a metric which preserves the distances between points after projection onto the subspace of metric \cite{zhang2003parametric}:
\begin{equation}
\begin{aligned}
& \underset{\b{W}}{\text{minimize}}
& & \sum_{i=1}^n \sum_{j=1}^n (\|\b{x}_i - \b{x}_j\|_2^2 - \|\b{x}_i - \b{x}_j\|_{\b{W}}^2)^2 \\
& \text{subject to} & & \b{W} \succeq \b{0}. 
\end{aligned}
\end{equation}
It can be solved using any optimization method \cite{ghojogh2021kkt}. 

\subsection{Kernel Spectral Metric Learning}

Let $k(\b{x}_i, \b{x}_j) := \b{\phi}(\b{x}_i)^\top \b{\phi}(\b{x}_j)$ be the kernel function over data points $\b{x}_i$ and $\b{x}_j$, where $\b{\phi}(.)$ is the pulling function to the Reproducing Kernel Hilbert Space (RKHS) \cite{ghojogh2021reproducing}. 
Let $\mathbb{R}^{n \times n} \ni \b{K} := \b{\Phi}(\b{X})^\top \b{\Phi}(\b{X})$ be the kernel matrix of data. 
In the following, we introduce some of the kernel spectral metric learning methods. 

\subsubsection{Using Eigenvalue Decomposition of Kernel}

One of the kernel methods for spectral metric learning is \cite{yeung2007kernel}. It has two approaches; we explain one of its approaches here. 
The eigenvalue decomposition of the kernel matrix is:
\begin{align}\label{equation_kernel_eigenvalue_decomposition}
\b{K} = \sum_{r=1}^p \beta_r^2 \b{\alpha}_r \b{\alpha}_r^\top \overset{(a)}{=} \sum_{r=1}^p \beta_r^2 \b{K}_r
\end{align}
where $p$ is the rank of kernel matrix, $\beta_r^2$ is the non-negative $r$-th eigenvalue (because $\b{K} \succeq \b{0}$), $\b{\alpha}_r \in \mathbb{R}^n$ is the $r$-th eigenvector, and $(a)$ is because we define $\b{K}_r := \b{\alpha}_r \b{\alpha}_r^\top$.
We can consider $\{\beta_r^2\}_{r=1}^p$ as learnable parameters and not the eigenvalues. Hence, we learn $\{\beta_r^2\}_{r=1}^p$ for the sake of metric learning. 
The distance metric of pulled data points to RKHS is \cite{scholkopf2001kernel,ghojogh2021reproducing}:
\begin{equation}\label{equation_distance_in_RKHS}
\begin{aligned}
\|\b{\phi}(\b{x}_i) - &\b{\phi}(\b{x}_j)\|_2^2 \\
&= k(\b{x}_i, \b{x}_i) + k(\b{x}_j, \b{x}_j) - 2 k(\b{x}_i, \b{x}_j).
\end{aligned}
\end{equation}
In metric learning, we want to make the distances of similar points small; hence the objective to be minimized is:
Hence, we have:
\begin{align*}
& \sum_{(\b{x}_i, \b{x}_j) \in \mathcal{S}} \|\b{\phi}(\b{x}_i) - \b{\phi}(\b{x}_j)\|_2^2 \\
&= \sum_{(\b{x}_i, \b{x}_j) \in \mathcal{S}} k(\b{x}_i, \b{x}_i) + k(\b{x}_j, \b{x}_j) - 2 k(\b{x}_i, \b{x}_j) \\
&\overset{(\ref{equation_kernel_eigenvalue_decomposition})}{=} \sum_{r=1}^p \beta_r^2 \sum_{(\b{x}_i, \b{x}_j) \in \mathcal{S}} k_r(\b{x}_i, \b{x}_i) + k_r(\b{x}_j, \b{x}_j) \\
&~~~~~~~~~~~~~~~~~~~~~~~~~~~~ - 2 k_r(\b{x}_i, \b{x}_j) \\
&\overset{(a)}{=} \sum_{r=1}^p \beta_r^2 \sum_{(\b{x}_i, \b{x}_j) \in \mathcal{S}} (\b{e}_i - \b{e}_j)^\top \b{K}_r (\b{e}_i - \b{e}_j) \\
&\overset{(b)}{=} \sum_{r=1}^p \beta_r^2 f_r \overset{(c)}{=} \b{\beta}^\top \b{D}_{\mathcal{S}} \b{\beta},
\end{align*}
where $(a)$ is because $\b{e}_i$ is the vector whose $i$-th element is one and other elements are zero, $(b)$ is because we define $f_r := \sum_{(\b{x}_i, \b{x}_j) \in \mathcal{S}} (\b{e}_i - \b{e}_j)^\top \b{K}_r (\b{e}_i - \b{e}_j)$, and $(c)$ is because we define $\b{D}_{\mathcal{S}} := \textbf{diag}([f_1, \dots, f_p]^\top)$ and $\b{\beta} := [\beta_1, \dots, \beta_p]^\top$. 
By adding a constraint on the summation of $\{\beta_r^2\}_{r=1}^p$, the optimization problem for metric learning is:
\begin{equation}
\begin{aligned}
& \underset{\b{\beta}}{\text{minimize}}
& & \b{\beta}^\top \b{D}_{\mathcal{S}} \b{\beta} \\
& \text{subject to} & & \b{1}^\top \b{\beta} = 1.
\end{aligned}
\end{equation}
This optimization is similar to the form of one of the optimization problems in locally linear embedding \cite{roweis2000nonlinear,ghojogh2020locally}.
The Lagrangian for this problem is \cite{ghojogh2021kkt}:
\begin{align*}
&\mathcal{L} = \b{\beta}^\top \b{D}_{\mathcal{S}} \b{\beta} - \lambda (\b{1}^\top \b{\beta} - 1),
\end{align*}
where $\lambda$ is the dual variable. 
Taking derivative of the Lagrangian w.r.t. the variables and setting to zero gives:
\begin{align*}
& \frac{\partial \mathcal{L}}{\partial \b{\beta}} = 2 \b{D}_{\mathcal{S}} \b{\beta} - \lambda \b{1} \overset{\text{set}}{=} 0 \implies \b{\beta} = \frac{\lambda}{2} \b{D}_{\mathcal{S}}^{-1} \b{1}, \\
& \frac{\partial \mathcal{L}}{\partial \lambda} = \b{1}^\top \b{\beta} - 1 \overset{\text{set}}{=} 0 \implies \b{1}^\top \b{\beta} = 1, \\
& \implies \frac{\lambda}{2} \b{1}^\top \b{D}_{\mathcal{S}}^{-1} \b{1} = 1 \implies \lambda = \frac{2}{\b{1}^\top \b{D}_{\mathcal{S}}^{-1} \b{1}} \\
& \implies \b{\beta} = \frac{\b{D}_{\mathcal{S}}^{-1} \b{1}}{\b{1}^\top \b{D}_{\mathcal{S}}^{-1} \b{1}}. 
\end{align*}
Hence, the optimal $\b{\beta}$ is obtained for metric learning in the RKHS where the distances of similar points is smaller than in the input Euclidean space. 

\subsubsection{Regularization by Locally Linear Embedding}

The method \cite{baghshah2009semi}, which was introduced in Section \ref{section_spectral_ML_regularization_by_LLE}, can be kernelized. Recall that this method used locally linear embedding for regularization. 
According to the representation theory \cite{ghojogh2021reproducing}, the solution in the RKHS can be represented as a linear combination of all pulled data points to RKHS:
\begin{align}\label{equation_kernelization_representation_theory}
\b{\Phi}(\b{U}) = \b{\Phi}(\b{X}) \b{T},
\end{align}
where $\b{X} = [\b{x}_1, \dots, \b{x}_n]$ and $\b{T} \in \mathbb{R}^{n \times p}$ ($p$ is the dimensionality of subspace) is the coefficients. 

We define the similarity and dissimilarity adjacency matrices as:
\begin{equation}
\begin{aligned}
& \b{A}_S(i,j) := 
\left\{
    \begin{array}{ll}
        1 & \mbox{if } (\b{x}_i, \b{x}_j) \in \mathcal{S}, \\
        0 & \mbox{otherwise.}
    \end{array}
\right. \\
& \b{A}_D(i,j) := 
\left\{
    \begin{array}{ll}
        1 & \mbox{if } (\b{x}_i, \b{x}_j) \in \mathcal{D}, \\
        0 & \mbox{otherwise.}
    \end{array}
\right. 
\end{aligned}
\end{equation}
Let $\b{L}_w$ and $\b{L}_b$ denote the Laplacian matrices \cite{ghojogh2021laplacian} of these adjacency matrices:
\begin{align*}
& \b{L}_w := \b{D}_S - \b{A}_S(i,j), \quad \b{L}_b := \b{D}_D - \b{A}_D(i,j), 
\end{align*}
where $\b{D}_S(i,i) := \sum_{j=1}^n \b{A}_S(i,j)$ and $\b{D}_D(i,i) := \sum_{j=1}^n \b{A}_D(i,j)$ are diagonal matrices. 
The terms in the objective of Eq. (\ref{equation_optimization_ML_LLE_3}) can be restated using Laplacian of adjacency matrices rather than the scatters:
\begin{equation}\label{equation_optimization_ML_LLE_4}
\begin{aligned}
& \underset{\b{U}}{\text{maximize}}
& & \frac{\textbf{tr}(\b{U}^\top \b{L}_b \b{U})}{\textbf{tr}\big(\b{U}^\top (\b{L}_w + \lambda \b{X}\b{E}\b{X}^\top) \b{U}\big)}, \\ 
& \text{subject to}
& & \b{U}^\top \b{U} = \b{I}.
\end{aligned}
\end{equation}
According to the representation theory, the pulled Laplacian matrices to RKHS are $\b{\Phi}(\b{L}_b) = \b{\Phi}(\b{X}) \b{L}_b \b{\Phi}(\b{X})^\top$ and $\b{\Phi}(\b{L}_w) = \b{\Phi}(\b{X}) \b{L}_w \b{\Phi}(\b{X})^\top$. 
Hence, the numerator of Eq. (\ref{equation_optimization_ML_LLE_3}) in RKHS becomes:
\begin{align*}
&\textbf{tr}(\b{\Phi}(\b{U})^\top \b{\Phi}(\b{X}) \b{L}_b \b{\Phi}(\b{X})^\top \b{\Phi}(\b{U})) \\
&= \textbf{tr}\big(\b{T}^\top \b{\Phi}(\b{X})^\top \b{\Phi}(\b{X}) \b{L}_b \b{\Phi}(\b{X})^\top \b{\Phi}(\b{X}) \b{T}\big) \\
&\overset{(a)}{=} \textbf{tr}\big(\b{T}^\top \b{K}_x \b{L}_b \b{K}_x \b{T}\big),
\end{align*}
where $(a)$ is because of the kernel trick \cite{ghojogh2021reproducing}, i.e., 
\begin{align}\label{equation_Kernel_X}
\b{K}_x := \b{\Phi}(\b{X})^\top \b{\Phi}(\b{X}).
\end{align}
similarly, the denominator of Eq. (\ref{equation_optimization_ML_LLE_3}) in RKHS becomes:
\begin{align*}
& \textbf{tr}\big(\b{\Phi}(\b{U})^\top (\b{\Phi}(\b{X}) \b{L}_w \b{\Phi}(\b{X})^\top + \lambda \b{\Phi}(\b{X})\b{E}\b{\Phi}(\b{X})^\top) \b{\Phi}(\b{U})\big) \\
& \overset{(\ref{equation_kernelization_representation_theory})}{=} \textbf{tr}\big(\b{T}^\top \b{\Phi}(\b{X})^\top (\b{\Phi}(\b{X}) \b{L}_w \b{\Phi}(\b{X})^\top \\
&~~~~~~~~~~~~~~~~~~~~~~~~~~~~~~~ + \lambda \b{\Phi}(\b{X})\b{E}\b{\Phi}(\b{X})^\top) \b{\Phi}(\b{X}) \b{T}\big) \\
&\overset{(a)}{=} \textbf{tr}\big(\b{T}^\top \b{K}_x (\b{L}_w + \lambda \b{E}) \b{K}_x \b{T} \big),
\end{align*}
where $(a)$ is because of the kernel trick \cite{ghojogh2021reproducing}.
The constrain in RKHS becomes:
\begin{align*}
& \b{\Phi}(\b{U})^\top \b{\Phi}(\b{U}) \overset{(\ref{equation_kernelization_representation_theory})}{=} \b{T}^\top \b{\Phi}(\b{X})^\top \b{\Phi}(\b{X}) \b{T} \overset{(a)}{=} \b{T}^\top \b{K}_x \b{T},
\end{align*}
where $(a)$ is because of the kernel trick \cite{ghojogh2021reproducing}.
The Eq. (\ref{equation_optimization_ML_LLE_3}) in RKHS is:
\begin{equation}\label{equation_optimization_ML_LLE_3_kernel}
\begin{aligned}
& \underset{\b{T}}{\text{maximize}}
& & \frac{\textbf{tr}\big(\b{T}^\top \b{K}_x \b{L}_b \b{K}_x \b{T}\big)}{\textbf{tr}\big(\b{T}^\top \b{K}_x (\b{L}_w + \lambda \b{E}) \b{K}_x \b{T} \big)}, \\ 
& \text{subject to}
& & \b{T}^\top \b{K}_x \b{T} = \b{I}.
\end{aligned}
\end{equation}
It can be solved using projected gradient method \cite{ghojogh2021kkt} to find the optimal $\b{T}$. Then, the projected data onto the subspace of metric is found as:
\begin{align}
\b{\Phi}(\b{U})^\top \b{\Phi}(\b{X}) \overset{(\ref{equation_kernelization_representation_theory})}{=} \b{T}^\top \b{\Phi}(\b{X})^\top \b{\Phi}(\b{X}) \overset{(a)}{=} \b{T}^\top \b{K}_x,
\end{align}
where $(a)$ is because of the kernel trick \cite{ghojogh2021reproducing}.

\subsubsection{Regularization by Laplacian}

Another kernel spectral metric learning method is \cite{baghshah2010kernel} whose optimization is in the form:
\begin{equation}\label{equation_optimization_kernel_soectral_ML_LaplacianRegularization}
\begin{aligned}
& \underset{\b{\Phi}(\b{X})}{\text{minimize}}
& & \frac{1}{|\mathcal{S}|} \sum_{(\b{x}_i, \b{x}_j) \in \mathcal{S}} \!\!\!\!\! \|\b{\phi}(\b{x}_i) - \b{\phi}(\b{x}_j)\|_2^2 + \lambda \Omega(\b{\Phi}(\b{X})), \\ 
& \text{subject to}
& & \|\b{\phi}(\b{x}_i) - \b{\phi}(\b{x}_j)\|_2^2 \geq c, \quad \forall (\b{x}_i, \b{x}_j) \in \mathcal{D},
\end{aligned}
\end{equation}
where $c>0$ is a hyperparameter and $\lambda>0$ is the regularization parameter.
Consider the $k$NN graph of data with an adjacency matrix $\b{A} \in \mathbb{R}^{n \times n}$ whose $(i,j)$-th element is one if $\b{x}_i$ and $\b{x}_j$ are neighbors and is zero otherwise. Let the Laplacian matrix of this adjacency matrix be denoted by $\b{L}$. 

In this method, the regularization term $\Omega(\b{\Phi}(\b{X}))$ can be the objective of Laplacian eigenmap \cite{ghojogh2021laplacian}:
\begin{align*}
\Omega(\b{\Phi}(\b{X})) := &\frac{1}{2n} \sum_{i=1}^n \sum_{j=1}^n \|\b{\phi}(\b{x}_i) - \b{\phi}(\b{x}_j)\|_2^2 \b{A}(i,j) \\
&\overset{(a)}{=} \textbf{tr}(\b{\Phi}(\b{X}) \b{L} \b{\Phi}(\b{X})^\top) \\
&\overset{(b)}{=} \textbf{tr}(\b{L} \b{\Phi}(\b{X})^\top \b{\Phi}(\b{X})) \overset{(c)}{=} \textbf{tr}(\b{L} \b{K}_x),
\end{align*}
where $(a)$ is according to \cite{belkin2001laplacian} (see \cite{ghojogh2021laplacian} for proof), $(b)$ is because of the cyclic property of trace, and $(c)$ is because of the kernel trick \cite{ghojogh2021reproducing}. 
Moreover, according to Eq. (\ref{equation_distance_in_RKHS}), the distance in RKHS is $\|\b{\phi}(\b{x}_i) - \b{\phi}(\b{x}_j)\|_2^2 = k(\b{x}_i, \b{x}_i) + k(\b{x}_j, \b{x}_j) - 2 k(\b{x}_i, \b{x}_j)$. 
We can simplify the term in Eq. (\ref{equation_optimization_kernel_soectral_ML_LaplacianRegularization}) as:
\begin{align*}
&\frac{1}{|\mathcal{S}|} \sum_{(\b{x}_i, \b{x}_j) \in \mathcal{S}} \!\!\!\!\! \|\b{\phi}(\b{x}_i) - \b{\phi}(\b{x}_j)\|_2^2 \\
&\overset{(\ref{equation_distance_in_RKHS})}{=} \frac{1}{|\mathcal{S}|} \sum_{(\b{x}_i, \b{x}_j) \in \mathcal{S}} \!\!\!\!\! k(\b{x}_i, \b{x}_i) + k(\b{x}_j, \b{x}_j) - 2 k(\b{x}_i, \b{x}_j) \\
&= \frac{1}{|\mathcal{S}|} \sum_{(\b{x}_i, \b{x}_j) \in \mathcal{S}} \!\!\!\!\! (\b{e}_i - \b{e}_j)^\top \b{K}_x (\b{e}_i - \b{e}_j) \overset{(a)}{=} \textbf{tr}(\b{E}_{\mathcal{S}} \b{K}_x), 
\end{align*}
where $(a)$ is because the scalar is equal to its trace and we use the cyclic property of trace, i.e., $(\b{e}_i - \b{e}_j)^\top \b{K}_x (\b{e}_i - \b{e}_j) = \textbf{tr}((\b{e}_i - \b{e}_j)^\top \b{K}_x (\b{e}_i - \b{e}_j)) = \textbf{tr}((\b{e}_i - \b{e}_j) (\b{e}_i - \b{e}_j)^\top \b{K}_x)$, and then we define $\b{E}_{\mathcal{S}} := (1 / |\mathcal{S}|) \sum_{(\b{x}_i, \b{x}_j) \in \mathcal{S}} (\b{e}_i - \b{e}_j) (\b{e}_i - \b{e}_j)^\top$.

Hence, Eq. (\ref{equation_optimization_kernel_soectral_ML_LaplacianRegularization}) can be restated as:
\begin{equation}\label{equation_optimization_kernel_soectral_ML_LaplacianRegularization_2}
\begin{aligned}
& \underset{\b{K}_x}{\text{minimize}}
& & \textbf{tr}(\b{E}_{\mathcal{S}} \b{K}_x) + \lambda\, \textbf{tr}(\b{L} \b{K}_x), \\ 
& \text{subject to}
& & k(\b{x}_i, \b{x}_i) + k(\b{x}_j, \b{x}_j) - 2 k(\b{x}_i, \b{x}_j) \geq c, \\
& & & ~~~~~~~~~~~~~~ \quad \forall (\b{x}_i, \b{x}_j) \in \mathcal{D}, \\
& & & \b{K}_x \succeq \b{0},
\end{aligned}
\end{equation}
noticing that the kernel matrix is positive semidefinite. 
This problem is a Semidefinite Programming (SDP) problem and can be solved using the interior point method \cite{ghojogh2021kkt}.
The optimal kernel matrix can be decomposed using eigenvalue decomposition to find the embedding of data in RKHS, i.e., $\b{\Phi}(\b{X})$:
\begin{align*}
\b{K}_x = \b{V}^\top \b{\Sigma} \b{V} = \b{V}^\top \b{\Sigma}^{(1/2} \b{\Sigma}^{(1/2)} \b{V} \overset{(\ref{equation_Kernel_X})}{=} \b{\Phi}(\b{X})^\top \b{\Phi}(\b{X}),
\end{align*}
where $\b{V}$ and $\b{\Sigma}$ are the eigenvectors and eigenvalues, $(a)$ is because $\b{K}_x \succeq \b{0}$ so its eigenvalues are non-negative can be taken second root of, and $(b)$ is because we get $\b{\Phi}(\b{X}) := \b{\Sigma}^{(1/2)} \b{V}$. 

\subsubsection{Kernel Discriminative Component Analysis}

Here, we explain the kernel version of DCA \cite{hoi2006learning} which was introduced in Section \ref{section_DCA}. 

\begin{lemma}
The generalized Mahalanobis distance metric in RKHS, with the pulled weight matrix to RKHS denoted by $\b{\Phi}(\b{W})$, can be seen as measuring the Euclidean distance in RKHS after projection onto the column subspace of $\b{T}$ where $\b{T}$ is the coefficient matrix in Eq. (\ref{equation_kernelization_representation_theory}). In other words:
\begin{equation}\label{equation_Mahalanobis_distance_in_RKHS}
\begin{aligned}
\|\b{\phi}(\b{x}_i) - &\,\b{\phi}(\b{x}_j)\|_{\b{\Phi}(\b{W})}^2 = \|\b{k}_i - \b{k}_j\|_{\b{T}\b{T}^\top}^2 \\
&= (\b{k}_i - \b{k}_j)^\top \b{T}\b{T}^\top (\b{k}_i - \b{k}_j),
\end{aligned}
\end{equation}
where $\b{k}_i := \b{k}(\b{X}, \b{x}_i) = \b{\Phi}(\b{X})^\top \b{\phi}(\b{x}_i) = [k(\b{x}_1, \b{x}_i), \dots, k(\b{x}_n, \b{x}_i)]^\top \in \mathbb{R}^n$ is the kernel vector between $\b{X}$ and $\b{x}_i$.
\end{lemma}
\begin{proof}
We can have the decomposition of the weight matrix, i.e. Eq. (\ref{equation_W_U_UT}), in RKHS which is:
\begin{align}\label{equation_W_U_UT_RKHS}
\b{\Phi}(\b{W}) = \b{\Phi}(\b{U}) \b{\Phi}(\b{U})^\top.
\end{align}
The generalized Mahalanobis distance metric in RKHS is:
\begin{align*}
&\|\b{\phi}(\b{x}_i) - \b{\phi}(\b{x}_j)\|_{\b{\Phi}(\b{W})}^2 \\
&\overset{(\ref{equation_W_U_UT})}{=} (\b{\phi}(\b{x}_i) - \b{\phi}(\b{x}_j))^\top \b{\Phi}(\b{U}) \b{\Phi}(\b{U})^\top (\b{\phi}(\b{x}_i) - \b{\phi}(\b{x}_j)) \\
&= \big(\b{\Phi}(\b{U})^\top \b{\phi}(\b{x}_i) - \b{\Phi}(\b{U})^\top \b{\phi}(\b{x}_j)\big)^\top \\
& ~~~~~~~~~~~~~~~~~~~~~~ \big(\b{\Phi}(\b{U})^\top \b{\phi}(\b{x}_i) - \b{\Phi}(\b{U})^\top \b{\phi}(\b{x}_j)\big) \\
&\overset{(\ref{equation_kernelization_representation_theory})}{=} \big(\b{T}^\top \b{\Phi}(\b{X})^\top \b{\phi}(\b{x}_i) - \b{T}^\top \b{\Phi}(\b{X})^\top \b{\phi}(\b{x}_j)\big)^\top \\
& ~~~~~~~~~~~~~~ \big(\b{T}^\top \b{\Phi}(\b{X})^\top \b{\phi}(\b{x}_i) - \b{T}^\top \b{\Phi}(\b{X})^\top \b{\phi}(\b{x}_j)\big) \\
&\overset{(a)}{=} \big(\b{T}^\top \b{k}_i - \b{T}^\top \b{k}_j\big)^\top \big(\b{T}^\top \b{k}_i - \b{T}^\top \b{k}_j\big) \\
&= \big(\b{k}_i - \b{k}_j\big)^\top \b{T}\b{T}^\top \big(\b{k}_i - \b{k}_j\big) = \|\b{k}_i - \b{k}_j\|_{\b{T}\b{T}^\top}^2, 
\end{align*}
where $(a)$ is because of the kernel trick, i.e., $\b{k}(\b{X}, \b{x}_i) = \b{\Phi}(\b{X})^\top \b{\phi}(\b{x}_i)$. Q.E.D. 
\end{proof}

Let $\b{\nu}_l := [\frac{1}{n_l} \sum_{i=1}^{n_l} \b{k}(\b{x}_1, \b{x}_i), \dots, \frac{1}{n_l} \sum_{i=1}^{n_l} \b{k}(\b{x}_n, \b{x}_i)]^\top \in \mathbb{R}^n$ where $n_l$ denotes the cardinality of the $l$-th class. 
Let $\b{K}_w$ and $\b{K}_b$ be the kernelized versions of $\b{S}_w$ and $\b{S}_b$, respectively (see Eqs. (\ref{equation_intra_cluster_scatter}) and (\ref{equation_inter_cluster_scatter})). If $\mathcal{X}_l$ denotes the $l$-th class, we have:
\begin{align}
& \mathbb{R}^{n \times n} \ni \b{K}_w := \frac{1}{n} \sum_{l=1}^c \sum_{\b{x}_i \in \mathcal{X}_l} (\b{k}_i - \b{\nu}_l) (\b{k}_i - \b{\nu}_l)^\top \\
&\mathbb{R}^{n \times n} \ni \b{K}_b := \frac{1}{n} \sum_{l=1}^c \sum_{j=1}^c (\b{\nu}_l - \b{\nu}_j) (\b{\nu}_l - \b{\nu}_j)^\top.
\end{align}
We saw the metric in RKHS can be seen as projection onto a subspace with the projection matrix $\b{T}$. Therefore, Eq. (\ref{equation_optimization_DCA}) in RKHS becomes \cite{hoi2006learning}:
\begin{equation}\label{equation_optimization_DCA_kernel}
\begin{aligned}
& \underset{\b{T}}{\text{maximize}}
& & \frac{\textbf{tr}(\b{T}^\top \b{K}_b \b{T})}{\textbf{tr}(\b{T}^\top \b{K}_w \b{T})},
\end{aligned}
\end{equation}
which is a generalized Rayleigh-Ritz quotient. The solution $\b{T}$ to this optimization problem is the generalized eigenvalue problem $(\b{K}_b, \b{K}_w)$ \cite{ghojogh2019eigenvalue}. 
The weight matrix of the generalized Mahalanobis distance is obtained by Eqs. (\ref{equation_kernelization_representation_theory}) and (\ref{equation_W_U_UT_RKHS}). 

\subsubsection{Relevant to Kernel Fisher Discriminant Analysis}

Here, we explain the kernel version of the metric learning method \cite{alipanahi2008distance} which was introduced in Section \ref{section_relation_to_FDA}. 

According to Eq. (\ref{equation_Mahalanobis_distance_in_RKHS}), we have:
\begin{align*}
&\|\b{\phi}(\b{x}_i) - \b{\phi}(\b{x}_j)\|_{\b{\Phi}(\b{W})}^2 = (\b{k}_i - \b{k}_j)^\top \b{T}\b{T}^\top (\b{k}_i - \b{k}_j) \\
&~~~~~~~~~~~~~~~~~~\overset{(a)}{=} \textbf{tr}\big((\b{k}_i - \b{k}_j)^\top \b{T}\b{T}^\top (\b{k}_i - \b{k}_j)\big) \\
&~~~~~~~~~~~~~~~~~~\overset{(b)}{=} \textbf{tr}\big(\b{T}^\top (\b{k}_i - \b{k}_j) (\b{k}_i - \b{k}_j)^\top \b{T}\big),
\end{align*}
where $(a)$ is because a scalar it equal to its trace and $(b)$ is because of the cyclic property of trace. 
Hence, Eq. (\ref{equation_Sigma_S_prime_Sigma_D_prime}) in RKHS becomes:
\begin{align*}
& \frac{1}{|\mathcal{S}|} \sum_{(\b{x}_i, \b{x}_j) \in \mathcal{S}} \textbf{tr}\big(\b{T}^\top (\b{k}_i - \b{k}_j) (\b{k}_i - \b{k}_j)^\top \b{T}\big) \\
& = \textbf{tr}\Big(\b{T}^\top \big(\frac{1}{|\mathcal{S}|} \sum_{(\b{x}_i, \b{x}_j) \in \mathcal{S}} (\b{k}_i - \b{k}_j) (\b{k}_i - \b{k}_j)^\top \b{T}\big)\Big) \\
&= \textbf{tr}(\b{T}^\top \b{\Sigma}^{\phi}_{\mathcal{S}} \b{T}), 
\end{align*}
and likewise:
\begin{align*}
&\frac{1}{|\mathcal{D}|} \sum_{(\b{x}_i, \b{x}_j) \in \mathcal{D}} \textbf{tr}\big(\b{T}^\top (\b{k}_i - \b{k}_j) (\b{k}_i - \b{k}_j)^\top \b{T}\big) \\
&~~~~~~~~~~~~~~~~~~~~~~ = \textbf{tr}(\b{T}^\top \b{\Sigma}^{\phi}_{\mathcal{D}} \b{T}),
\end{align*}
where:
\begin{align*}
& \b{\Sigma}^{\phi}_{\mathcal{S}} := \frac{1}{|\mathcal{S}|} \sum_{(\b{x}_i, \b{x}_j) \in \mathcal{S}} (\b{k}_i - \b{k}_j) (\b{k}_i - \b{k}_j)^\top, \\
& \b{\Sigma}^{\phi}_{\mathcal{D}} := \frac{1}{|\mathcal{D}|} \sum_{(\b{x}_i, \b{x}_j) \in \mathcal{D}} (\b{k}_i - \b{k}_j) (\b{k}_i - \b{k}_j)^\top.
\end{align*}
Hence, in RKHS, the objective of the optimization problem (\ref{equation_ML_optimization_spectral_method3_approach2}) becomes $\textbf{tr}(\b{T}^\top (\b{\Sigma}^{\phi}_{\mathcal{S}} - \b{\Sigma}^{\phi}_{\mathcal{D}}) \b{T}^\top)$. 
We change the constraint in Eq. (\ref{equation_ML_optimization_spectral_method3_approach2}) to $\b{U}^\top \b{U} = \b{I}$. In RKHS, this constraint becomes:
\begin{align*}
\b{\Phi}(\b{U})^\top \b{\Phi}(\b{U}) &\overset{(\ref{equation_kernelization_representation_theory})}{=} \b{T}^\top \b{\Phi}(\b{X})^\top \b{\Phi}(\b{X}) \b{T} \\
&\overset{(\ref{equation_Kernel_X})}{=} \b{T}^\top \b{K}_x \b{T} \overset{\text{set}}{=} \b{I}, 
\end{align*}
Finally, (\ref{equation_ML_optimization_spectral_method3_approach2}) in RKHS becomes:
\begin{equation}\label{equation_ML_optimization_spectral_method3_approach2_kernel}
\begin{aligned}
& \underset{\b{T}}{\text{minimize}}
& & \textbf{tr}(\b{T}^\top (\b{\Sigma}^{\phi}_{\mathcal{S}} - \b{\Sigma}^{\phi}_{\mathcal{D}}) \b{T}) \\
& \text{subject to} 
& & \b{T}^\top \b{K}_x \b{T} = \b{I},
\end{aligned}
\end{equation}
whose solution is a generalized eigenvalue problem $(\b{\Sigma}^{\phi}_{\mathcal{S}} - \b{\Sigma}^{\phi}_{\mathcal{D}}, \b{K}_x)$ where $\b{T}$ is the matrix of eigenvectors. 
The weight matrix of the generalized Mahalanobis distance is obtained by Eqs. (\ref{equation_kernelization_representation_theory}) and (\ref{equation_W_U_UT_RKHS}). 
This is relevant to kernel Fisher discriminant analysis \cite{mika1999fisher,ghojogh2019fisher} which minimizes and maximizes the intra-class and inter-class variances in RKHS. 



\subsubsection{Relevant to Kernel Support Vector Machine}

Here, we explain the kernel version of the metric learning method \cite{tsang2003distance} which was introduced in Section \ref{section_relation_to_SVM}. It is relevant to kernel SVM. 
Using kernel trick \cite{ghojogh2021reproducing} and Eq. (\ref{equation_Mahalanobis_distance_in_RKHS}), the Eq. (\ref{equation_relation_to_SVM_dual_optimization}) can be kernelized as \cite{tsang2003distance}:
\begin{equation}\label{equation_relation_to_SVM_dual_optimization_kernel}
\begin{aligned}
& \underset{\{\alpha_{ij}\}}{\text{maximize}}
~~~~~ \sum_{(\b{x}_i, \b{x}_j) \in \mathcal{D}} \alpha_{ij} \b{T}^\top (k_{ii} + k_{jj} - 2k_{ij}) \\
& -\frac{1}{2} \sum_{(\b{x}_i, \b{x}_j) \in \mathcal{D}} \sum_{(\b{x}_k, \b{x}_l) \in \mathcal{D}} \alpha_{ij} \alpha_{kl} (k_{ik} - k_{il} - k_{jk} + k_{jl})^2 \\
& + \frac{\lambda_1}{|\mathcal{S}|} \sum_{(\b{x}_i, \b{x}_j) \in \mathcal{D}} \sum_{(\b{x}_k, \b{x}_l) \in \mathcal{S}} \alpha_{ij} (k_{ik} - k_{il} - k_{jk} + k_{jl})^2 \\
& \text{subject to} 
~~~~~~~~~~ \frac{1}{\lambda_2} \sum_{(\b{x}_i, \b{x}_j) \in \mathcal{D}} \alpha_{ij} \geq \nu, \\
& ~~~~~~~~~~~~~~~~~~~~~~~~~~ \alpha_{ij} \in [0, \frac{\lambda_2}{|\mathcal{D}|}],
\end{aligned}
\end{equation}
which is a quadratic programming problem and can be solved by optimization solvers. 

\subsection{Geometric Spectral Metric Learning}

Some spectral metric learning methods are geometric methods which use Riemannian manifolds. In the following, we introduce the mist well-known geometric methods. There are some other geometric methods, such as \cite{hauberg2012geometric}, which are not covered for brevity. 

\subsubsection{Geometric Mean Metric Learning}\label{section_geometric_mean_metric_learning}

One of the geometric spectral metric learning is Geometric Mean Metric Learning (GMML) \cite{zadeh2016geometric}. 
Let $\b{W}$ be the weight matrix in the generalized Mahalanobis distance for similar points. 

\hfill\break
\textbf{-- Regular GMML:}
In GMML, we use the inverse of weight matrix, i.e. $\b{W}^{-1}$ , for the dissimilar points. 
The optimization problem of GMML is \cite{zadeh2016geometric}:
\begin{equation}\label{equation_GMML_optimization}
\begin{aligned}
& \underset{\b{W}}{\text{minimize}}
& & \sum_{(\b{x}_i, \b{x}_j) \in \mathcal{S}} \|\b{x}_i - \b{x}_j\|_{\b{W}}^2 \\ 
& & & ~~~~~~~~~~~~~~~~~~ + \sum_{(\b{x}_i, \b{x}_j) \in \mathcal{D}} \|\b{x}_i - \b{x}_j\|_{\b{W}^{-1}}^2 \\
& \text{subject to}
& & \b{W} \succeq \b{0}.
\end{aligned}
\end{equation}
According to Eq. (\ref{equation_spectral_ML_first_method_trace_W_Sigma_S}), this problem can be restated as:
\begin{equation}\label{equation_GMML_optimization_1}
\begin{aligned}
& \underset{\b{W}}{\text{minimize}}
& & \textbf{tr}(\b{W} \b{\Sigma}_{\mathcal{S}}) + \textbf{tr}(\b{W}^{-1} \b{\Sigma}_{\mathcal{D}}) \\
& \text{subject to}
& & \b{W} \succeq \b{0},
\end{aligned}
\end{equation}
where $\b{\Sigma}_{\mathcal{S}}$ and $\b{\Sigma}_{\mathcal{D}}$ are defined in Eq. (\ref{equation_spectral_ML_first_method_Sigma_S}). 
Taking derivative of the objective function w.r.t. $\b{W}$ and setting it to zero gives:
\begin{align}
&\frac{\partial }{\partial \b{W}} \big( \textbf{tr}(\b{W} \b{\Sigma}_{\mathcal{S}}) + \textbf{tr}(\b{W}^{-1} \b{\Sigma}_{\mathcal{D}}) \big) \nonumber\\
&= \b{\Sigma}_{\mathcal{S}} - \b{W}^{-1} \b{\Sigma}_{\mathcal{D}} \b{W}^{-1} \overset{\text{set}}{=} \b{0} \implies \b{\Sigma}_{\mathcal{D}} =\b{W} \b{\Sigma}_{\mathcal{S}} \b{W}. \label{equation_GMML_solution_1}
\end{align}
This equation is the Riccati equation \cite{riccati1724animadversiones} and its solution is the midpoint of the geodesic connecting $\b{\Sigma}_{\mathcal{S}}^{-1}$ and $\b{\Sigma}_{\mathcal{D}}$ {\citep[Section 1.2.13]{bhatia2007positive}}. 

\begin{lemma}[{\citep[Chapter 6]{bhatia2007positive}}]
The geodesic curve connecting two points $\b{\Sigma}_1$ and $\b{\Sigma}_2$ on the Symmetric Positive Definite (SPD) Riemannian manifold is denoted by $\b{\Sigma}_1 \sharp_t \b{\Sigma}_2$ and is computed as:
\begin{align}\label{equation_SPD_geodesic}
\b{\Sigma}_1 \sharp_t \b{\Sigma}_2 := \b{\Sigma}_1^{(1/2)} \big(\b{\Sigma}_1^{(-1/2)} \b{\Sigma}_2 \b{\Sigma}_1^{(-1/2)}\big)^t \b{\Sigma}_1^{(1/2)}, 
\end{align}
where $t \in [0,1]$. 
\end{lemma}

Hence, the solution of Eq. (\ref{equation_GMML_solution_1}) is:
\begin{align}
\b{W} &= \b{\Sigma}_{\mathcal{S}}^{-1} \sharp_{(1/2)} \b{\Sigma}_{\mathcal{D}} \nonumber\\
&\overset{(\ref{equation_SPD_geodesic})}{=} \b{\Sigma}_{\mathcal{S}}^{(-1/2)} \big(\b{\Sigma}_{\mathcal{S}}^{(1/2)} \b{\Sigma}_{\mathcal{D}} \b{\Sigma}_{\mathcal{S}}^{(1/2)}\big)^{(1/2)} \b{\Sigma}_{\mathcal{S}}^{(-1/2)}. \label{equation_GMML_solution_1_1}
\end{align}
The proof of Eq. (\ref{equation_GMML_solution_1_1}) is as follows \cite{hajiabadi2019layered}:
\begin{align*}
&\b{\Sigma}_{\mathcal{D}} \overset{(\ref{equation_GMML_solution_1})}{=} \b{W} \b{\Sigma}_{\mathcal{S}} \b{W} \\
&\implies \b{\Sigma}_{\mathcal{S}}^{(1/2)} \b{\Sigma}_{\mathcal{D}} \b{\Sigma}_{\mathcal{S}}^{(1/2)} = \b{\Sigma}_{\mathcal{S}}^{(1/2)} \b{W} \b{\Sigma}_{\mathcal{S}} \b{W} \b{\Sigma}_{\mathcal{S}}^{(1/2)} \\
&\implies (\b{\Sigma}_{\mathcal{S}}^{(1/2)} \b{\Sigma}_{\mathcal{D}} \b{\Sigma}_{\mathcal{S}}^{(1/2)})^{(1/2)} \\
&~~~~~~~~~~~~~~~~~~~~~~~~~~~ = (\b{\Sigma}_{\mathcal{S}}^{(1/2)} \b{W} \b{\Sigma}_{\mathcal{S}} \b{W} \b{\Sigma}_{\mathcal{S}}^{(1/2)})^{(1/2)} \\
&\implies (\b{\Sigma}_{\mathcal{S}}^{(1/2)} \b{\Sigma}_{\mathcal{D}} \b{\Sigma}_{\mathcal{S}}^{(1/2)})^{(1/2)} \\
&~~~~~~~~~~~~~~~~ \overset{(a)}{=} ((\b{\Sigma}_{\mathcal{S}}^{(1/2)} \b{W} \b{\Sigma}_{\mathcal{S}}^{(1/2)}) (\b{\Sigma}_{\mathcal{S}}^{(1/2)} \b{W} \b{\Sigma}_{\mathcal{S}}^{(1/2)}))^{(1/2)} \\
&~~~~~~~~~~~~~~~~ = (\b{\Sigma}_{\mathcal{S}}^{(1/2)} \b{W} \b{\Sigma}_{\mathcal{S}}^{(1/2)}) \\
&\implies \b{\Sigma}_{\mathcal{S}}^{(-1/2)} (\b{\Sigma}_{\mathcal{S}}^{(1/2)} \b{\Sigma}_{\mathcal{D}} \b{\Sigma}_{\mathcal{S}}^{(1/2)})^{(1/2)} \b{\Sigma}_{\mathcal{S}}^{(-1/2)} \\
&~~~~~~~~~~~~~~~~ = \b{\Sigma}_{\mathcal{S}}^{(-1/2)} (\b{\Sigma}_{\mathcal{S}}^{(1/2)} \b{W} \b{\Sigma}_{\mathcal{S}}^{(1/2)}) \b{\Sigma}_{\mathcal{S}}^{(-1/2)} = \b{W},
\end{align*}
where $(a)$ is because $\b{\Sigma}_{\mathcal{S}} \succeq \b{0}$ so its eigenvalues are non-negative and the matrix of eigenvalues can be decomposed by the second root in its eigenvalue decomposition to have $\b{\Sigma}_{\mathcal{S}} = \b{\Sigma}_{\mathcal{S}}^{(1/2)} \b{\Sigma}_{\mathcal{S}}^{(1/2)}$.

\hfill\break
\textbf{-- Regularized GMML:}
The matrix $\b{\Sigma}_{\mathcal{S}}$ might be singular or near singular and hence non-invertible. Therefore, we regularize Eq. (\ref{equation_GMML_optimization_1}) to make the weight matrix close to a prior known positive definite matrix $\b{W}_0$. 
\begin{equation}\label{equation_GMML_optimization_2}
\begin{aligned}
& \underset{\b{W}}{\text{minimize}}
& & \textbf{tr}(\b{W} \b{\Sigma}_{\mathcal{S}}) + \textbf{tr}(\b{W}^{-1} \b{\Sigma}_{\mathcal{D}}) \\
& & & + \lambda \big( \textbf{tr}(\b{W} \b{W}_0^{-1}) + \textbf{tr}(\b{W}^{-1} \b{W}_0) - 2 d \big), \\
& \text{subject to}
& & \b{W} \succeq \b{0},
\end{aligned}
\end{equation}
where $\lambda>0$ is the regularization parameter. The regularization term is the symmetrized log-determinant divergence between $\b{W}$ and $\b{W}_0$. 
Taking derivative of the objective function w.r.t. $\b{W}$ and setting it to zero gives:
\begin{align*}
&\frac{\partial }{\partial \b{W}} \big( \textbf{tr}(\b{W} \b{\Sigma}_{\mathcal{S}}) + \textbf{tr}(\b{W}^{-1} \b{\Sigma}_{\mathcal{D}}) + \lambda \textbf{tr}(\b{W} \b{W}_0^{-1}) \\
&~~~~~~~~~ + \lambda \textbf{tr}(\b{W}^{-1} \b{W}_0) - 2\lambda d \big) \nonumber\\
&= \b{\Sigma}_{\mathcal{S}} - \b{W}^{-1} \b{\Sigma}_{\mathcal{D}} \b{W}^{-1} + \lambda \b{W}_0^{-1} \\
& ~~~~~~~~~~~~~~~~~~~~~~ + \lambda \b{W}^{-1} \b{W}_0 \b{W}^{-1} \overset{\text{set}}{=} \b{0} \\
& \implies \b{\Sigma}_{\mathcal{D}} + \lambda \b{W}_0 =\b{W} (\b{\Sigma}_{\mathcal{S}} + \lambda \b{W}_0^{-1}) \b{W},
\end{align*}
which is again a Riccati equation \cite{riccati1724animadversiones} whose solution is the midpoint of the geodesic connecting $(\b{\Sigma}_{\mathcal{S}} + \lambda \b{W}_0^{-1})^{-1}$ and $(\b{\Sigma}_{\mathcal{D}} + \lambda \b{W}_0)$:
\begin{align}\label{equation_GMML_solution_2}
\b{W} &= (\b{\Sigma}_{\mathcal{S}} + \lambda \b{W}_0^{-1})^{-1} \sharp_{(1/2)} (\b{\Sigma}_{\mathcal{D}} + \lambda \b{W}_0).
\end{align}

\textbf{-- Weighted GMML:}
Eq. (\ref{equation_GMML_optimization_1}) can be restated as:
\begin{equation}\label{equation_GMML_optimization_3}
\begin{aligned}
& \underset{\b{W}}{\text{minimize}}
& & \delta^2(\b{W}, \b{\Sigma}_{\mathcal{S}}^{-1}) + \delta^2(\b{W}, \b{\Sigma}_{\mathcal{D}}) \\
& \text{subject to}
& & \b{W} \succeq \b{0},
\end{aligned}
\end{equation}
where $\delta(.,.)$ is the Riemannian distance (or Fr{\'e}chet mean) on the SPD manifold {\citep[Eq 1.1]{arsigny2007geometric}}:
\begin{align*}
\delta(\b{\Sigma}_1, \b{\Sigma}_2) := \|\log(\b{\Sigma}_2^{(-1/2)} \b{\Sigma}_1 \b{\Sigma}_2^{(-1/2)})\|_F,
\end{align*}
where $\|.\|_F$ is the Frobenius norm. 
We can weight the objective in Eq. (\ref{equation_GMML_optimization_3}):
\begin{equation}\label{equation_GMML_optimization_4}
\begin{aligned}
& \underset{\b{W}}{\text{minimize}}
& & (1-t) \delta^2(\b{W}, \b{\Sigma}_{\mathcal{S}}^{-1}) + t \delta^2(\b{W}, \b{\Sigma}_{\mathcal{D}}) \\
& \text{subject to}
& & \b{W} \succeq \b{0},
\end{aligned}
\end{equation}
where $t \in [0,1]$ is a hyperparameter. 
The solution of this problem is the weighted version of Eq. (\ref{equation_GMML_solution_2}):
\begin{align}\label{equation_GMML_solution_3}
\b{W} &= (\b{\Sigma}_{\mathcal{S}} + \lambda \b{W}_0^{-1})^{-1} \sharp_t (\b{\Sigma}_{\mathcal{D}} + \lambda \b{W}_0).
\end{align}

\subsubsection{Low-rank Geometric Mean Metric Learning}

We can learn a low-rank weight matrix in GMML \cite{bhutani2018low}, where the rank of wight matrix is set to be $p \ll d$:
\begin{equation}\label{equation_GMML_lowRank_optimization_1}
\begin{aligned}
& \underset{\b{W}}{\text{minimize}}
& & \textbf{tr}(\b{W} \b{\Sigma}_{\mathcal{S}}) + \textbf{tr}(\b{W}^{-1} \b{\Sigma}_{\mathcal{D}}) \\
& \text{subject to}
& & \b{W} \succeq \b{0}, \\
& & & \textbf{rank}(\b{W}) = p.
\end{aligned}
\end{equation}
We can decompose it using eigenvalue decomposition as done in Eq. (\ref{equation_W_U_UT}), i.e., $\b{W} = \b{V} \b{\Lambda} \b{V}^\top = \b{U} \b{U}^\top$, where we only have $p$ eigenvectors and $p$ eigenvalues. Therefore, the sizes of matrices are $\b{V} \in \mathbb{R}^{d \times p}$, $\b{\Lambda} \in \mathbb{R}^{p \times p}$, and $\b{U} \in \mathbb{R}^{d \times p}$. 
By this decomposition, the objective function in Eq. (\ref{equation_GMML_lowRank_optimization_1} can be restated as:
\begin{align*}
&\textbf{tr}(\b{V} \b{\Lambda} \b{V}^\top \b{\Sigma}_{\mathcal{S}}) + \textbf{tr}(\b{V} \b{\Lambda}^{-1} \b{V}^\top \b{\Sigma}_{\mathcal{D}}) \\
&\overset{(a)}{=} \textbf{tr}(\b{\Lambda} \b{V}^\top \b{\Sigma}_{\mathcal{S}} \b{V}) + \textbf{tr}(\b{\Lambda}^{-1} \b{V}^\top \b{\Sigma}_{\mathcal{D}} \b{V}) \\
&\overset{(b)}{=} \textbf{tr}(\b{\Lambda} \widetilde{\b{\Sigma}}_{\mathcal{S}}) + \textbf{tr}(\b{\Lambda}^{-1} \widetilde{\b{\Sigma}}_{\mathcal{D}}),
\end{align*}
where $(\b{V}^\top)^{-1} = \b{V}$ because it is orthogonal, $(a)$ is because of the cyclic property of trace, and $(b)$ is because we define $\widetilde{\b{\Sigma}}_{\mathcal{S}} := \b{V}^\top \b{\Sigma}_{\mathcal{S}} \b{V}$ and $\widetilde{\b{\Sigma}}_{\mathcal{D}} := \b{V}^\top \b{\Sigma}_{\mathcal{D}} \b{V}$. 
Noticing that the matrix of eigenvectors $\b{V}$ is orthogonal, the Eq. (\ref{equation_GMML_lowRank_optimization_1}) is restated to:
\begin{equation}\label{equation_GMML_lowRank_optimization_2}
\begin{aligned}
& \underset{\b{\Lambda}, \b{V}}{\text{minimize}}
& & \textbf{tr}(\b{\Lambda} \widetilde{\b{\Sigma}}_{\mathcal{S}}) + \textbf{tr}(\b{\Lambda}^{-1} \widetilde{\b{\Sigma}}_{\mathcal{D}}) \\
& \text{subject to}
& & \b{\Lambda} \succeq \b{0}, \\
& & & \b{V}^\top \b{V} = \b{I},
\end{aligned}
\end{equation}
where $\textbf{rank}(\b{W}) = p$ is automatically satisfied by taking $\b{V} \in \mathbb{R}^{d \times p}$ and $\b{\Lambda} \in \mathbb{R}^{p \times p}$ in the decomposition. 
This problem can be solved by the alternative optimization \cite{ghojogh2021kkt}. If the variable $\b{V}$ is fixed, minimization w.r.t. $\b{\Lambda}$ is similar to the problem (\ref{equation_GMML_optimization_1}); hence, its solution is similar to Eq. (\ref{equation_GMML_solution_1_1}), i.e., $\b{\Lambda} = {\widetilde{\b{\Sigma}}_{\mathcal{S}}}^{-1} \sharp_{(1/2)} \widetilde{\b{\Sigma}}_{\mathcal{D}}$ (see Eq. (\ref{equation_SPD_geodesic}) for the definition of $\sharp_t$). If $\b{\Lambda}$ is fixed, the orthogonality constraint $\b{V}^\top \b{V} = \b{I}$ can be modeled by $\b{V}$ belonging to the Grassmannian manifold $G(p,d)$ which is the set of $p$-dimensional subspaces of $\mathbb{R}^d$. 
To sum up, the alternative optimization is:
\begin{align*}
& \b{\Lambda}^{(\tau+1)} = (\b{V}^{(\tau)\top} \b{\Sigma}_{\mathcal{S}} \b{V}^{(\tau)})^{-1} \sharp_{(1/2)} (\b{V}^{(\tau)\top} \b{\Sigma}_{\mathcal{D}} \b{V}^{(\tau)}), \\
& \b{V}^{(\tau+1)} := \arg \min_{\b{V} \in G(p,d)} \Big( \textbf{tr}(\b{\Lambda}^{(\tau+1)} \b{V}^\top \b{\Sigma}_{\mathcal{S}} \b{V}) \\
&~~~~~~~~~~~~~~~~~~~~~~~~~~~~~~~~~~~~~~~~~~ + \textbf{tr}((\b{\Lambda}^{(\tau+1)})^{-1} \b{V}^\top \b{\Sigma}_{\mathcal{D}} \b{V})\Big), 
\end{align*}
where $\tau$ is the iteration index. 
Optimization of $\b{V}$ can be solved by Riemannian optimization \cite{absil2009optimization}.

\subsubsection{Geometric Mean Metric Learning for Partial Labels}

Partial label learning \cite{cour2011learning} refers to when a set of candidate labels is available for every data point. 
GMML can be modified to be used for partial label learning \cite{zhou2018geometric}. 
Let $\mathcal{Y}_i$ denote the set of candidate labels for $\b{x}_i$. 
If there are $q$ candidate labels in total, we denote $\b{y}_i = [y_{i1}, \dots, y_{iq}]^\top \in \{0,1\}^q$ where $y_{ij}$ is one if the $j$-th label is a candidate label for $\b{x}_i$ and is zero otherwise.
We define $\b{X}_i^+ := \{\b{x}_j | j=1, \dots, n, j \neq i, \mathcal{Y}_i \cap \mathcal{Y}_j \neq \varnothing\}$ and $\b{X}_i^- := \{\b{x}_j | j=1, \dots, n, \mathcal{Y}_i \cap \mathcal{Y}_j = \varnothing\}$. 
In other words, $\b{X}_i^+$ and $\b{X}_i^-$ are the data points which share and do not share some candidate labels with $\b{x}_i$, respectively. 
Let $\mathcal{N}_i^+$ be the indices of the $k$ nearest neighbors of $\b{x}_i$ among $\b{X}_i^+$. Also, let $\mathcal{N}_i^-$ be the indices of points in $\b{X}_i^-$ whose distance from $\b{x}_i$ are smaller than the distance of the furthest point in $\mathcal{N}_i^+$ from $\b{x}_i$. In other words, $\mathcal{N}_i^- := \{j | j=1, \dots, n, \b{x}_j \in \b{X}_i^-, \|\b{x}_i - \b{x}_j\|_2 \leq \max_{t \in \mathcal{N}_i^+} \|\b{x}_i - \b{x}_t\|_2\}$.

Let $\b{w}_i^{(1)} = [w_{i,t}^{(1)}, \forall t \in \mathcal{N}_i^+]^\top \in \mathbb{R}^k$ contain the probabilities that each of the $k$ neighbors of $\b{x}_i$ share the same label with $\b{x}_i$. It can be estimated by linear reconstruction of $\b{y}_i$ by the neighbor $\b{y}_t$'s:
\begin{equation*}
\begin{aligned}
& \underset{\b{w}_i^{(1)}}{\text{minimize}}
& & \frac{1}{q} \big\|\b{y}_i - \sum_{t \in \mathcal{N}_i^+} w_{i,t}^{(1)} \b{y}_t\big\|_2^2 + \frac{\lambda_1}{k} \sum_{t \in \mathcal{N}_i^+} (w_{i,t}^{(1)})^2 \\ 
& \text{subject to}
& & w_{i,t}^{(1)} \geq 0, \quad t \in \mathcal{N}_i^+, 
\end{aligned}
\end{equation*}
where $\lambda_1 > 0$ is the regularization parameter. 
Let $\b{w}_i^{(2)} = [w_{i,t}^{(2)}, \forall t \in \mathcal{N}_i^+]^\top \in \mathbb{R}^k$ denote the coefficients for linear reconstruction of $\b{x}_i$ by its $k$ nearest neighbors. It is obtained as:
\begin{equation*}
\begin{aligned}
& \underset{\b{w}_i^{(2)}}{\text{minimize}}
& & \big\|\b{x}_i - \sum_{t \in \mathcal{N}_i^+} w_{i,t}^{(2)} \b{x}_t\big\|_2^2 \\ 
& \text{subject to}
& & w_{i,t}^{(2)} \geq 0, \quad t \in \mathcal{N}_i^+.
\end{aligned}
\end{equation*}
These two optimization problems are quadratic programming and can be solved using the interior point method \cite{ghojogh2021kkt}. 

The main optimization problem of GMML for partial labels is \cite{zhou2018geometric}:
\begin{equation}\label{equation_optimization_GMML_partial_labels}
\begin{aligned}
& \underset{\b{W}}{\text{minimize}}
& & \textbf{tr}(\b{W}\b{\Sigma}'_\mathcal{S}) + \textbf{tr}(\b{W}^{-1} \b{\Sigma}'_\mathcal{D}) \\ 
& \text{subject to}
& & \b{W} \succeq \b{0},
\end{aligned}
\end{equation}
where:
\begin{align*}
& \b{\Sigma}'_\mathcal{S} := \sum_{i=1}^n \Bigg( \frac{\sum_{t \in \mathcal{N}_i^+} w_{i,t}^{(1)} (\b{x}_i - \b{x}_t) (\b{x}_i - \b{x}_t)^\top}{\sum_{t \in \mathcal{N}_i^+} w_{i,t}^{(1)}} \\
& ~~~~~~~~~~ + \lambda \Big(\b{x}_i - \sum_{t \in \mathcal{N}_i^+} w_{i,t}^{(2)} \b{x}_t\Big) \Big(\b{x}_i - \sum_{t \in \mathcal{N}_i^+} w_{i,t}^{(2)} \b{x}_t\Big)^\top \Bigg), \\
& \b{\Sigma}'_\mathcal{D} := \sum_{i=1}^n \sum_{t \in \mathcal{N}_i^-} (\b{x}_i - \b{x}_t) (\b{x}_i - \b{x}_t)^\top. 
\end{align*}
Minimizing the first term of $\b{\Sigma}'_\mathcal{S}$ in $\textbf{tr}(\b{W}\b{\Sigma}'_\mathcal{S})$ decreases the distances of similar points which share some candidate labels. Minimizing the second term of $\b{\Sigma}'_\mathcal{S}$ in $\textbf{tr}(\b{W}\b{\Sigma}'_\mathcal{S})$ tries to preserve linear reconstruction of $\b{x}_i$ by its neighbors after projection onto the subspace of metric. 
Minimizing $\textbf{tr}(\b{W}^{-1} \b{\Sigma}'_\mathcal{D})$ increases the the distances of dissimilar points which do not share any candidate labels. 
The problem (\ref{equation_optimization_GMML_partial_labels}) is similar to the problem (\ref{equation_GMML_optimization_1}); hence, its solution is similar to Eq. (\ref{equation_GMML_solution_1_1}), i.e., $\b{W} = {\b{\Sigma}'_{\mathcal{S}}}^{-1} \sharp_{(1/2)} \b{\Sigma}'_{\mathcal{D}}$ (see Eq. (\ref{equation_SPD_geodesic}) for the definition of $\sharp_t$).

\subsubsection{Geometric Mean Metric Learning on SPD and Grassmannian Manifolds}

The GMML method \cite{zadeh2016geometric}, introduced in Section \ref{section_geometric_mean_metric_learning}, can be implemented on Symmetric Positive Definite (SPD) and Grassmannian manifolds \cite{zhu2018towards}. 
If $\b{X}_i, \b{X}_j \in \mathcal{S}_{++}^d$ is a point on the SPD manifold, the distance metric on this manifold is \cite{zhu2018towards}:
\begin{align}
d_{\b{W}}(\b{T}_i, \b{T}_j) := \textbf{tr}\big(\b{W} (\b{T}_i - \b{T}_j) (\b{T}_i - \b{T}_j)\big),
\end{align}
where $\b{W} \in \mathbb{R}^{d \times d}$ is the weight matrix of metric and $\b{T}_i := \log(\b{X}_i)$ is the logarithm operation on the SPD manifold. 
The Grassmannian manifold $Gr(k,d)$ is the $k$-dimensional subspaces of the $d$-dimensional vector space. 
A point in $Gr(k,d)$ is a linear subspace spanned by a full-rank $\b{X}_i \in \mathbb{R}^{d \times k}$ which is orthogonal, i.e., $\b{X}_i^\top \b{X}_i = \b{I}$. If $\b{M} \in \mathbb{R}^{d \times r}$ is any matrix, We define $\b{X}'_i$ in a way that $\b{M}^\top \b{X}'_i$ is the orthogonal components of $\b{M}^\top \b{X}_i$. 
If $\mathbb{R}^{d \times d} \ni \b{T}_{ij} := \b{X}'_i \b{X}^{'\top}_i - \b{X}'_j \b{X}^{'\top}_j$, the distance on the Grassmannian manifold is \cite{zhu2018towards}:
\begin{align}
d_{\b{W}}(\b{T}_{ij}) := \textbf{tr}\big(\b{W} \b{T}_{ij} \b{T}_{ij}\big),
\end{align}
$\b{W} \in \mathbb{R}^{d \times d}$ is the weight matrix of metric.

Similar to the optimization problem of GMML, i.e. Eq. (\ref{equation_GMML_optimization}), we solve the following problem for the SPD manifold:
\begin{equation}
\begin{aligned}
& \underset{\b{W}}{\text{minimize}}
& & \sum_{(\b{T}_i, \b{T}_j) \in \mathcal{S}} \textbf{tr}\big(\b{W} (\b{T}_i - \b{T}_j) (\b{T}_i - \b{T}_j)\big) \\
& & & + \sum_{(\b{T}_i, \b{T}_j) \in \mathcal{D}} \textbf{tr}\big(\b{W}^{-1} (\b{T}_i - \b{T}_j) (\b{T}_i - \b{T}_j)\big) \\
& \text{subject to}
& & \b{W} \succeq \b{0}.
\end{aligned}
\end{equation}
Likewise, for the Grassmannian manifold, the optimization problem is:
\begin{equation}
\begin{aligned}
& \underset{\b{W}}{\text{minimize}}
& & \sum_{(\b{T}_i, \b{T}_j) \in \mathcal{S}} \textbf{tr}\big(\b{W} \b{T}_{ij} \b{T}_{ij}\big) \\
& & & + \sum_{(\b{T}_i, \b{T}_j) \in \mathcal{D}} \textbf{tr}\big(\b{W}^{-1} \b{T}_{ij} \b{T}_{ij}\big) \\
& \text{subject to}
& & \b{W} \succeq \b{0}.
\end{aligned}
\end{equation}
Suppose, for the SPD manifold, we define:
\begin{align*}
& \b{\Sigma}'_{\mathcal{S}} := \sum_{(\b{T}_i, \b{T}_j) \in \mathcal{S}} (\b{T}_i - \b{T}_j) (\b{T}_i - \b{T}_j), \\
& \b{\Sigma}'_{\mathcal{D}} := \sum_{(\b{T}_i, \b{T}_j) \in \mathcal{D}} (\b{T}_i - \b{T}_j) (\b{T}_i - \b{T}_j).
\end{align*}
and, for the Grassmannian manifold, we define:
\begin{align*}
& \b{\Sigma}'_{\mathcal{S}} := \sum_{(\b{T}_i, \b{T}_j) \in \mathcal{S}} \b{T}_{ij} \b{T}_{ij}, \\
& \b{\Sigma}'_{\mathcal{D}} := \sum_{(\b{T}_i, \b{T}_j) \in \mathcal{D}} \b{T}_{ij} \b{T}_{ij}.
\end{align*}
Hence, for either SPD or Grassmannian manifold, the optimization problem becomes Eq. (\ref{equation_GMML_optimization_1}) in which $\b{\Sigma}_{\mathcal{S}}$ and $\b{\Sigma}_{\mathcal{D}}$ are replaced with $\b{\Sigma}'_{\mathcal{S}}$ and $\b{\Sigma}'_{\mathcal{D}}$, respectively. 

\subsubsection{Metric Learning on Stiefel and SPD Manifolds}

According to Eq. (\ref{equation_W_U_UT}), the weight matrix in the metric can be decomposed as $\b{W} = \b{V} \b{\Lambda} \b{V}^\top$. If we do not restrict $\b{V}$ and $\b{\Lambda}$ to be the matrices of eigenvectors and eigenvalues as in Eq. (\ref{equation_W_U_UT}), we can learn both $\b{V} \in \mathbb{R}^{d \times p}$ and $\b{\Lambda} \in \mathbb{R}^{p \times p}$ by optimization \cite{harandi2017joint}. 
The optimization problem in this method is:
\begin{equation}
\begin{aligned}
& \underset{\b{V}, \b{\Lambda}}{\text{minimize}}
& & \sum_{(\b{x}_i, \b{x}_j) \in \mathcal{S}} \log(1 + q_{ij}) \\
& & &+ \sum_{(\b{x}_i, \b{x}_j) \in \mathcal{D}} \log(1 + q_{ij}^{-1}) \\
& & &+\lambda \Big( \textbf{tr}(\b{\Lambda} \b{\Lambda}_0^{-1}) - \log\big(\textbf{det}(\b{\Lambda} \b{\Lambda}_0^{-1})\big) - p \Big) \\ 
& \text{subject to}
& & \b{V}^\top \b{V} = \b{I}, \\
& & & \b{\Lambda} \succeq \b{0}, 
\end{aligned}
\end{equation}
where $\lambda>0$ is the regularization parameter, $\textbf{det}(.)$ denotes the determinant of matrix, and $q_{ij}$ models Gaussian distribution with the generalized Mahalanobis distance metric:
\begin{align*}
& q_{ij} := \exp(\|\b{x}_i - \b{x}_j\|_{\b{V} \b{\Lambda} \b{V}^\top}).
\end{align*}
The constraint $\b{V}^\top \b{V} = \b{I}$ means that the matrix $\b{V}$ belongs to the Stiefel manifold $\text{St}(p,d) := \{\b{V} \in \mathbb{R}^{d \times p} | \b{V}^\top \b{V} = \b{I}\}$ and the constraint $\b{\Lambda} \succeq \b{0}$ means $\b{\Lambda}$ belongs to the SPD manifold $\mathcal{S}^p_{++}$. 
Hence, these two variables belong to the product manifold $\text{St}(p,d) \times \mathcal{S}^p_{++}$. 
Hence, we can solve this optimization problem using Riemannian optimization methods \cite{absil2009optimization}. This method can also be kernelized; the reader can refer to {\citep[Section 4]{harandi2017joint}} for its kernel version. 



\subsubsection{Curvilinear Distance Metric Learning (CDML)}

\begin{lemma}[\cite{chen2019curvilinear}]
The generalized Mahalanobis distance can be restated as:
\begin{align}\label{equation_metric_int_arc_length}
& \|\b{x}_i - \b{x}_j\|_{\b{W}}^2 = \sum_{l=1}^p \|\b{u}_l\|_2^2 \Big( \int_{T_l(\b{x}_i)}^{T_l(\b{x}_j)} \|\b{u}_l\|_2\, dt \Big)^2,
\end{align}
where $\b{u}_l \in \mathbb{R}^d$ is the $l$-th column of $\b{U}$ in Eq. (\ref{equation_W_U_UT}), $t \in \mathbb{R}$, and $T_l(\b{x}) \in \mathbb{R}$ is the projection of $\b{x}$ satisfying $(\b{u}_l T_l(\b{x}) - \b{x})^\top \b{u}_l = 0$. 
\end{lemma}
\begin{proof}
\begin{align*}
& \|\b{x}_i - \b{x}_j\|_{\b{W}}^2 = (\b{x}_i - \b{x}_j)^\top \b{W} (\b{x}_i - \b{x}_j) \\
&\overset{(\ref{equation_W_U_UT})}{=} (\b{x}_i - \b{x}_j)^\top \b{U}\b{U}^\top (\b{x}_i - \b{x}_j) = \|\b{U}^\top (\b{x}_i - \b{x}_j)\|_2^2 \\
&= \|[\b{u}_1^\top (\b{x}_i - \b{x}_j), \dots, \b{u}_p^\top (\b{x}_i - \b{x}_j)]^\top\|_2^2 \\
&= \sum_{l=1}^p \big( \b{u}_l^\top (\b{x}_i - \b{x}_j) \big)^2 \\
&\overset{(a)}{=} \sum_{l=1}^p \|\b{u}_l\|_2^2 \|\b{x}_i - \b{x}_j\|_2^2 \cos^2(\b{u}_l, \b{x}_i - \b{x}_j) \\
&\overset{(b)}{=} \sum_{l=1}^p \|\b{u}_l\|_2^2 \|\b{u}_l T_l(\b{x}_i) - \b{u}_l T_l(\b{x}_j)\|_2^2,
\end{align*}
where $(a)$ is because of the law of cosines and $(b)$ is because of $(\b{u}_l T_l(\b{x}) - \b{x})^\top \b{u}_l = 0$. 
The distance $\|\b{u}_l T_l(\b{x}_i) - \b{u}_l T_l(\b{x}_j)\|_2$ can be replaced by the length of the arc between $T_l(\b{x}_i)$ and $T_l(\b{x}_j)$ on the straight line $\b{u}_l t$ for $t \in \mathbb{R}$. This gives the Eq. (\ref{equation_metric_int_arc_length}). Q.E.D.
\end{proof}

The condition $(\b{u}_l T_l(\b{x}) - \b{x})^\top \b{u}_l = 0$ is equivalent to finding the nearest neighbor to the line $\b{u}_l t, \forall t \in \mathbb{R}$, i.e., $T_l(\b{x}) := \arg \min_{t \in \mathbb{R}} \|\b{u}_l t - \b{x}\|_2^2$ \cite{chen2019curvilinear}. This equation can be generalized to find the nearest neighbor to the geodesic curve $\b{\theta}_l(t)$ rather than the line $\b{u}_l t$:
\begin{align}
T_{\b{\theta}_l}(\b{x}) := \arg \min_{t \in \mathbb{R}} \|\b{\theta}_l(t) - \b{x}\|_2^2.
\end{align}
Hence, we can replace the arc length of the straight line in Eq. (\ref{equation_metric_int_arc_length}) with the arc length of the curve:
\begin{align}\label{equation_metric_int_arc_length_2}
& \|\b{x}_i - \b{x}_j\|_{\b{W}}^2 = \sum_{l=1}^p \alpha_l \Big( \int_{T_{\b{\theta}_l}(\b{x}_i)}^{T_{\b{\theta}_l}(\b{x}_j)} \|\b{\theta}'_l(t)\|_2\, dt \Big)^2,
\end{align}
where $\b{\theta}'_l(t)$ is derivative of $\b{\theta}_l(t)$ w.r.t. $t$ and $\alpha_l := ( \int_{0}^{1} \|\b{\theta}'_l(t)\|_2\, dt )^2$ is the scale factor. 
The Curvilinear Distance Metric Learning (CDML) \cite{chen2019curvilinear} uses this approximation of distance metric by the above curvy geodesic on manifold, i.e., Eq. (\ref{equation_metric_int_arc_length_2}). The optimization problem in CDML is:
\begin{equation}\label{equation_CDML_optimization}
\begin{aligned}
& \underset{\b{\Theta}}{\text{minimize}}
& & \frac{1}{n} \sum_{i=1}^n \mathcal{L}(\|\b{x}_i - \b{x}_j\|_{\b{W}}^2; y_{ij}) + \lambda \Omega(\b{\Theta}),
\end{aligned}
\end{equation}
where $n$ is the number of points, $\b{\Theta} := [\b{\theta}_1, \dots, \b{\theta}_p]$, $y_{ij}=1$ if $(\b{x}_i, \b{x}_j) \in \mathcal{S}$ and $y_{ij}=0$ if $(\b{x}_i, \b{x}_j) \in \mathcal{D}$, $\|\b{x}_i - \b{x}_j\|_{\b{W}}^2$ is defined in Eq. (\ref{equation_metric_int_arc_length_2}), $\lambda>0$ is the regularization parameter, $\mathcal{L}(.)$ is some loss function, and $\Omega(\b{\Theta})$ is some penalty term. 
The optimal $\b{\Theta}$, obtained from Eq. (\ref{equation_CDML_optimization}), can be used in Eq. (\ref{equation_metric_int_arc_length_2}) to have the optimal distance metric. 
A recent follow-up of CDML is \cite{zhang2021curvilinear}.

\subsection{Adversarial Metric Learning (AML)}

Adversarial Metric Learning (AML) \cite{chen2018adversarial} uses adversarial learning \cite{goodfellow2014generative,ghojogh2021generative} for metric learning. 
On one hand, we have a distinguishment stage which tries to discriminate the dissimilar points and push similar points close to one another. 
On the other hand, we have an confusion or adversarial stage which tries to fool the metric learning method by pulling the dissimilar points close to each other and pushing the similar points away. The distinguishment and confusion stages are trained simultaneously and they make each other stronger gradually.  

From the dataset, we form random pairs $\mathcal{X} := \{(\b{x}_i, \b{x}'_i)\}_{i=1}^{n/2}$. If $\b{x}_i$ and $\b{x}'_i$ are similar points, we set $y_i=1$ and if they are dissimilar, we have $y_i=-1$.
We also generate some random new points in pairs $\mathcal{X}^g := \{(\b{x}_i^g, \b{x}^{g'}_i)\}_{i=1}^{n/2}$. The generated points are updated iteratively by optimization of the confusion stage to fool the metric. 
The loss functions for both stages are Eq. (\ref{equation_GMML_optimization}) used in geometric mean metric learning (see Section \ref{section_geometric_mean_metric_learning}). 

The alternative optimization \cite{ghojogh2021kkt} used in AML is:
\begin{equation*}
\begin{aligned}
& \b{W}^{(t+1)} := \arg \min_{\b{W}} \Big( \sum_{y_i = 1} \|\b{x}_i - \b{x}'_i\|_{\b{W}}^2 \\
& + \sum_{y_i = -1} \|\b{x}_i - \b{x}'_i\|_{\b{W}^{-1}}^2 + \lambda_1 \big( \sum_{y_i = 1} \|\b{x}^{g(t)}_i - \b{x}^{g'(t)}_i\|_{\b{W}}^2 \\
& + \sum_{y_i = -1} \|\b{x}^{g(t)}_i - \b{x}^{g'(t)}_i\|_{\b{W}^{-1}}^2 \big) \Big), 
\end{aligned}
\end{equation*}
\begin{equation}
\begin{aligned}
& {\mathcal{X}}^{g(t+1)} := \arg \min_{\mathcal{X}'} \Big( \sum_{y_i = -1} \|\b{x}_i - \b{x}'_i\|_{\b{W}^{(t+1)}}^2 \\
& + \sum_{y_i = 1} \|\b{x}_i - \b{x}'_i\|_{(\b{W}^{(t+1)})^{-1}}^2 + \lambda_2 \big( \sum_{i=1}^{n/2} \|\b{x}_i - \b{x}^g_i\|_{\b{W}^{(t+1)}}^2 \\
&+ \sum_{i=1}^{n/2} \|\b{x}'_i - \b{x}^{g'}_i\|_{\b{W}^{(t+1)}}^2 \big) \Big),
\end{aligned}
\end{equation}
until convergence, where $\lambda_1, \lambda_2>0$ are the regularization parameters. 
Updating $\b{W}$ and $\mathcal{X}^g$ are the distinguishment and confusion stages, respectively. 
In the distinguishment stage, we find a weight matrix $\b{W}$ to minimize the distances of similar points in both $\mathcal{X}$ and $\mathcal{X}^g$ and maximize the distances of dissimilar points in both $\mathcal{X}$ and $\mathcal{X}^g$. 
In the confusion stage, we generate new points $\mathcal{X}^g$ to adversarially maximize the distances of similar points in $\mathcal{X}$ and adversarially minimize the distances of dissimilar points in $\mathcal{X}$. In this stage, we also make the points $\b{x}_i^g$ and $\b{x}_i^{g'}$ similar to their corresponding points $\b{x}_i$ and $\b{x}'_i$, respectively. 

\section{Probabilistic Metric Learning}\label{section_probabilistic_metric_learning}

Probabilistic methods for metric learning learn the weight matrix in the generalized Mahalanobis distance using probability distributions. They define some probability distribution for each point accepting other points as its neighbors. Of course, the closer points have higher probability for being neighbors. 

\subsection{Collapsing Classes}

One probabilistic method for metric learning is collapsing similar points to the same class while pushing the dissimilar points away from one another \cite{globerson2005metric}. 
The probability distribution between points for being neighbors can be a Gaussian distribution which uses the generalized Mahalanobis distance as its metric. The distribution for $\b{x}_i$ to take $\b{x}_j$ as its neighbor is \cite{goldberger2005neighbourhood}:
\begin{align}\label{equation_P_W_classCollapse}
p^W_{ij} := \frac{\exp(-\|\b{x}_i - \b{x}_j\|_{\b{W}}^2)}{\sum_{k \neq i} \exp(-\|\b{x}_i - \b{x}_k\|_{\b{W}}^2)}, \quad j \neq i,
\end{align}
where we define the normalization factor, also called the partition function, as $Z_i := \sum_{k \neq i} \exp(-\|\b{x}_i - \b{x}_k\|_{\b{W}}^2)$. This factor makes the summation of distribution one. 
Eq. (\ref{equation_P_W_classCollapse}) is a Gaussian distribution whose covariance matrix is $\b{W}^{-1}$ because it is equivalent to:
\begin{align*}
p^W_{ij} := \frac{1}{Z_i} \exp\big(-(\b{x}_i - \b{x}_j)^\top \b{W} (\b{x}_i - \b{x}_j)\big).
\end{align*}

We want the similar points to collapse to the same point after projection onto the subspace of metric (see Proposition \ref{proposition_metric_learning_projection}). Hence, we define the desired neighborhood distribution to be a bi-level distribution \cite{globerson2005metric}:
\begin{align}
p^0_{ij} := 
\left\{
    \begin{array}{ll}
        1 & \mbox{if } (\b{x}_i, \b{x}_j) \in \mathcal{S} \\
        0 & \mbox{if } (\b{x}_i, \b{x}_j) \in \mathcal{D}.
    \end{array}
\right.
\end{align}
This makes all similar points of a group/class a same point after projection. 

\subsubsection{Collapsing Classes in the Input Space}

For making $p^W_{ij}$ close to the desired distribution $p^0_{ij}$, we minimize the KL-divergence between them \cite{globerson2005metric}:
\begin{equation}\label{equation_optimization_collapse_classes}
\begin{aligned}
& \underset{\b{W}}{\text{minimize}}
& & \sum_{i=1}^n \sum_{j=1, j \neq i}^n \text{KL}(p^0_{ij}\, \|\, p^W_{ij}) \\ 
& \text{subject to}
& & \b{W} \succeq \b{0}.
\end{aligned}
\end{equation}

\begin{lemma}[\cite{globerson2005metric}]\label{lemma_classCollapse_gradient}
Let the the objective function in Eq. (\ref{equation_optimization_collapse_classes}) be denoted by $c := \sum_{i=1}^n \sum_{j=1, j \neq i}^n \text{KL}(p^0_{ij}\, \|\, p^W_{ij})$. The gradient of this function w.r.t. $\b{W}$ is:
\begin{align}
&\frac{\partial c}{\partial \b{W}} = \sum_{i=1}^n \sum_{j=1, j \neq i}^n (p^0_{ij} - p^W_{ij}) (\b{x}_i - \b{x}_j) (\b{x}_i - \b{x}_j)^\top. \label{equation_classCollapse_gradient}
\end{align}
\end{lemma}
\begin{proof}
The derivation is similar to the derivation of gradient in Stochastic Neighbor Embedding (SNE) and t-SNE \cite{hinton2003stochastic,maaten2008visualizing,ghojogh2020stochastic}.
Let:
\begin{align}\label{equation_SNE_r}
\mathbb{R} \ni r_{ij} := d_{ij}^2 = ||\b{x}_i - \b{x}_j||_{\b{W}}^2.
\end{align}
By changing $\b{x}_i$, we only have change impact in $d_{ij}$ and $d_{ji}$ (or $r_{ij}$ and $r_{ji}$) for all $j$'s.
According to chain rule, we have:
\begin{align*}
\frac{\partial c}{\partial \b{W}} = \sum_{i,j} \big(\frac{\partial c}{\partial r_{ij}} \frac{\partial r_{ij}}{\partial \b{W}} + \frac{\partial c}{\partial r_{ji}} \frac{\partial r_{ji}}{\partial \b{W}}\big).
\end{align*}
According to Eq. (\ref{equation_SNE_r}), we have:
\begin{align*}
& r_{ij} = ||\b{x}_i - \b{x}_j||_{\b{W}}^2 = \textbf{tr}((\b{x}_i - \b{x}_j)^\top \b{W} (\b{x}_i - \b{x}_j))\\
&~~~~~~~~~~~~~ \overset{(a)}{=} \textbf{tr}((\b{x}_i - \b{x}_j) (\b{x}_i - \b{x}_j)^\top \b{W}) \\
&\implies \frac{\partial r_{ij}}{\partial \b{W}} = (\b{x}_i - \b{x}_j) (\b{x}_i - \b{x}_j)^\top, \\
& r_{ji} = ||\b{x}_j - \b{x}_i||_{\b{W}}^2 = ||\b{x}_i - \b{x}_j||_{\b{W}}^2 = r_{ij} \\
&\implies \frac{\partial r_{ji}}{\partial \b{W}} = (\b{x}_i - \b{x}_j) (\b{x}_i - \b{x}_j)^\top,
\end{align*}
where $(a)$ is because of the cyclic property of trace. 
Therefore:
\begin{align}\label{equation_SNE_deriv_c1_y}
\therefore ~~~~ \frac{\partial c}{\partial \b{W}} = 2 \sum_{i,j} \big(\frac{\partial c}{\partial r_{ij}} \big) (\b{x}_i - \b{x}_j) (\b{x}_i - \b{x}_j)^\top.
\end{align}
The dummy variables in cost function can be re-written as:
\begin{align*}
c &= \sum_{k} \sum_{l\neq k} p_0(l|k) \log (\frac{p_0(l|k)}{p_W(l|k)}) \\
&= \sum_{k \neq l} p_0(l|k) \log (\frac{p_0(l|k)}{p_W(l|k)}) \\
&= \sum_{k \neq l} \big(p_0(l|k) \log (p_0(l|k)) - p_0(l|k) \log (p_W(l|k)) \big),
\end{align*}
whose first term is a constant with respect to $p_W(l|k)$ and thus to $\b{W}$. We have:
\begin{align*}
\mathbb{R} \ni \frac{\partial c}{\partial r_{ij}} = - \sum_{k \neq l} p_0(l|k) \frac{\partial (\log (p_W(l|k)))}{\partial r_{ij}}.
\end{align*}
According to Eqs. (\ref{equation_P_W_classCollapse}) and (\ref{equation_SNE_r}), the $p_W(l|k)$ is:
\begin{align*}
p_W(l|k) := \frac{\exp(-d_{kl}^2)}{\sum_{k \neq f}\exp(-d_{kf}^2)} = \frac{\exp(-r_{kl})}{\sum_{k \neq f}\exp(-r_{kf})}.
\end{align*}
We take the denominator of $p_W(l|k)$ as:
\begin{align}\label{equation_SNE_beta}
\beta := \sum_{k \neq f} \exp(- d_{kf}^2) = \sum_{k \neq f} \exp(- r_{kf}).
\end{align}
We have $\log (p_W(l|k)) = \log (p_W(l|k)) + \log \beta - \log \beta = \log (p_W(l|k)\, \beta) - \log \beta$. Therefore:
\begin{align*}
&\therefore ~~~ \frac{\partial c}{\partial r_{ij}} = - \sum_{k \neq l} p_0(l|k) \frac{\partial \big(\log (p_W(l|k) \beta) - \log \beta\big)}{\partial r_{ij}} \\
&= - \sum_{k \neq l} p_0(l|k) \bigg[\frac{\partial \big(\log (p_W(l|k) \beta)\big)}{\partial r_{ij}} - \frac{\partial \big(\log \beta\big)}{\partial r_{ij}}\bigg] \\
&= - \sum_{k \neq l} p_0(l|k) \bigg[\frac{1}{p_W(l|k) \beta}\frac{\partial \big( p_W(l|k) \beta\big)}{\partial r_{ij}} - \frac{1}{\beta}\frac{\partial \beta}{\partial r_{ij}}\bigg].
\end{align*}
The $p_W(l|k) \beta$ is:
\begin{align*}
p_W(l|k) \beta &= \frac{\exp(-r_{kl})}{\sum_{f \neq k}\exp(-r_{kf})} \times \sum_{k \neq f} \exp(- r_{kf}) \\
&= \exp(-r_{kl}).
\end{align*}
Therefore, we have:
\begin{align*}
&\therefore ~~~ \frac{\partial c}{\partial r_{ij}} = \\
& - \sum_{k \neq l} p_0(l|k) \bigg[\frac{1}{p_W(l|k) \beta}\frac{\partial \big( \exp(-r_{kl}) \big)}{\partial r_{ij}} - \frac{1}{\beta}\frac{\partial \beta}{\partial r_{ij}}\bigg].
\end{align*}
The $\partial \big( \exp(-r_{kl}) \big)/\partial r_{ij}$ is non-zero for only $k=i$ and $l=j$; therefore:
\begin{align*}
\frac{\partial \big( \exp(-r_{ij}) \big)}{\partial r_{ij}} &= - \exp(-r_{ij}), \\
\frac{\partial \beta}{\partial r_{ij}} &= \frac{\partial \sum_{k \neq f} \exp(- r_{kf})}{\partial r_{ij}} = \frac{\partial \exp(- r_{ij})}{\partial r_{ij}} \\
&= - \exp(- r_{ij}).
\end{align*}
Therefore:
\begin{align*}
&\therefore ~~~ \frac{\partial c}{\partial r_{ij}} = \\
&- \bigg( p^0_{ij} \Big[\frac{-1}{p^W_{ij} \beta} \exp(-r_{ij})\Big] + 0 + \dots + 0 \bigg) \\
& - \sum_{k \neq l} p_0(l|k) \Big[\frac{1}{\beta} \exp(- r_{ij}) \Big].
\end{align*}
We have $\sum_{k \neq l} p_0(l|k) = 1$ because summation of all possible probabilities is one. Thus:
\begin{align}
\frac{\partial c}{\partial r_{ij}} &= -  p^0_{ij} \Big[\frac{-1}{p^W_{ij} \beta} \exp(-r_{ij})\Big] - \Big[\frac{1}{\beta} \exp(- r_{ij}) \Big] \nonumber \\
&= \underbrace{\frac{\exp(- r_{ij})}{\beta}}_{=p^W_{ij}} \Big[\frac{p^0_{ij}}{p^W_{ij}} - 1\Big] = p^0_{ij}  - p^W_{ij}. \label{equation_SNE_derivative_r_ij}
\end{align}
Substituting the obtained derivative in Eq. (\ref{equation_SNE_deriv_c1_y}) gives Eq. (\ref{equation_classCollapse_gradient}). Q.E.D.
\end{proof}

The optimization problem (\ref{equation_optimization_collapse_classes}) is convex; hence, it has a unique solution. We can solve it using any optimization method such as the projected gradient method, where after every gradient descent step, we project the solution onto the positive semi-definite cone \cite{ghojogh2021kkt}:
\begin{align*}
& \b{W} := \b{W} - \eta \frac{\partial c}{\partial \b{W}}, \\
& \b{W} := \b{V}\, \textbf{diag}(\max(\lambda_1, 0), \dots, \max(\lambda_d, 0))\, \b{V}^\top,
\end{align*}
where $\eta>0$ is the learning rate and $\b{V}$ and $\b{\Lambda} = \textbf{diag}(\lambda_1, \dots, \lambda_d)$ are the eigenvectors and eigenvalues of $\b{W}$, respectively (see Eq. (\ref{equation_W_U_UT})). 

\subsubsection{Collapsing Classes in the Feature Space}

According to Eq. (\ref{equation_Mahalanobis_distance_in_RKHS}), the distance in the feature space can be stated using kernels as $\|\b{k}_i - \b{k}_j\|_{\b{T}\b{T}^\top}^2$ where $\b{k}_i \in \mathbb{R}^n$ is the kernel vector between dataset $\b{X}$ and the point $\b{x}_i$. 
We define $\b{R} := \b{T}\b{T}^\top \in \mathbb{R}^{n \times n}$.
Hence, in the feature space, Eq. (\ref{equation_P_W_classCollapse}) becomes:
\begin{align}
p^R_{ij} := \frac{\exp(-\|\b{k}_i - \b{k}_j\|_{\b{R}}^2)}{\sum_{k \neq i} \exp(-\|\b{k}_i - \b{k}_k\|_{\b{R}}^2)}, \quad j \neq i.
\end{align}
The gradient in Eq. (\ref{equation_classCollapse_gradient}) becomes:
\begin{align}
&\frac{\partial c}{\partial \b{R}} = \sum_{i=1}^n \sum_{j=1, j \neq i}^n (p^0_{ij} - p^R_{ij}) (\b{k}_i - \b{k}_j) (\b{k}_i - \b{k}_j)^\top. \label{equation_classCollapse_gradient_kernel}
\end{align}
Again, we can find the optimal $\b{R}$ using projected gradient method. This gives us the optimal metric for collapsing classes in the feature space \cite{globerson2005metric}. Note that we can also regularize the objective function, using the trace operator or Frobenius norm, for avoiding overfitting.

\subsection{Neighborhood Component Analysis Methods}

Neighborhood Component Analysis (NCA) is one of the most well-known probabilistic metric learning methods. In the following, we introduce different variants of NCA. 

\subsubsection{Neighborhood Component Analysis (NCA)}\label{section_NCA_spectral}

In the original NCA \cite{goldberger2005neighbourhood}, the probability that $\b{x}_j$ takes $\b{x}_i$ as its neighbor is as in Eq. (\ref{equation_P_W_classCollapse}), where we assume $p^W_{ii} = 0$ by convention:
\begin{align}\label{equation_P_W_NCA}
p^W_{ij} := 
\left\{
    \begin{array}{ll}
        \frac{\exp(-\|\b{x}_i - \b{x}_j\|_{\b{W}}^2)}{\sum_{k \neq i} \exp(-\|\b{x}_i - \b{x}_k\|_{\b{W}}^2)} & \mbox{if } j \neq i \\
        0 & \mbox{if } j = i.
    \end{array}
\right.
\end{align}
Consider the decomposition of the weight matrix of metric as in Eq. (\ref{equation_W_U_UT}), i.e., $\b{W} = \b{U} \b{U}^\top$. 
Let $\mathcal{S}_i$ denote the set of similar points to $\b{x}_i$ where $(\b{x}_i, \b{x}_j) \in \mathcal{S}$. 
The optimization problem of NCA is to find a $\b{U}$ to maximize this probability distribution for similar points \cite{goldberger2005neighbourhood}:
\begin{equation}\label{equation_optimization_NCA}
\begin{aligned}
& \underset{\b{U}}{\text{maximize}}
& & \sum_{(\b{x}_i, \b{x}_j) \in \mathcal{S}} p^W_{ij} = \sum_{i=1}^n \sum_{\b{x}_j \in \mathcal{S}_i} p^W_{ij} = \sum_{i=1}^n p^W_i,
\end{aligned}
\end{equation}
where:
\begin{align}
p^W_i := \sum_{\b{x}_j \in \mathcal{S}_i} p^W_{ij}.
\end{align}
Note that the required constraint $\b{W} \succeq \b{0}$ is already satisfied because of the decomposition in Eq. (\ref{equation_P_W_classCollapse}). 

\begin{lemma}[\cite{goldberger2005neighbourhood}]
Suppose the objective function of Eq. (\ref{equation_optimization_NCA}) is denoted by $c$. The gradient of this cost function w.r.t. $\b{U}$ is:
\begin{equation}\label{equation_NCA_gradient}
\begin{aligned}
\frac{\partial c}{\partial \b{U}} = &\,2 \sum_{i=1}^n \Big( p^W_i \sum_{k=1}^n p^W_{ik} (\b{x}_i - \b{x}_k) (\b{x}_i - \b{x}_k)^\top \\
&- \sum_{\b{x}_j \in \mathcal{S}_i} p^W_{ij} (\b{x}_i - \b{x}_j) (\b{x}_i - \b{x}_j)^\top \Big) \b{U}.
\end{aligned}
\end{equation}
\end{lemma}
The derivation of this gradient is similar to the approach in the proof of Lemma \ref{lemma_classCollapse_gradient}. 
We can use gradient ascent for solving the optimization.

Another approach is to maximize the log-likelihood of neighborhood probability \cite{goldberger2005neighbourhood}:
\begin{equation}\label{equation_optimization_NCA_2}
\begin{aligned}
& \underset{\b{U}}{\text{maximize}}
& & \sum_{i=1}^n \log\Big( \sum_{\b{x}_j \in \mathcal{S}_i} p^W_{ij} \Big),
\end{aligned}
\end{equation}
whose gradient is \cite{goldberger2005neighbourhood}:
\begin{equation}
\begin{aligned}
\frac{\partial c}{\partial \b{U}} = &\,2 \sum_{i=1}^n \Big( \sum_{k=1}^n p^W_{ik} (\b{x}_i - \b{x}_k) (\b{x}_i - \b{x}_k)^\top \\
& - \frac{\sum_{\b{x}_j \in \mathcal{S}_i} p^W_{ij} (\b{x}_i - \b{x}_j) (\b{x}_i - \b{x}_j)^\top}{\sum_{\b{x}_j \in \mathcal{S}_i} p^W_{ij}} \Big) \b{U}.
\end{aligned}
\end{equation}
Again, gradient ascent can give us the optimal $\b{U}$. 
As explained in Proposition \ref{proposition_metric_learning_projection}, the subspace is metric is the column space of $\b{U}$ and projection of points onto this subspace reduces the dimensionality of data. 

\subsubsection{Regularized Neighborhood Component Analysis}

It is shown by some experiments that NCA can overfit to training data for high-dimensional data \cite{yang2007regularized}. 
Hence, we can regularize it to avoid overfitting. In regularized NCA \cite{yang2007regularized}, we use the log-posterior of the matrix $\b{U}$ which is equal to:
\begin{align}\label{equation_regularized_NCA_posterior}
\mathbb{P}(\b{U}|\b{x}_i, \mathcal{S}_i) = \frac{\mathbb{P}(\b{x}_i, \mathcal{S}_i|\b{U})\, \mathbb{P}(\b{U})}{\mathbb{P}(\b{x}_i, \mathcal{S}_i)},
\end{align}
according to the Bayes' rule. 
We can use Gaussian distribution for the prior:
\begin{align}\label{equation_regularized_NCA_prior}
\mathbb{P}(\b{U}) = \prod_{k=1}^d \prod_{l=1}^d c\, \exp(-\lambda (\b{U}(k,l))^2),
\end{align}
where $c>0$ is a constant factor including the normalization factor, $\lambda>0$ is the inverse of variance, and $\b{U}(k,l)$ is the $(k,l)$-th element of $\b{U} \in \mathbb{R}^{d \times d}$. Note that we can have $\b{U} \in \mathbb{R}^{d \times p}$ if we truncate it to have $p$ leading eigenvectors of $\b{W}$ (see Eq. (\ref{equation_W_U_UT})). 
The likelihood 
\begin{align}\label{equation_regularized_NCA_likelihood}
\mathbb{P}(\b{x}_i, \mathcal{S}_i|\b{U}) \propto \exp\Big( \sum_{(\b{x}_i, \b{x}_j) \in \mathcal{S}} p^W_{ij} \Big).
\end{align}
The regularized NCA maximizes the log-posterior \cite{yang2007regularized}:
\begin{equation*}
\begin{aligned}
&\log \mathbb{P}(\b{U}|\b{x}_i, \mathcal{S}_i) \overset{(\ref{equation_regularized_NCA_posterior})}{=} \log \mathbb{P}(\b{x}_i, \mathcal{S}_i|\b{U}) + \log \mathbb{P}(\b{U}) \\
&- \underbrace{\log \mathbb{P}(\b{x}_i, \mathcal{S}_i)}_{\text{constant w.r.t. } \b{U}} \overset{(a)}{=} \sum_{(\b{x}_i, \b{x}_j) \in \mathcal{S}} p^W_{ij} - \lambda \|\b{U}\|_F^2,
\end{aligned}
\end{equation*}
where $(a)$ is because of Eqs. (\ref{equation_regularized_NCA_prior}) and (\ref{equation_regularized_NCA_likelihood}) and $\|.\|_F$ denotes the Frobenius norm. 
Hence, the optimization problem of regularized NCA is \cite{yang2007regularized}:
\begin{equation}\label{equation_optimization_regularized_NCA}
\begin{aligned}
& \underset{\b{U}}{\text{maximize}}
& & \sum_{(\b{x}_i, \b{x}_j) \in \mathcal{S}} p^W_{ij} - \lambda \|\b{U}\|_F^2,
\end{aligned}
\end{equation}
where $\lambda>0$ can be seen as the regularization parameter. 
The gradient is similar to Eq. (\ref{equation_NCA_gradient}) but plus the derivative of the regularization term which is $-2\lambda \b{U}$. 

\subsubsection{Fast Neighborhood Component Analysis}\label{section_fast_NCA}

\textbf{-- Fast NCA:}
The fast NCA \cite{yang2012fast} accelerates NCA by using $k$-Nearest Neighbors ($k$NN) rather than using all points for computing the neighborhood distribution of every point. Let $\mathcal{N}_i$ and $\mathcal{M}_i$ denote the $k$NN of $\b{x}_i$ among the similar points to $\b{x}_i$ (denoted by $\mathcal{S}_i$) and dissimilar points (denoted by $\mathcal{D}_i$), respectively. 
Fast NCA uses following probability distribution for $\b{x}_i$ to take $\b{x}_i$ as its neighbor \cite{yang2012fast}:
\begin{align}\label{equation_P_W_fastNCA}
&p^W_{ij} := \left\{
    \begin{array}{ll}
        \frac{\exp(-\|\b{x}_i - \b{x}_j\|_{\b{W}})}{\sum_{\b{x}_k \in \mathcal{N}_i \cup \mathcal{M}_i} \exp(-\|\b{x}_i - \b{x}_k\|_{\b{W}})} & \mbox{if } \b{x}_k \in \mathcal{N}_i \cup \mathcal{M}_i \\
        0 & \mbox{otherwise.}
    \end{array}
\right.
\end{align}
The optimization problem of fast NCA is similar to Eq. (\ref{equation_optimization_regularized_NCA}):
\begin{equation}\label{equation_optimization_fast_NCA}
\begin{aligned}
& \underset{\b{U}}{\text{maximize}}
& & \sum_{i=1}^n \sum_{\b{x}_j \in \mathcal{M}_i} p^W_{ij} - \lambda \|\b{U}\|_F^2,
\end{aligned}
\end{equation}
where $p^W_{ij}$ is Eq. (\ref{equation_P_W_fastNCA}) and $\b{U}$ is the matrix in the decomposition of $\b{W}$ (see Eq. (\ref{equation_W_U_UT})). 

\begin{lemma}[\cite{yang2012fast}]
Suppose the objective function of Eq. (\ref{equation_optimization_fast_NCA}) is denoted by $c$. The gradient of this cost function w.r.t. $\b{U}$ is:
\begin{equation}\label{equation_fastNCA_gradient}
\begin{aligned}
&\frac{\partial c}{\partial \b{U}} = \sum_{i=1}^n \Big( p^W_i \sum_{\b{x}_k \in \mathcal{N}_i} p^W_{ik} (\b{x}_i - \b{x}_k) (\b{x}_i - \b{x}_k)^\top \\
&+ (p^W_i-1) \sum_{\b{x}_j \in \mathcal{M}_i} p^W_{ij} (\b{x}_i - \b{x}_j) (\b{x}_i - \b{x}_j)^\top \Big) \b{U} - 2\lambda \b{U}.
\end{aligned}
\end{equation}
\end{lemma}
This is similar to Eq. (\ref{equation_NCA_gradient}).
See \cite{yang2012fast} for the derivation.
We can use gradient ascent for solving the optimization. 

\textbf{-- Kernel Fast NCA:}
According to Eq. (\ref{equation_Mahalanobis_distance_in_RKHS}), the distance in the feature space is $\|\b{k}_i - \b{k}_j\|_{\b{T}\b{T}^\top}^2$ where $\b{k}_i \in \mathbb{R}^n$ is the kernel vector between dataset $\b{X}$ and the point $\b{x}_i$. 
We can use this distance metric in Eq. (\ref{equation_P_W_fastNCA}) to have kernel fast NCA \cite{yang2012fast}. 
Hence, the gradient of kernel fast NCA is similar to Eq. (\ref{equation_fastNCA_gradient}):
\begin{equation}\label{equation_kernel_fastNCA_gradient}
\begin{aligned}
&\frac{\partial c}{\partial \b{T}} = \sum_{i=1}^n \Big( p^W_i \sum_{\b{x}_k \in \mathcal{N}_i} p^W_{ik} (\b{k}_i - \b{k}_k) (\b{k}_i - \b{k}_k)^\top \\
&+ (p^W_i-1) \sum_{\b{x}_j \in \mathcal{M}_i} p^W_{ij} (\b{k}_i - \b{k}_j) (\b{k}_i - \b{k}_j)^\top \Big) \b{T} - 2\lambda \b{T}.
\end{aligned}
\end{equation}
Again, we can find the optimal $\b{T}$ using gradient ascent. 
Note that the same technique can be used to kernelize the original NCA.

\subsection{Bayesian Metric Learning Methods}

In this section, we introduce the Bayesian metric learning methods which use variational inference \cite{ghojogh2021factor} for metric learning. 
In Bayesian metric learning, we learn a distribution for the distance metric between every two points; we sample the pairwise distances from these learned distributions. 

First, we provide some definition required in these methods. 
According to Eq. (\ref{equation_W_U_UT}), we can decompose the weight matrix in the metric using the eigenvalue decomposition. Accordingly, we can approximate this matrix by:
\begin{align}\label{equation_Bayesian_NCA_W_decomposition}
\b{W} \approx \b{V}_x \b{\Lambda} \b{V}_x^\top,
\end{align}
where $\b{V}_x$ contains the eigenvectors of $\b{X}\b{X}^\top$ and $\b{\Lambda} = \textbf{diag}([\lambda_1, \dots, \lambda_d]^\top)$ is the diagonal matrix of eigenvalues which we learn in Bayesian metric learning. 
Let $X$ and $Y$ denote the random variables for data and labels, respectively, and let $\b{\lambda} = [\lambda_1, \dots, \lambda_d]^\top \in \mathbb{R}^d$ denote the learnable eigenvalues. 
Let $\b{v}_x^l \in \mathbb{R}^d$ denote the $l$-th column of $\b{V}_x$. We define $\b{w}_{ij} = [w_{ij}^1, \dots, w_{ij}^d]^\top := [((\b{v}_x^1)^\top (\b{x}_i - \b{x}_j))^2, \dots, ((\b{v}_x^d)^\top (\b{x}_i - \b{x}_j))^2]^\top \in \mathbb{R}^d$. 
The reader should not confuse $\b{w}_{ij}$ with $\b{W}$ which is the weight matrix of metric in out notations.

\subsubsection{Bayesian Metric Learning Using Sigmoid Function}

One of the Bayesian metric learning methods is \cite{yang2007bayesian}. 
We define:
\begin{align}\label{equation_y_ij}
y_{ij} := 
\left\{
    \begin{array}{ll}
        1 & \mbox{if } (\b{x}_i, \b{x}_j) \in \mathcal{S} \\
        -1 & \mbox{if } (\b{x}_i, \b{x}_j) \in \mathcal{D}.
    \end{array}
\right.
\end{align}
We can consider a sigmoid function for the likelihood \cite{yang2007bayesian}:
\begin{align}
\mathbb{P}(Y|X,\b{\Lambda}) = \frac{1}{1 + \exp(y_{ij} (\sum_{l=1}^d \lambda_l w_{ij}^l - \mu))},
\end{align}
where $\mu>0$ is a threshold. 
We can also derive an evidence lower bound for $\mathbb{P}(\mathcal{S}, \mathcal{D})$; we do not provide the derivation for brevity (see \cite{yang2007bayesian} for derivation of the lower bound). As in the variational inference, we maximize this lower bound for likelihood maximization \cite{ghojogh2021factor}. 
We assume a Gaussian distribution with mean $\b{m}_{\lambda} \in \mathbb{R}^d$ and covariance $\b{V}_{\lambda} \in \mathbb{R}^{d \times d}$ for the distribution $\mathbb{P}(\b{\lambda})$. 
By maximizing the lower bound, we can estimate these parameters as \cite{yang2007bayesian}:
\begin{align}
& \b{V}_T := \Big(\delta \b{I} + 2\sum_{(\b{x}_i, \b{x}_j) \in \mathcal{S}} \frac{\tanh(\xi_{ij}^s)}{4 \xi_{ij}^s} \b{w}_{ij} \b{w}_{ij}^\top \nonumber \\
&~~~~~~~~~~~~~ + 2\sum_{(\b{x}_i, \b{x}_j) \in \mathcal{D}} \frac{\tanh(\xi_{ij}^d)}{4 \xi_{ij}^d} \b{w}_{ij} \b{w}_{ij}^\top \Big)^{-1}, \label{equation_Bayesian_ML_V}\\
& \b{m}_T := \b{V}_T \Big(\delta \b{\gamma}_0 - \frac{1}{2} \sum_{(\b{x}_i, \b{x}_j) \in \mathcal{S}} \b{w}_{ij} + \frac{1}{2} \sum_{(\b{x}_i, \b{x}_j) \in \mathcal{D}} \b{w}_{ij} \Big), \label{equation_Bayesian_ML_m}
\end{align}
where $\delta>0$ and $\b{\gamma}_0$ are hyper-parameters related to the priors on the weight matrix of metric and the threshold. We define the following variational parameter \cite{yang2007bayesian}:
\begin{align}
& \xi_{ij}^s := \sqrt{(\b{m}_T^\top \b{w}_{ij})^2 + \b{w}_{ij}^\top \b{V}_T \b{w}_{ij}}, \label{equation_Bayesian_ML_xi}
\end{align}
for $(\b{x}_i, \b{x}_j) \in \mathcal{S}$. We similarly define the variational parameter $\xi_{ij}^d$ for $(\b{x}_i, \b{x}_j) \in \mathcal{D}$. 
The variables $\b{V}_T$, $\b{m}_T$, $\xi_{ij}^s$, and $\xi_{ij}^d$ are updated iteratively by Eqs. (\ref{equation_Bayesian_NCA_V}), (\ref{equation_Bayesian_NCA_m}), and (\ref{equation_Bayesian_ML_xi}), respectively, until convergence.
After these parameters are learned, we can sample the eigenvalues from the posterior, $\b{\lambda} \sim \mathcal{N}(\b{m}_T, \b{V}_T)$. These eigenvalues can be used in Eq. (\ref{equation_Bayesian_NCA_W_decomposition}) to obtain the weight matrix in the metric.
Note that Bayesian metric learning can also be used for active learning (see \cite{yang2007bayesian} for details). 

\subsubsection{Bayesian Neighborhood Component Analysis}

Bayesian NCA \cite{wang2017bayesian} using variational inference \cite{ghojogh2021factor} in the NCA formulation. 
If $\mathcal{N}_{im}$ denotes the dataset index of the $m$-th nearest neighbor of $\b{x}_i$, we define $\b{W}_i^j := [w_{ij} - w_{i \mathcal{N}_{i1}}, \dots, w_{ij} - w_{i \mathcal{N}_{ik}}] \in \mathbb{R}^{d \times k}$. 
As in the variational inference \cite{ghojogh2021factor}, we consider an evidence lower-bound on the log-likelihood:
\begin{align*}
\log(\mathbb{P}(Y|X,\b{\Lambda})) > \sum_{i=1}^n \sum_{\b{x}_j \in \mathcal{N}_i} \Big( &-\frac{1}{2} \b{\lambda}^\top \b{W}_i^j \b{H} (\b{W}_i^j)^\top \b{\lambda} \\
&+ \b{b}_{ij}^\top (\b{W}_i^j)^\top \b{\lambda} - c_{ij}\Big),
\end{align*}
where $\mathcal{N}_i$ was defined before in Section \ref{section_fast_NCA}, $\b{H} := \frac{1}{2} (\b{I} - \frac{1}{k+1} \b{1 }\b{1}^\top) \in \mathbb{R}^{k \times k}$ is the centering matrix, and:
\begin{align}
& \mathbb{R}^k \ni \b{b}_{ij} := \b{H} \b{\psi}_{ij} \nonumber\\
&- \exp\Bigg(\b{\psi}_{ij} - \log\Big(1 + \sum_{\b{x}_t \in \mathcal{N}_i} \exp\big( (\b{w}_{ij} - \b{w}_{it})^\top \b{\lambda} \big)\Big)\Bigg), \label{equation_Bayesian_NCA_b}
\end{align}
in which $\b{\psi}_{ij} \in \mathbb{R}^k$ is the learnable variational parameter. 
See \cite{wang2017bayesian} for the derivation of this lower-bound. The sketch of this derivation is using Eq. (\ref{equation_P_W_classCollapse}) but for the $k$NN among the similar points, i.e., $\mathcal{N}_i$. Then, the lower-bound is obtained by a logarithm inequality as well as the Bohning's quadratic bound \cite{murphy2012machine}. 

We assume a Gaussian distribution for the prior of $\b{\lambda}$ with mean $\b{m}_0 \in \mathbb{R}^d$ and covariance $\b{V}_0 \in \mathbb{R}^{d \times d}$. This prior is assumed to be known. 
Likewise, we assume a Gaussian distribution with mean $\b{m}_T \in \mathbb{R}^d$ and covariance $\b{V}_T \in \mathbb{R}^{d \times d}$ for the posterior $\mathbb{P}(X,\b{\Lambda}|Y)$. 
Using Bayes' rule and the above lower-bound on the likelihood, we can estimate these parameters as \cite{wang2017bayesian}:
\begin{align}
& \b{V}_T := \Big(\b{V}_0^{-1} + \sum_{i=1}^n \sum_{\b{x}_j \in \mathcal{N}_i} \b{W}_i^j \b{H} (\b{W}_i^j)^\top \Big)^{-1}, \label{equation_Bayesian_NCA_V}\\
& \b{m}_T := \b{V}_T \Big(\b{V}_0^{-1} \b{m}_0 + \sum_{i=1}^n \sum_{\b{x}_j \in \mathcal{N}_i} \b{W}_i^j \b{b}_{ij} \Big). \label{equation_Bayesian_NCA_m}
\end{align}
The variational parameter can also be obtained by \cite{wang2017bayesian}:
\begin{align}\label{equation_Bayesian_NCA_psi}
\b{\psi}_{ij} := (\b{W}_i^j)^\top \b{m}_T. 
\end{align}
The variables $\b{b}_{ij}$, $\b{V}_T$, $\b{m}_T$, and $\b{\psi}_{ij}$ are updated iteratively by Eqs. (\ref{equation_Bayesian_NCA_b}), (\ref{equation_Bayesian_NCA_V}), (\ref{equation_Bayesian_NCA_m}), and (\ref{equation_Bayesian_NCA_psi}), respectively, until convergence.

After these parameters are learned, we can sample the eigenvalues from the posterior, $\b{\lambda} \sim \mathcal{N}(\b{m}_T, \b{V}_T)$. These eigenvalues can be used in Eq. (\ref{equation_Bayesian_NCA_W_decomposition}) to obtain the weight matrix in the metric. Alternatively, we can directly sample the distance metric from the following distribution:
\begin{align}
\|\b{x}_i - \b{x}_j\|_{\b{W}}^2 \sim \mathcal{N}(\b{w}_{ij}^\top \b{m}_T, \b{w}_{ij}^\top \b{V}_T \b{w}_{ij}).
\end{align}

\subsubsection{Local Distance Metric (LDM)}

Let the set of similar and dissimilar points for the point $\b{x}_i$ be denoted by $\mathcal{S}_i$ and $\mathcal{D}_i$, respectively. 
In Local Distance Metric (LDM) \cite{yang2006efficient}, we consider the following for the likelihood: 
\begin{align}
\mathbb{P}(y_i|\b{x}_i) = &\sum_{\b{x}_j \in \mathcal{S}_i} \exp(-\|\b{x}_i - \b{x}_j\|_{\b{W}}^2) \nonumber\\
&\times \Big(\sum_{\b{x}_j \in \mathcal{S}_i} \exp(-\|\b{x}_i - \b{x}_j\|_{\b{W}}^2) \nonumber\\
&~~~~~ + \sum_{\b{x}_j \in \mathcal{D}_i} \exp(-\|\b{x}_i - \b{x}_j\|_{\b{W}}^2)\Big)^{-1}.
\end{align}
If we consider Eq. (\ref{equation_Bayesian_NCA_W_decomposition}) for decomposition of the weight matrix, the log-likelihood becomes: 
\begin{align*}
&\sum_{i=1}^n \log (\mathbb{P}(y_i|\b{x}_i,\b{\Lambda})) = \\
&~~~~~~~~~~~~ \sum_{i=1}^n \log \Big(\sum_{\b{x}_j \in \mathcal{S}_i} \exp\big(\!-\sum_{l=1}^d \lambda_l w_{ij}^l\big)\Big) \\
&~~~~~~~~~~~~\sum_{i=1}^n \log \Big(\sum_{\b{x}_j \in \mathcal{S}_i} \exp\big(\!-\sum_{l=1}^d \lambda_l w_{ij}^l\big) \\
&~~~~~~~~~~~~~~~~~~~~~~~~~~~+ \sum_{\b{x}_j \in \mathcal{D}_i} \exp\big(\!-\sum_{l=1}^d \lambda_l w_{ij}^l\big)\Big).
\end{align*}
We want to maximize this log-likelihood for learning the variables $\{\lambda_1, \dots, \lambda_d\}$. 
An evidence lower bound on this log-likelihood can be \cite{yang2006efficient}:
\begin{equation}\label{equation_LDM_lower_bound}
\begin{aligned}
\sum_{i=1}^n \log &(\mathbb{P}(y_i|\b{x}_i,\b{\Lambda})) \geq \\
&\sum_{i=1}^n \sum_{\b{x}_j \in \mathcal{S}_i} \phi_{ij} \sum_{l=1}^d \lambda_l w_{ij}^l \\
& - \sum_{i=1}^n \log \Big(\sum_{\b{x}_j \in \mathcal{S}_i} \exp\big(\!-\sum_{l=1}^d \lambda_l w_{ij}^l\big) \\
&~~~~~~~~~~~~ + \sum_{\b{x}_j \in \mathcal{D}_i} \exp\big(\!-\sum_{l=1}^d \lambda_l w_{ij}^l\big)\Big),
\end{aligned}
\end{equation}
where $\phi_{ij}$ is the variational parameter which is:
\begin{equation}\label{equation_LDM_phi}
\begin{aligned}
& \phi_{ij} := \frac{\exp\big(\!-\sum_{l=1}^d \lambda_l w_{ij}^l\big)}{\sum_{\b{x}_j \in \mathcal{S}_i} \exp\big(\!-\sum_{l=1}^d \lambda_l w_{ij}^l\big)} \times \\
&~~~~~~~~~~~~ \Big(1 + \frac{\exp\big(\!-\sum_{l=1}^d \lambda_l w_{ij}^l\big)}{\sum_{\b{x}_j \in \mathcal{S}_i} \exp\big(\!-\sum_{l=1}^d \lambda_l w_{ij}^l\big)}\Big)^{-1}.
\end{aligned}
\end{equation}
See \cite{yang2006efficient} for derivation of the lower bound. 
Iteratively, we maximize the lower bound, i.e. Eq. (\ref{equation_LDM_lower_bound}), and update $\phi_{ij}$ by Eq. (\ref{equation_LDM_phi}). 
The learned parameters $\{\lambda_1, \dots, \lambda_d\}$ can be used in Eq. (\ref{equation_Bayesian_NCA_W_decomposition}) to obtain the weight matrix in the metric. 

\subsection{Information Theoretic Metric Learning}

There exist information theoretic approaches for metric learning where KL-divergence (relative entropy) or mutual information is used. 

\subsubsection{Information Theoretic Metric Learning with a Prior Weight Matrix}

One of the information theoretic methods for metric learning is using a prior weight matrix \cite{davis2007information} where we consider a known weight matrix $\b{W}_0$ as the regularizer and try to minimize the KL-divergence between the distributions with $\b{W}$ and $\b{W}_0$: 
\begin{align}
\text{KL}(p_{ij}^{W_0} \| p_{ij}^W) := \sum_{i=1}^n \sum_{j=1}^n p_{ij}^{W_0} \log\Big(\frac{p_{ij}^{W_0}}{p_{ij}^{W}}\Big).
\end{align}
There are both offline and online approaches for metric learning using batch and streaming data, respectively. 

\textbf{-- Offline Information Theoretic Metric Learning:}
We consider a Gaussian distribution, i.e. Eq. (\ref{equation_P_W_classCollapse}), for the probability of $\b{x}_i$ taking $\b{x}_j$ as its neighbor, i.e. $p_{ij}^W$.
While we make the weight matrix similar to the prior weight matrix through KL-divergence, we find a weight matrix which makes all the distances of similar points less than an upper bound $u>0$ and all the distances of dissimilar points larger than a lower bound $l$ (where $l>u$).
Note that, for Gaussian distributions, the KL divergence is related to the LogDet $D_{ld}(.,.)$ between covariance matrices \cite{dhillon2007differential}; hence, we can say:
\begin{align*}
&\text{KL}(p_{ij}^{W_0} \| p_{ij}^W) = \frac{1}{2} D_{ld}(\b{W}_0^{-1}, \b{W}^{-1}) = \frac{1}{2} D_{ld}(\b{W}, \b{W}_0) \\
&~~~~~~~~~~~~~~~~~ \overset{(a)}{=} \textbf{tr}(\b{W} \b{W}_0^{-1}) - \log (\det(\b{W} \b{W}_0^{-1})) - n,
\end{align*}
where $(a)$ is because of the definition of LogDet.
Hence, the optimization problem can be \cite{davis2007information}:
\begin{equation}\label{equation_optimization_information_theoretic}
\begin{aligned}
& \underset{\b{W}}{\text{minimize}}
& & D_{ld}(\b{W}, \b{W}_0) \\ 
& \text{subject to}
& & \|\b{x}_i - \b{x}_j\|_{\b{W}}^2 \leq u, \quad \forall (\b{x}_i, \b{x}_j) \in \mathcal{S}, \\
& & & \|\b{x}_i - \b{x}_j\|_{\b{W}}^2 \geq l, \quad \forall (\b{x}_i, \b{x}_j) \in \mathcal{D}.
\end{aligned}
\end{equation}

\textbf{-- Online Information Theoretic Metric Learning:}
The online information theoretic metric learning \cite{davis2007information} is suitable for streaming data. For this, we use the offline approach where the known weight matrix $\b{W}_0$ is learned weight matrix by the data which have been received so far. Consider the time slot $t$ where we have been accumulated some data until then and some new data points are received at this time. 
The optimization problem is Eq. (\ref{equation_optimization_information_theoretic}) where $\b{W}_0 = \b{W}_t$ which is the learned weight matrix so far at time $t$.
Note that if there is some label information available, we can incorporate it in the optimization problem as a regularizer. 

\subsubsection{Information Theoretic Metric Learning for Imbalanced Data}

Distance Metric by Balancing KL-divergence (DMBK) \cite{feng2018learning} can be used for imbalanced data where the cardinality of classes are different.
Assume the classes have Gaussian distributions where $\b{\mu}_i \in \mathbb{R}^d$ and $\b{\Sigma}_i \in \mathbb{R}^{d \times d}$ denote the mean and covariance of the $i$-th class. Recall the projection matrix $\b{U}$ in Eq. (\ref{equation_W_U_UT}) and Proposition \ref{proposition_metric_learning_projection}. The KL-divergence between the probabilities of the $i$-th and $j$-th classes after projection onto the subspace of metric is \cite{feng2018learning}:
\begin{equation}
\begin{aligned}
\text{KL}(p_i \| p_j) = &\,\frac{1}{2} \Big(\log\big(\det(\b{U}^\top \b{\Sigma}_j \b{U})\big) \\
&- \log\big(\det(\b{U}^\top \b{\Sigma}_i \b{U})\big) \\
&+ \textbf{tr}\big( (\b{U}^\top \b{\Sigma}_j \b{U})^{-1} \b{U}^\top (\b{\Sigma}_i + \b{D}_{ij}) \b{U} \big) \Big),
\end{aligned}
\end{equation}
where $\b{D}_{ij} := (\b{\mu}_i - \b{\mu}_j) (\b{\mu}_i - \b{\mu}_j)^\top$. 
To cancel the effect of cardinality of classes in imbalanced data, we use the normalized divergence of classes:
\begin{align}
e_{ij} := \frac{n_i n_j \text{KL}(p_i \| p_j)}{\sum_{1 \leq k < l \leq c} n_k n_l \text{KL}(p_k \| p_l)}, 
\end{align}
where $n_i$ and $c$ denote the number of the $i$-th class and the number of classes, respectively. 
We maximize the geometric mean of this divergence between pairs of classes to separate classes after projection onto the subspace of metric. A regularization term is used to increase the distances of dissimilar points and a constraint is used to decrease the similar points \cite{feng2018learning}:
\begin{equation}\label{equation_optimization_information_theoretic_imbalanced}
\begin{aligned}
& \underset{\b{W}}{\text{maximize}}
& & \log\Big(\Big(\prod_{1 \leq i < j \leq c} e_{ij}\Big)^{\frac{1}{c(c-1)}}\Big) \\
& & & ~~~~~~~~~~~ + \lambda \sum_{(\b{x}_i, \b{x}_j) \in \mathcal{D}} \|\b{x}_i - \b{x}_j\|_{\b{W}} \\ 
& \text{subject to}
& & \sum_{(\b{x}_i, \b{x}_j) \in \mathcal{S}} \|\b{x}_i - \b{x}_j\|_{\b{W}}^2 \leq 1, \\
& & & \b{W} \succeq \b{0},
\end{aligned}
\end{equation}
where $\lambda>0$ is the regularization parameter. 
This problem can be solved using projected gradient method \cite{ghojogh2021kkt}. 

\subsubsection{Probabilistic Relevant Component Analysis Methods}

Recall the Relevant Component Analysis (RCA method) \cite{shental2002adjustment} which was introduced in Section \ref{section_RCA}. 
Here, we introduce probabilistic RCA \cite{bar2003learning,bar2005learning} which uses information theory.
Suppose the $n$ data points can be divided into $c$ clusters, or so-called chunklets. Let $\mathcal{X}_l$ denote the data of the $l$-th chunklet and $\b{\mu}_l$ be the mean of $\mathcal{X}_l$. 
Consider Eq. (\ref{equation_W_U_UT}) for decomposition of the weight matrix in the metric where the column-space of $\b{U}$ is the subspace of metric. 
Let projection of data onto this subspace be denoted by $\b{Y} = \b{U}^\top \b{X}$, the projected data in the $l$-th chunklet be $\mathcal{Y}_l$, and $\b{\mu}^y_l$ be the mean of $\mathcal{Y}_l$. 

In probabilistic RCA, we maximize the mutual information between data and the projected data while we want the summation of distances of points in a chunklet from the mean of chunklet is less than a threshold or margin $m>0$.
The mutual information is related to the entropy as $I(X,Y) := H(Y) - H(Y|X)$; hence, we can maximize the entropy of projected data $H(Y)$ rather than the mutual information. Because $\b{Y} = \b{U}^\top \b{X}$, we have $H(Y) \propto \det(\b{U})$. According to Eq. (\ref{equation_W_U_UT}), we have $\det(\b{U}) \propto \det(\b{W})$. 
Hence, the optimization problem can be \cite{bar2003learning,bar2005learning}:
\begin{equation}\label{equation_optimization_probabilistic_RCA}
\begin{aligned}
& \underset{\b{W}}{\text{maximize}}
& & \det(\b{W}) \\
& \text{subject to}
& & \sum_{l=1}^c \sum_{\b{y}_i \in \mathcal{Y}_l} \|\b{y}_i - \b{\mu}^y_l\|_{\b{W}}^2 \leq m, \\
& & & \b{W} \succeq \b{0}.
\end{aligned}
\end{equation}
This preserves the information of data after projection while the inter-chunklet variances are upper-bounded by a margin. 

If we assume Gaussian distribution for each chunklet with the covariance matrix $\b{\Sigma}_l$ for the $l$-th chunklet, we have $\det(\b{W}) \propto \log(\det(\b{U}^\top \b{\Sigma}_l \b{U}))$ because of the quadratic characteristic of covariance. In this case, the optimization problem becomes:
\begin{equation}\label{equation_optimization_probabilistic_RCA_2}
\begin{aligned}
& \underset{\b{U}}{\text{maximize}}
& & \sum_{l=1}^c \log(\det(\b{U}^\top \b{\Sigma}_l \b{U})) \\
& \text{subject to}
& & \sum_{l=1}^c \sum_{\b{y}_i \in \mathcal{Y}_l} \|\b{y}_i - \b{\mu}^y_l\|_{\b{U}\b{U}^\top}^2 \leq m,
\end{aligned}
\end{equation}
where $\b{W} \succeq \b{0}$ is already satisfied because of Eq. (\ref{equation_W_U_UT}). 

\subsubsection{Metric Learning by Information Geometry}

Another information theoretic methods for metric learning is using information geometry in which kernels on data and labels are used \cite{wang2009information}. 
Let $\b{L} \in \mathbb{R}^{c \times n}$ denote the one-hot encoded labels of $n$ data points with $c$ classes and let $\b{X} \in \mathbb{R}^{d \times n}$ be the data points. 
The kernel matrix on the labels is $\b{K}_L = \b{Y}^\top \b{Y} + \lambda \b{I}$ whose main diagonal is strengthened by a small positive number $\lambda$ to have a full rank. 
Recall Proposition \ref{proposition_metric_learning_projection} and Eq. (\ref{equation_W_U_UT}) where $\b{U}$ is the projection matrix onto the subspace of metric. 
The kernel matrix over the projected data, $\b{Y} = \b{U}^\top \b{X}$, is:
\begin{align}
\b{K}_Y = &\b{Y}^\top \b{Y} = (\b{U}^\top \b{X})^\top (\b{U}^\top \b{X}) \nonumber \\
&= \b{X}^\top \b{U} \b{U}^\top \b{X} \overset{(\ref{equation_W_U_UT})}{=} \b{X}^\top \b{W} \b{X}. \label{equation_ML_information_geometry_K_Y}
\end{align}
We can minimize the KL-divergence between the distributions of kernels $\b{K}_Y$ and $\b{K}_L$ \cite{wang2009information}:
\begin{equation}\label{equation_optimization_by_information_geometry}
\begin{aligned}
& \underset{\b{W}}{\text{minimize}}
& & \text{KL}(\b{K}_Y \| \b{K}_L) \\
& \text{subject to}
& & \b{W} \succeq \b{0}.
\end{aligned}
\end{equation}
For simplicity, we assume Gaussian distributions for the kernels. The KL divergence between the distributions of two matrices, $\b{K}_Y \in \mathbb{R}^{n \times n}$ and $\b{K}_L \in \mathbb{R}^{n \times n}$, with Gaussian distributions is simplified to {\citep[Theorem 1]{wang2009information}}:
\begin{align*}
&\text{KL}(\b{K}_Y \| \b{K}_L) = \frac{1}{2} \Big(\textbf{tr}(\b{K}_L^{-1} \b{K}_Y) + \log(\det(\b{K}_L)) \\
&~~~~~~~~~~~~~~~~~~~~~~~~~~~~~~~~~ - \log(\det(\b{K}_Y)) - n\Big) \\
&\overset{(\ref{equation_ML_information_geometry_K_Y})}{\propto} \frac{1}{2} \Big(\textbf{tr}(\b{K}_L^{-1} \b{X}^\top \b{W} \b{X}) + \log(\det(\b{K}_L)) \\
&~~~~~~~~~~~~~ - \log(\det(\b{W})) - n\Big).
\end{align*}
After ignoring the constant terms w.r.t. $\b{W}$, we can restate Eq. (\ref{equation_optimization_by_information_geometry}) to:
\begin{equation}\label{equation_optimization_by_information_geometry_2}
\begin{aligned}
& \underset{\b{W}}{\text{minimize}}
& & \textbf{tr}(\b{K}_L^{-1} \b{X}^\top \b{W} \b{X}) - \log(\det(\b{W})) \\
& \text{subject to}
& & \b{W} \succeq \b{0}.
\end{aligned}
\end{equation}
If we take the derivative of the objective function in Eq. (\ref{equation_optimization_by_information_geometry_2}) and set it to zero, we have:
\begin{align}
& \frac{\partial c}{\partial \b{W}} = \b{X} \b{K}_L^{-1} \b{X}^\top - \b{W}^{-1} \overset{\text{set}}{=} \b{0} \nonumber \\
&\implies \b{W} = (\b{X} \b{K}_L^{-1} \b{X}^\top)^{-1}. \label{equation_optimization_by_information_geometry_solution}
\end{align}
Note that the constraint $\b{W} \succeq \b{0}$ is already satisfied by the solution, i.e., Eq. (\ref{equation_optimization_by_information_geometry_solution}). 

Although this method has used kernels, it can be kernelized further. 
We can also have a kernel version of this method by using Eq. (\ref{equation_Mahalanobis_distance_in_RKHS}) as the generalized Mahalanobis distance in the feature space, where $\b{T}$ (defined in Eq. (\ref{equation_kernelization_representation_theory})) is the projection matrix for the metric. 
Using this in Eqs. (\ref{equation_optimization_by_information_geometry_2}) and (\ref{equation_optimization_by_information_geometry_solution}) can give us the kernel version of this method. 
See \cite{wang2009information} for more information about it. 

\subsection{Empirical Risk Minimization in Metric Learning}

We can learn the metric by minimizing some empirical risk. In the following, some metric learning metric learning methods by risk minimization are introduced. 

\subsubsection{Metric Learning Using the Sigmoid Function}

One of the metric learning methods by risk minimization is \cite{guillaumin2009you}. 
The distribution for $\b{x}_i$ to take $\b{x}_j$ as its neighbor can be stated using a sigmoid function:
\begin{align}\label{equation_P_W_ERM_1}
p^W_{ij} &:= \frac{1}{1 + \exp(\|\b{x}_i - \b{x}_j\|_{\b{W}}^2 - b)},
\end{align}
where $b>0$ is a bias, because close-by points should have larger probability. 
We can maximize and minimize this probability for similar and dissimilar points, respectively:
\begin{equation}\label{equation_optimization_ERM_1}
\begin{aligned}
& \underset{\b{W}}{\text{maximize}}
& & \sum_{i=1}^n \sum_{j=1}^n y_{ij} \log(p^W_{ij}) + (1-y_{ij}) \log(1 - p^W_{ij}) \\ 
& \text{subject to}
& & \b{W} \succeq \b{0},
\end{aligned}
\end{equation}
where $y_{ij}$ is defined in Eq. (\ref{equation_y_ij}).
This can be solved using projected gradient method \cite{ghojogh2021kkt}.
This optimization can be seen as minimization of the empirical risk where close-by points are pushed toward each other and dissimilar points are pushed away to have less error. 



\subsubsection{Pairwise Constrained Component Analysis (PCCA)}

Pairwise Constrained Component Analysis (PCCA) \cite{mignon2012pcca} minimizes the following empirical risk to minimize and maximize the distances of similar points and dissimilar points, respectively:
\begin{equation}\label{equation_optimization_PCCA}
\begin{aligned}
& \underset{\b{U}}{\text{minimize}} \\
&~~~~~~~~~ \sum_{i=1}^n \sum_{j=1}^n \log\Big(1 + \exp\big(y_{ij} (\|\b{x}_i - \b{x}_j\|_{\b{U}\b{U}^\top}^2 - b)\big)\Big),
\end{aligned}
\end{equation}
where $y_{ij}$ is defined in Eq. (\ref{equation_y_ij}), $b>0$ is a bias, $\b{W} \succeq \b{0}$ is already satisfied because of Eq. (\ref{equation_W_U_UT}). 
This can be solved using projected gradient method \cite{ghojogh2021kkt} with the gradient \cite{mignon2012pcca}:
\begin{equation}\label{equation_gradient_PCCA}
\begin{aligned}
\frac{\partial c}{\partial \b{U}} &= 2 \sum_{i=1}^n \sum_{j=1}^n \frac{y_{ij}}{1 + \exp\big(y_{ij} (\|\b{x}_i - \b{x}_j\|_{\b{U}\b{U}^\top}^2 - b)\big)} \\
&~~~~~~~~~~~~ \times (\b{x}_i - \b{x}_j)(\b{x}_i - \b{x}_j)^\top \b{U}.
\end{aligned}
\end{equation}
Note that we can have kernel PCCA by using Eq. (\ref{equation_Mahalanobis_distance_in_RKHS}). In other words, we can replace $\|\b{x}_i - \b{x}_j\|_{\b{U}\b{U}^\top}^2$ and $(\b{x}_i - \b{x}_j)(\b{x}_i - \b{x}_j)^\top \b{U}$ with $\|\b{k}_i - \b{k}_j\|_{\b{T}\b{T}^\top}^2$ and $(\b{k}_i - \b{k}_j)(\b{k}_i - \b{k}_j)^\top \b{T}$, respectively, to have PCCA in the feature space. 

\subsubsection{Metric Learning for Privileged Information}

In some applications, we have a dataset with privileged information where for every point, we have two feature vector; one for the main feature (denoted by $\{\b{x}_i\}_{i=1}^n$) and one for the privileged information (denoted by $\{\b{z}_i\}_{i=1}^n$). 
A metric learning method for using privileged information is \cite{yang2016empirical} where we minimize and maximize the distances of similar and dissimilar points, respectively, for the main features. Simultaneously, we make the distances of privileged features close to the distances of main features. 
Having these two simultaneous goals, we minimize the following empirical risk \cite{yang2016empirical}:
\begin{equation}\label{equation_optimization_ERM_3}
\begin{aligned}
& \underset{\b{W}_1, \b{W}_2}{\text{minimize}} 
& & \sum_{i=1}^n \log\Big(1 + \\
& & &\exp\big(y_{ij}\, (\|\b{x}_i - \b{x}_j\|_{\b{W}_1}^2 - \|\b{z}_i - \b{z}_j\|_{\b{W}_2}^2 ) \big) \Big) \\ 
& \text{subject to}
& & \b{W}_1 \succeq \b{0}, \quad \b{W}_2 \succeq \b{0}.
\end{aligned}
\end{equation}

\section{Deep Metric Learning}\label{section_deep_metric_learning}

We saw in Sections \ref{section_spectral_metric_learning} and \ref{section_probabilistic_metric_learning} that both spectral and probabilistic metric learning methods use the generalized Mahalanobis distance, i.e. Eq. (\ref{equation_generalized_Mahalanobis_distance}), and learn the weight matrix in the metric.
Deep metric learning, however, has a different approach. The methods in deep metric learning usually do not use a generalized Mahalanobis distance but they earn an embedding space using a neural network. 
The network learns a $p$-dimensional embedding space for discriminating classes or the dissimilar points and making the similar points close to each other.
The network embeds data in the embedding space (or subspace) of metric. Then, any distance metric $d(.,.): \mathbb{R}^p \times \mathbb{R}^p \rightarrow \mathbb{R}$ can be used in this embedding space. 
In the loss functions of network, we can use the distance function $d(.,.)$ in the embedding space. For example, an option for the distance function is the squared $\ell_2$ norm or squared Euclidean distance:
\begin{align}
d\big(\textbf{f}(\b{x}_i^1), \textbf{f}(\b{x}_i^2)\big) := \|\textbf{f}(\b{x}_i^1) - \textbf{f}(\b{x}_i^2)\|_2^2,
\end{align}
where $\textbf{f}(\b{x}_i) \in \mathbb{R}^p$ denotes the output of network for the input $\b{x}_i$ as its $p$-dimensional embedding. 
We train the network using mini-batch methods such as the mini-batch stochastic gradient descent and denote the mini-batch size by $b$. The shared weights of sub-networks are denoted by the learnable parameter $\theta$.

\subsection{Reconstruction Autoencoders}

\subsubsection{Types of Autoencoders}

An autoencoder is a model consisting of an encoder $E(.)$ and a decoder $D(.)$. 
There are several types of autoencoders.
All types of autoencoders learn a code layer in the middle of encoder and decoder.
Inferential autoencoders learn a stochastic latent space in the code layer between the encoder and decoder. 
Variational autoencoder \cite{ghojogh2021factor} and adversarial autoencoder \cite{ghojogh2021generative} are two important types of inferential autoencoders. 
Another type of autoencoder is the reconstruction autoencoder consisting of an encoder, transforming data to a code, and a decoder, transforming the code back to the data. Hence, the decoder reconstructs the input data to the encoder. 
The code is a representation for data. 
Each of the encoder and decoder can be multiple layers of neural network with activation functions. 

\subsubsection{Reconstruction Loss}

We denote the input data point to the encoder by $\b{x} \in \mathbb{R}^d$ where $d$ is the dimensionality of data. The reconstructed data point is the output of decoder and is denoted by $\widehat{\b{x}} \in \mathbb{R}^d$.
The representation code, which is the output of encoder and the input of decoder, is denoted by $\textbf{f}(\b{x}) := E(\b{x}) \in \mathbb{R}^p$. 
We have $\widehat{\b{x}} = D(E(\b{x})) = D(\textbf{f}(\b{x}))$. 
If the dimensionality of code is greater than the dimensionality of input data, i.e. $p > d$, the autoencoder is called an over-complete autoencoder \cite{goodfellow2016deep}. Otherwise, if $p < d$, the autoencoder is an under-complete autoencoder \cite{goodfellow2016deep}. 
The loss function of reconstruction autoencoder tries to make the reconstructed data close to the input data:
\begin{equation}\label{equation_reconstruction_autoencoder_loss}
\begin{aligned}
\underset{\theta} {\text{minimize}} ~~~ \sum_{i=1}^b \Big( d\big(\b{x}_i, \widehat{\b{x}}_i\big) + \lambda \Omega(\theta) \Big),
\end{aligned}
\end{equation}
where $\lambda \geq 0$ is the regularization parameter and $\Omega(\theta)$ is some penalty or regularization on the weights. Here, the distance function $d(.,.)$ is defined on $\mathbb{R}^d \times \mathbb{R}^d$. Note that the penalty term can be regularization on the code $\textbf{f}(\b{x}_i)$. 
If the used distance metric is the squared Euclidean distance, this loss is named the regularized Mean Squared Error (MSE) loss. 

\subsubsection{Denoising Autoencoder}

A problem with over-complete autoencoder is that its training only copies each feature of data input to one of the neurons in the code layer and then copies it back to the corresponding feature of output layer. 
This is because the number of neurons in the code layer is greater than the number of neurons in the input and output layers. 
In other words, the networks just memorizes or gets overfit. 
This coping happens by making some of the weights equal to one (or a scale of one depending on the activation functions) and the rest of weights equal to zero. 
To avoid this problem in over-complete autoencoders, one can add some noise to the input data and try to reconstruct the data without noise. For this, Eq. (\ref{equation_reconstruction_autoencoder_loss}) is used while the input to the network is the mini-batch plus some noise. 
This forces the over-complete autoencoder to not just copy data to the code layer. 
This autoencoder can be used for denoising as it reconstructs the data without noise for a noisy input. This network is called the Denoising Autoencoder (DAE) \cite{goodfellow2016deep}. 

\subsubsection{Metric Learning by Reconstruction Autoencoder}

The under-complete reconstruction autoencoder can be used for metric learning and dimensionality reduction, especially when $p \ll d$. 
The loss function for learning a low-dimensional representation code and reconstructing data by the autoencoder is Eq. (\ref{equation_reconstruction_autoencoder_loss}). 
The code layer between the encoder and decoder is the embedding space of metric. 

Note that if the activation functions of all layers are linear, the under-complete autoencoder is reduced to Principal Component Analysis \cite{ghojogh2019unsupervised}. Let $\b{U}_l$ denote the weight matrix of the $l$-th layer of network, $\ell_e$ be the number of layers of encoder, and $\ell_d$ be the number of layers of decoder. With linear activation function, the encoder and decoder are:
\begin{align*}
&\text{encoder: } \quad \mathbb{R}^p \ni \textbf{f}(\b{x}_i) = \underbrace{\b{U}_{\ell_e}^\top \b{U}_{\ell_e-1}^\top \dots \b{U}_{1}^\top}_{\b{U}_e^\top} \b{x}_i, \\
&\text{decoder: } \quad \mathbb{R}^d \ni \widehat{\b{x}}_i = \underbrace{\b{U}_{1} \dots \b{U}_{\ell_d-1} \b{U}_{\ell_d}}_{\b{U}_d} \textbf{f}(\b{x}_i), 
\end{align*}
where linear projection by $\ell$ projection matrices can be replaced by linear projection with one projection matrices $\b{U}_e$ and $\b{U}_d$. 

For learning complicated data patterns, we can use nonlinear activation functions between layers of the encoder and decoder to have nonlinear metric learning and dimensionality reduction. 
It is noteworthy that nonlinear neural network can be seen as an ensemble or concatenation of dimensionality reduction (or feature extraction) and kernel methods. The justification of this claim is as follows. Let the dimensionality for a layer of network be $\b{U} \in \mathbb{R}^{d_1 \times d_2}$ so it connects $d_1$ neurons to $d_2$ neurons. 
Two cases can happen:
\begin{itemize}
\item If $d_1 \geq d_2$, this layer acts as dimensionality reduction or feature extraction because it has reduced the dimensionality of its input data. If this layer has a nonlinear activation function, the dimensionality reduction is nonlinear; otherwise, it is linear. 
\item If $d_1 < d_2$, this layer acts as a kernel method which maps its input data to the high-dimensional feature space in some Reproducing Kernel Hilbert Space (RKHS). This kernelization can help nonlinear separation of some classes which are not separable linearly \cite{ghojogh2021reproducing}. An example use of kernelization in machine learning is kernel support vector machine \cite{vapnik1995nature}. 
\end{itemize}
Therefore, a neural network is a complicated feature extraction method as a concatenation of dimensionality reduction and kernel methods. Each layer of network learns its own features from data. 

\subsection{Supervised Metric Learning by Supervised Loss Functions}


Various loss functions exist for supervised metric learning by neural networks. 
Supervised loss functions can teach the network to separate classes in the embedding space \cite{sikaroudi2020supervision}. 
For this, we use a network whose last layer is for classification of data points. The features of the one-to-last layer can be used for feature embedding. 
The last layer after the embedding features is named the classification layer. 
The structure of this network is shown in Fig. \ref{figure_Supervised_losses_embedding}.
Let the $i$-th point in the mini-batch be denoted by $\b{x}_i \in \mathbb{R}^d$ and its label be denoted by $y_i \in \mathbb{R}$. 
Suppose the network has one output neuron and its output for the input $\b{x}_i$ is denoted by $\textbf{f}_o(\b{x}_i) \in \mathbb{R}$. This output is the estimated class label by the network.
We denote output of the the one-to-last layer by $\textbf{f}(\b{x}_i) \in \mathbb{R}^p$ where $p$ is the number of neurons in that layer which is equivalent to the dimensionality of the embedding space. 
The last layer of network, connecting the $p$ neurons to the output neuron is a fully-connected layer.
The network until the one-to-last layer can be any feed-forward or convolutional network depending on the type of data. If the network is convolutional, it should be flattened at the one-to-last layer. 
The network learns to classify the classes, by the supervised loss functions, so the features of the one-to-last layers will be discriminating features and suitable for embedding. 

\begin{figure}[!t]
\centering
\includegraphics[width=3in]{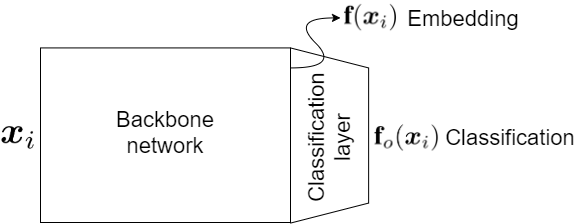}
\caption{The structure of network for metric learning with supervised loss function.}
\label{figure_Supervised_losses_embedding}
\end{figure}

\subsubsection{Mean Squared Error and Mean Absolute Value Losses}

One of the supervised losses is the Mean Squared Error (MSE) which makes the estimated labels close to the true labels using squared $\ell_2$ norm:
\begin{equation}
\begin{aligned}
\underset{\theta} {\text{minimize}} ~~~ \sum_{i=1}^b (\textbf{f}_o(\b{x}_i) - y_i)^2.
\end{aligned}
\end{equation}
One problem with this loss function is exaggerating outliers because of the square but its advantage is its differentiability. 
Another loss function is the Mean Absolute Error (MAE) which makes the estimated labels close to the true labels using $\ell_1$ norm or the absolute value:
\begin{equation}
\begin{aligned}
\underset{\theta} {\text{minimize}} ~~~ \sum_{i=1}^b |\textbf{f}_o(\b{x}_i) - y_i|.
\end{aligned}
\end{equation}
The distance used in this loss is also named the Manhattan distance. 
This loss function does not have the problem of MSE and it can be used for imposing sparsity in the embedding. It is not differentiable at the point $\textbf{f}(\b{x}_i) = y_i$ but as the derivatives are calculated numerically by the neural network, this is not a big issue nowadays.

\subsubsection{Huber and KL-Divergence Losss}

Another loss function is the Huber loss which is a combination of the MSE and MAE to have the advantages of both of them:
\begin{equation}
\begin{aligned}
&\underset{\theta} {\text{minimize}} ~~~ \\
&~~~~\sum_{i=1}^b 
\left\{
    \begin{array}{ll}
        0.5 (\textbf{f}_o(\b{x}_i) - y_i)^2 & \mbox{if } |\textbf{f}_o(\b{x}_i) - y_i| \leq \delta \\
        \delta (|\textbf{f}_o(\b{x}_i) - y_i| - 0.5 \delta) & \mbox{otherwise}.
    \end{array}
\right.
\end{aligned}
\end{equation}
KL-divergence loss function makes the distribution of the estimated labels close to the distribution of the true labels:
\begin{equation}
\begin{aligned}
\underset{\theta} {\text{minimize}} ~~~ \text{KL}(\mathbb{P}(\textbf{f}(\b{x})) \| \mathbb{P}(y)) = \sum_{i=1}^b \textbf{f}(\b{x}_i) \log(\frac{\textbf{f}(\b{x}_i)}{y_i}).
\end{aligned}
\end{equation}

\subsubsection{Hinge Loss}

If there are two classes, i.e. $c=2$, we can have true labels as $y_i \in \{-1, 1\}$. 
In this case, a possible loss function is the Hinge loss:
\begin{equation}
\begin{aligned}
\underset{\theta} {\text{minimize}} ~~~ \sum_{i=1}^b \big[m - y_i\, \textbf{f}_o(\b{x}_i)\big]_+,
\end{aligned}
\end{equation}
where $[\cdot]_+ := \max(\cdot,0)$ and $m>0$ is the margin. 
If the signs of the estimated and true labels are different, the loss is positive which should be minimized. 
If the signs are the same and $|\textbf{f}_o(\b{x}_i)| \geq m$, then the loss function is zero. If the signs are the same but $|\textbf{f}_o(\b{x}_i)| < m$, the loss is positive and should be minimized because the estimation is correct but not with enough margin from the incorrect estimation. 

\subsubsection{Cross-entropy Loss}

For any number of classes, denoted by $c$, we can have a cross-entropy loss. For this loss, we have $c$ neurons, rather than one neuron, at the last layer. In contrast to the MSE, MAE, Huber, and KL-divergence losses which use linear activation function at the last layer, cross-entropy requires softmax or sigmoid activation function at the last layer so the output values are between zero and one. For this loss, we have $c$ outputs, i.e. $\textbf{f}_o(\b{x}_i) \in \mathbb{R}^c$ (continuous values between zero and one), and the true labels are one-hot encoded, i.e., $\b{y}_i \in \{0,1\}^c$. This loss is defined as:
\begin{equation}\label{equation_cross_entropy_loss}
\begin{aligned}
\underset{\theta} {\text{minimize}} ~~~ -\sum_{i=1}^b \sum_{l=1}^c (\b{y}_i)_l \log\big(\textbf{f}_o(\b{x}_i)_l\big),
\end{aligned}
\end{equation}
where $(\b{y}_i)_l$ and $\textbf{f}_o(\b{x}_i)_l$ denote the $l$-th element of $\b{y}_i$ and $\textbf{f}_o(\b{x}_i)$, respectively. 
Minimizing this loss separates classes for classification; this separation of classes also gives us discriminating embedding in the one-to-last layer \cite{sikaroudi2020supervision,boudiaf2020unifying}. 

The reason for why cross-entropy can be suitable for metric learning is theoretically justified in \cite{boudiaf2020unifying}, explained in the following. 
Consider the mutual information between the true labels $Y$ and the estimated labels $\textbf{f}_o(X)$:
\begin{align}
I(\textbf{f}_o(X); Y) &= H(\textbf{f}_o(X)) - H(\textbf{f}_o(X)|Y) \label{equation_mutual_information_true_and_estimated_labels_1} \\
&= H(Y) - H(Y|\textbf{f}_o(X)), \label{equation_mutual_information_true_and_estimated_labels_2}
\end{align}
where $H(.)$ denotes entropy. 
On the one hand, Eq. (\ref{equation_mutual_information_true_and_estimated_labels_1}) has a generative view which exists in the metric learning loss functions generating embedding features. Eq. (\ref{equation_mutual_information_true_and_estimated_labels_2}), one the other hand, has a discriminative view used in the cross-entropy loss function. 
Therefore, the metric learning losses and the cross-entropy loss are related. 
It is shown in {\citep[Proposition 1]{boudiaf2020unifying}} that the cross-entropy is an upper-bound on the metric learning losses so its minimization for classification also provides embedding features. 

It is noteworthy that another supervised loss function is triplet loss, introduced in the next section. Triplet loss can be used for both hard labels (for classification) and soft labels (for similarity and dissimilarity of points). The triplet loss also does not need a last classification layer; therefore, the embedding layer can be the last layer for this loss.


\subsection{Metric Learning by Siamese Networks}\label{section_metric_learning_Siamese}

\begin{figure*}[!t]
\centering
\includegraphics[width=5in]{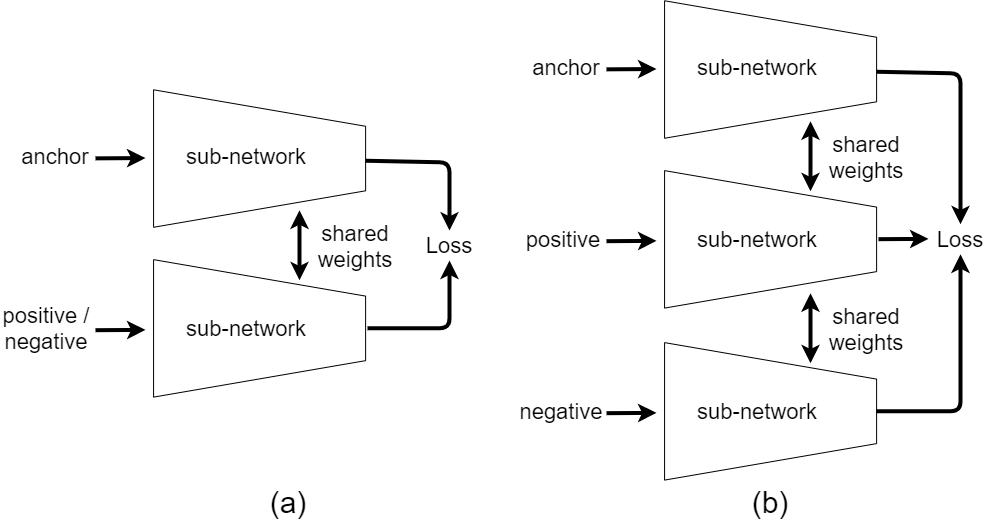}
\caption{The structure of Siamese network with (a) two and (b) three sub-networks.}
\label{figure_Siamese}
\end{figure*}

\subsubsection{Siamese and Triplet Networks}


One of the important deep metric learning methods is Siamese network which is widely used for feature extraction. 
Siamese network, originally proposed in \cite{bromley1993signature}, is a network consisting of several equivalent sub-networks sharing their weights. 
The number of sub-networks in a Siamese network can be any number but it usually is two or three. A Siamese network with three sub-networks is also called a triplet network \cite{hoffer2015deep}. 
The weights of sub-networks in a Siamese network are trained in a way that the intra- and inter-class variances are decreased and increased, respectively. In other words, the similar points are pushed toward each other while the dissimilar points are pulled away from one another. 
Siamese networks have been used in various applications such as computer vision \cite{schroff2015facenet} and natural language processing \cite{yang2020beyond}. 

\subsubsection{Pairs and Triplets of Data Points}

Depending on the number of sub-networks in the Siamese network, we have loss functions for training. 
The loss functions of Siamese networks usually require pairs or triplets of data points. 
Siamese networks do not use the data points one by one but we need to make pairs or triplets of points out of dataset for training a Siamese network. 
For making the pairs or triplets, we consider every data point as the anchor point, denoted by $\b{x}_i^a$. Then, we take one of the similar points to the anchor point as the positive (or neighbor) point, denoted by $\b{x}_i^p$. 
We also take one of the dissimilar points to the anchor point as the negative (or distant) point, denoted by $\b{x}_i^n$. 
If class labels are available, we can use them to find the positive point as one of the points in the same class as the anchor point, and to find the the negative point as one of the points in a different class from the anchor point's class. 
Another approach is to augment the anchor point, using one of the augmentation methods, to obtain a positive points for the anchor point \cite{khodadadeh2019unsupervised,chen2020simple}. 

For Siamese networks with two sub-networks, we make pairs of anchor-positive points $\{(\b{x}_i^a, \b{x}_i^p)\}_{i=1}^{n_t}$ and anchor-negative points $\{(\b{x}_i^a, \b{x}_i^n)\}_{i=1}^{n_t}$, where $n_t$ is the number of pairs. 
For Siamese networks with three sub-networks, we make triplets of anchor-positive-negative points $\{(\b{x}_i^a, \b{x}_i^p, \b{x}_i^n)\}_{i=1}^{n_t}$, where $n_t$ is the number of triplets. 
If we consider every point of dataset as an anchor, the number of pairs/triplets is the same as the number of data points, i.e., $n_t = n$. 

Various loss functions of Siamese networks use pairs or triplets of data points to push the positive point towards the anchor point and pull the negative point away from it. Doing this iteratively for all pairs or triplets will make the intra-class variances smaller and the inter-class variances larger for better discrimination of classes or clusters. 
Later in the following, we introduce some of the loss functions for training a Siamese network. 

\subsubsection{Implementation of Siamese Networks}

A Siamese network with two and three sub-networks is depicted in Fig. \ref{figure_Siamese}.
We denote the output of Siamese network for input $\b{x} \in \mathbb{R}^d$ by $\textbf{f}(\b{x}) \in \mathbb{R}^p$ where $p$ is the dimensionality of embedding (or the number of neurons at the last layer of the network) which is usually much less than the dimensionality of data, i.e., $p \ll d$. 
Note that the sub-networks of a Siamese network can be any fully-connected or convolutional network depending on the type of data. The used network structure for the sub-networks is usually called the backbone network. 

\begin{figure*}[!t]
\centering
\includegraphics[width=5in]{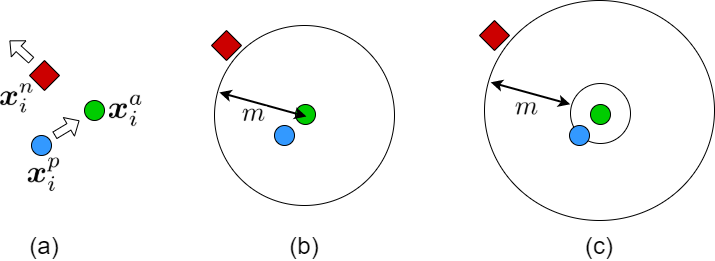}
\caption{Visualization of what contrastive and triplet losses do: (a) a triplet of anchor (green circle), positive (blue circle), and negative (red diamond) points, (b) the effect of contrastive loss making a margin between the anchor and negative point, and (c) the effect of triplet loss making a margin between the positive and negative points.}
\label{figure_triplet_contrastive_losses}
\end{figure*}

The weights of sub-networks are shared in the sense that the values of their weights are equal. Implementation of a Siamese network can be done in two ways:
\begin{enumerate}
\item We can implement several sub-networks in the memory. In the training phase, we feed every data point in the pairs or triplets to one of the sub-networks and take the outputs of sub-networks to have $\textbf{f}(\b{x}_i^a)$, $\textbf{f}(\b{x}_i^p)$, and $\textbf{f}(\b{x}_i^n)$. We use these in the loss function and update the weights of only one of the sub-networks by backpropagation \cite{ghojogh2021kkt}. Then, we copy the updated weights to the other sub-networks. We repeat this for all mini-batches and epochs until convergence. In the test phase, we feed the test point $\b{x}$ to only one of the sub-networks and get the output $\textbf{f}(\b{x})$ as its embedding. 
\item We can implement only one sub-network in the memory. In the training phase, we feed the data points in the pairs or triplets to the sub-network ont by one and take the outputs of sub-network to have $\textbf{f}(\b{x}_i^a)$, $\textbf{f}(\b{x}_i^p)$, and $\textbf{f}(\b{x}_i^n)$. We use these in the loss function and update the weights of the sub-network by backpropagation \cite{ghojogh2021kkt}. We repeat this for all mini-batches and epochs until convergence. In the test phase, we feed the test point $\b{x}$ to the sub-network and get the output $\textbf{f}(\b{x})$ as its embedding. 
\end{enumerate}
The advantage of the first approach is to have all the sub-networks ready and we do not need to feed the points of pairs or triplets one by one. Its disadvantage is using more memory. As the number of points in the pairs or triplets is small (i.e., only two or three), the second approach is more recommended as it is memory-efficient. 

\subsubsection{Contrastive Loss}

One loss function for Siamese networks is the contrastive loss which uses the anchor-positive and anchor-negative pairs of points. 
Suppose, in each mini-batch, we have $b$ pairs of points $\{(\b{x}_i^1, \b{x}_i^2)\}_{i=1}^{b}$ some of which are anchor-positive and some are anchor-negative pairs. The points in an anchor-positive pair are similar, i.e. $(\b{x}_i^1, \b{x}_i^2) \in \mathcal{S}$, and the points in an anchor-negative pair are dissimilar, i.e. $(\b{x}_i^1, \b{x}_i^2) \in \mathcal{D}$, where $\mathcal{S}$ and $\mathcal{D}$ denote the similar and dissimilar sets. 

\hfill\break
\textbf{-- Contrastive Loss:}
We define:
\begin{align}\label{equation_y_i_in_contrastive_loss}
y_i := 
\left\{
    \begin{array}{ll}
        0 & \mbox{if } (\b{x}_i^1, \b{x}_i^2) \in \mathcal{S} \\
        1 & \mbox{if } (\b{x}_i^1, \b{x}_i^2) \in \mathcal{D}.
    \end{array}
\right. \quad \forall i \in \{1, \dots, n_t\}.
\end{align}
The main contrastive loss was proposed in \cite{hadsell2006dimensionality} and is:
\begin{equation}\label{equation_constrastive_loss}
\begin{aligned}
\underset{\theta} {\text{minimize}} ~~~ \sum_{i=1}^b &\Big( (1-y_i) d\big(\textbf{f}(\b{x}_i^1), \textbf{f}(\b{x}_i^2)\big) \\
&+ y_i \big[\! -d\big(\textbf{f}(\b{x}_i^1), \textbf{f}(\b{x}_i^2)\big) + m \big]_+ \Big), 
\end{aligned}
\end{equation}
where $m>0$ is the margin and $[.]_+ := \max(.,0)$ is the standard Hinge loss. 
The first term of loss minimizes the embedding distances of similar points and the second term maximizes the embedding distances of dissimilar points. As shown in Fig. \ref{figure_triplet_contrastive_losses}-b, it tries to make the distances of similar points as small as possible and the distances of dissimilar points at least greater than a margin $m$ (because the term inside the Hinge loss should become close to zero). 

\hfill\break
\textbf{-- Generalized Contrastive Loss:}
The $y_i$, defined in Eq. (\ref{equation_y_i_in_contrastive_loss}), is used in the contrastive loss, i.e., Eq. (\ref{equation_constrastive_loss}).  
This variable is binary and a hard measure of similarity and dissimilarity. Rather than this hard measure, we can have a soft measure of similarity and dissimilarity, denoted by $\psi_i$, which states how similar $\b{x}_i^1$ and $\b{x}_i^2$ are. This measure is between zero (completely similar) and one (completely dissimilar). It can be either given by the dataset as a hand-set measure or can be computed using any similarity measure such as the cosine function:
\begin{align}
[0,1] \ni \psi_i := \frac{1}{2} \big(-\cos(\b{x}_i^1, \b{x}_i^2) + 1\big).
\end{align}
In this case, the pairs $\{(\b{x}_i^1, \b{x}_i^2)\}_{i=1}^{b}$ need not be completely similar or dissimilar points but they can be any two random points from the dataset with some level of similarity/dissimilarity. 
The generalized contrastive loss generalizes the contrastive loss using this soft measure of similarity \cite{leyva2021generalized}:
\begin{equation}\label{equation_constrastive_loss_generalized}
\begin{aligned}
\underset{\theta} {\text{minimize}} ~~~ \sum_{i=1}^b &\Big( (1-\psi_i) d\big(\textbf{f}(\b{x}_i^1), \textbf{f}(\b{x}_i^2)\big) \\
&+ \psi_i \big[\! -d\big(\textbf{f}(\b{x}_i^1), \textbf{f}(\b{x}_i^2)\big) + m \big]_+ \Big).
\end{aligned}
\end{equation}

\subsubsection{Triplet Loss}\label{section_triplet_loss}

One of the losses for Siamese networks with three sub-networks is the triplet loss \cite{schroff2015facenet} which uses the triplets in mini-batches, denoted by $\{(\b{x}_i^a, \b{x}_i^p, \b{x}_i^n)\}_{i=1}^{b}$. It is defined as:
\begin{equation}\label{equation_triplet_loss}
\begin{aligned}
\underset{\theta} {\text{minimize}}\, \sum_{i=1}^b &\Big[ d\big(\textbf{f}(\b{x}_i^a), \textbf{f}(\b{x}_i^p)\big) - d\big(\textbf{f}(\b{x}_i^a), \textbf{f}(\b{x}_i^n)\big) + m \Big]_+, 
\end{aligned}
\end{equation}
where $m>0$ is the margin and $[.]_+ := \max(.,0)$ is the standard Hinge loss. 
As shown in Fig. \ref{figure_triplet_contrastive_losses}-c, because of the used Hinge loss, this loss makes the distances of dissimilar points greater than the distances of similar points by at least a margin $m$; in other words, there will be a distance of at least margin $m$ between the positive and negative points. 
This loss desires to eventually have:
\begin{align}\label{equation_triplet_loss_desire}
d\big(\textbf{f}(\b{x}_i^a), \textbf{f}(\b{x}_i^p)\big) + m \leq d\big(\textbf{f}(\b{x}_i^a), \textbf{f}(\b{x}_i^n)\big),
\end{align}
for all triplets.
The triplet loss is closely related to the cost function for spectral large margin metric learning \cite{weinberger2006distance,weinberger2009distance} (see Section \ref{section_large_margin_metric_learning}). 
It is also noteworthy that using the triplet loss as regularization for cross-entropy loss has been shown to increase robustness of network to some adversarial attacks \cite{mao2019metric}.

\subsubsection{Tuplet Loss}

In triplet loss, i.e. Eq. (\ref{equation_triplet_loss}), we use one positive and one negative point per anchor point. The tuplet loss \cite{sohn2016improved} uses several negative points per anchor point. If $k$ denotes the number of negative points per anchor point and $\b{x}_i^{n,j}$ denotes the $j$-th negative point for $\b{x}_i$, the tuplet loss is \cite{sohn2016improved}:
\begin{equation}\label{equation_tuplet_loss}
\begin{aligned}
\underset{\theta} {\text{minimize}}\, \sum_{i=1}^b \sum_{j=1}^k &\Big[ d\big(\textbf{f}(\b{x}_i^a), \textbf{f}(\b{x}_i^p)\big) \\
&- d\big(\textbf{f}(\b{x}_i^a), \textbf{f}(\b{x}_i^{n,j})\big) + m \Big]_+.
\end{aligned}
\end{equation}
This loss function pushes multiple negative points away from the anchor point simultaneously. 

\subsubsection{Neighborhood Component Analysis Loss}

Neighborhood Component Analysis (NCA) \cite{goldberger2005neighbourhood} was originally proposed as a spectral metric learning method (see Section \ref{section_NCA_spectral}). 
After the success of deep learning, it was used as the loss function of Siamese networks where we minimize the negative log-likelihood using Gaussian distribution or the softmax form within the mini-batch.
Assume we have $c$ classes in every mini-batch. We denote the class index of $\b{x}_i$ by $c(\b{x}_i)$ and the data points of the $j$-th class in the mini-batch by $\mathcal{X}_j$. 
The NCA loss is:
\begin{equation}\label{equation_NCA_loss_Siamese}
\begin{aligned}
&\underset{\theta} {\text{minimize}} ~-\! \sum_{i=1}^{b} \log \Big(\exp\big(\!-\!d\big(\textbf{f}(\b{x}_i^a), \textbf{f}(\b{x}_i^p)\big)\big) \\
&\times \Big[\sum_{j=1, j \neq c(\b{x}_i)}^c \sum_{\b{x}_j^n \in \mathcal{X}_j} \exp\big(\!-\!d\big(\textbf{f}(\b{x}_i^a) - \textbf{f}(\b{x}_j^n)\big)\big)\Big]^{-1}\Big).
\end{aligned}
\end{equation}
The numerator minimizes the distances of similar points and the denominator maximizes the distances of dissimilar points. 



\subsubsection{Proxy Neighborhood Component Analysis Loss}

Computation of terms, especially the normalization factor in the denominator, is time- and memory-consuming in the NCA loss function (see Eq. (\ref{equation_NCA_loss_Siamese})). 
Proxy-NCA loss functions define some proxy points in the embedding space of network and use them in the NCA loss to accelerate computation and make it memory-efficient \cite{movshovitz2017no}. The proxies are representatives of classes in the embedding space and they can be defined in various ways. The simplest way is to define the proxy of every class as the mean of embedded points of that class. Of course, new mini-batches come during training. We can accumulate the embedded points of mini-batches and update the proxies after training the network by every mini-batch.
Another approach for defining proxies is to cluster the embedded points into $c$ clusters (e.g., by K-means) and use the centroid of clusters. 

Let the set of proxies be denotes by $\mathcal{P}$ whose cardinality is the number of classes, i.e., $c$. 
Every embedded point is assigned to one of the proxies by \cite{movshovitz2017no}:
\begin{align}
\Pi(\textbf{f}(\b{x}_i)) := \arg \min_{\b{\pi} \in \mathcal{P}} \|\textbf{f}(\b{x}_i) - \b{\pi}\|_2^2,
\end{align}
or we can assign every point to the proxy of its own class. 
Let $\b{pi}_j$ denote the proxy associated with the $j$-th class. 
The Proxy-NCA loss is the NCA loss, i.e. Eq. (\ref{equation_NCA_loss_Siamese}), but using proxies \cite{movshovitz2017no}:
\begin{equation}\label{equation_proxy_NCA_loss_Siamese}
\begin{aligned}
&\underset{\theta} {\text{minimize}} ~-\! \sum_{i=1}^{b} \log \Big(\exp\big(\!-\!d\big(\textbf{f}(\b{x}_i^a), \Pi(\textbf{f}(\b{x}_i^p))\big)\big) \\
&\times \Big[\sum_{j=1, j \neq c(\b{x}_i)}^c \exp\big(\!-\!d\big(\textbf{f}(\b{x}_i^a) - \b{\pi}_j\big)\big)\Big]^{-1}\Big).
\end{aligned}
\end{equation}
It is shown in \cite{movshovitz2017no} that the Proxy-NCA loss, i.e. Eq. (\ref{equation_proxy_NCA_loss_Siamese}), is an upper-bound on the NCA loss, i.e. Eq. (\ref{equation_NCA_loss_Siamese}); hence, its minimization also achieves the goal of NCA. 
Comparing Eqs. (\ref{equation_NCA_loss_Siamese}) and (\ref{equation_proxy_NCA_loss_Siamese}) shows that Proxy-NCA is faster and more efficient than NCA because it uses only proxies of negative classes rather than using all negative points in the mini-batch. 
Proxy-NCA has also been used in feature extraction from medical images \cite{teh2020learning}. 
It is noteworthy that we can incorporate temperature scaling \cite{hinton2014distilling} in the Proxy-NCA loss. The obtained loss is named Proxy-NCA++ \cite{teh2020proxyncaPlusPlus} and is defined as:
\begin{equation}\label{equation_proxy_NCA_plus_plus_loss_Siamese}
\begin{aligned}
&\underset{\theta} {\text{minimize}} ~-\! \sum_{i=1}^{b} \log \Big(\exp\big(\!-\!d\big(\textbf{f}(\b{x}_i^a), \Pi(\textbf{f}(\b{x}_i^p))\big) \times \frac{1}{\tau} \big) \\
&\times \Big[\sum_{j=1, j \neq c(\b{x}_i)}^c \exp\big(\!-\!d\big(\textbf{f}(\b{x}_i^a) - \b{\pi}_j\big) \times \frac{1}{\tau} \big)\Big]^{-1}\Big),
\end{aligned}
\end{equation}
where $\tau>0$ is the temperature which is a hyper-parameter. 

\subsubsection{Softmax Triplet Loss}

Consider a mini-batch containing points from $c$ classes where $c(\b{x}_i)$ is the class index of $\b{x}_i$ and $\mathcal{X}_j$ denotes the points of the $j$-th class in the mini-batch. 
We can use the softmax function or the Gaussian distribution for the probability that the point $\b{x}_i$ takes $\b{x}_j$ as its neighbor. Similar to Eq. (\ref{equation_P_W_classCollapse}) or Eq. (\ref{equation_NCA_loss_Siamese}), we can have the softmax function used in NCA \cite{goldberger2005neighbourhood}:
\begin{align}
p_{ij} := \frac{\exp\big(\!-\!d\big(\textbf{f}(\b{x}_i), \textbf{f}(\b{x}_j)\big)\big)}{\sum_{k \neq i, k=1}^b \exp\big(\!-\!d\big(\textbf{f}(\b{x}_i), \textbf{f}(\b{x}_k)\big)\big)}, \quad j \neq i.
\end{align}
Another approach for the softmax form is to use inner product in the exponent \cite{ye2019unsupervised}:
\begin{align}\label{equation_softmax_triplet_loss_innerproduct}
p_{ij} := \frac{\exp\big(\textbf{f}(\b{x}_i)^\top \textbf{f}(\b{x}_j)\big)}{\sum_{k = 1, k \neq i}^b \exp\big(\textbf{f}(\b{x}_i)^\top \textbf{f}(\b{x}_k)\big)}, \quad j \neq i.
\end{align}
The loss function for training the network can be the negative log-likelihood which can be called the softmax triplet loss \cite{ye2019unsupervised}:
\begin{equation}
\begin{aligned}
\underset{\theta} {\text{minimize}} ~~~ - \sum_{i=1}^b \Big( &\sum_{\b{x}_j \in \mathcal{X}_{c(\b{x}_i)}} \log(p_{ij}) \\
&- \sum_{\b{x}_j \not \in \mathcal{X}_{c(\b{x}_i)}} \log(1 - p_{ij}) \Big).
\end{aligned}
\end{equation}
This decreases and increases the distances of similar points and dissimilar points, respectively. 

\subsubsection{Triplet Global Loss}

The triplet global loss \cite{kumar2016learning} uses the mean and variance of the anchor-positive pairs and anchor-negative pairs. It is defined as:
\begin{equation}
\begin{aligned}
\underset{\theta} {\text{minimize}} ~~~ (\sigma_p^2 + \sigma_n^2) + \lambda\, [\mu_p - \mu_n + m]_+, 
\end{aligned}
\end{equation}
where $\lambda>0$ is the regularization parameter, $m>0$ is the margin, the means of pairs are:
\begin{align*}
& \mu_p := \frac{1}{b} \sum_{i=1}^b d\big(\textbf{f}(\b{x}_i^a), \textbf{f}(\b{x}_i^p)\big), \\
& \mu_n := \frac{1}{b} \sum_{i=1}^b d\big(\textbf{f}(\b{x}_i^a), \textbf{f}(\b{x}_i^n)\big),
\end{align*}
and the variances of pairs are:
\begin{align*}
& \sigma_p^2 := \frac{1}{b} \sum_{i=1}^b \Big( d\big(\textbf{f}(\b{x}_i^a), \textbf{f}(\b{x}_i^p)\big) - \mu_p \Big)^2, \\
& \sigma_n^2 := \frac{1}{b} \sum_{i=1}^b \Big( d\big(\textbf{f}(\b{x}_i^a), \textbf{f}(\b{x}_i^n)\big) - \mu_n \Big)^2. 
\end{align*}
The first term of this loss minimizes the variances of anchor-positive and anchor-negative pairs. The second term, however, discriminates the anchor-positive pairs from the anchor-negative pairs. Hence, the negative points are separated from the positive points. 

\subsubsection{Angular Loss}

For a triplet $(\b{x}_i^a, \b{x}_i^p, \b{x}_i^n)$, consider a triangle whose vertices are the anchor, positive, and negative points. To satisfy Eq. (\ref{equation_triplet_loss_desire}) in the triplet loss, the angle at the vertex $\b{x}_i^n$ should be small so the edge $d\big(\textbf{f}(\b{x}_i^a), \textbf{f}(\b{x}_i^n)\big)$ becomes larger than the edge $d\big(\textbf{f}(\b{x}_i^a), \textbf{f}(\b{x}_i^p)\big)$. 
Hence, we need to have and upper bound $\alpha>0$ on the angle at the vertex $\b{x}_i^n$. 
If $\b{x}_i^c := (\b{x}_i^a + \b{x}_i^p) / 2$, the angular loss is defined to be \cite{wang2017deep}:
\begin{equation}\label{equation_angular_loss}
\begin{aligned}
&\underset{\theta} {\text{minimize}}\, \\
&\sum_{i=1}^b \Big[ d\big(\textbf{f}(\b{x}_i^a), \textbf{f}(\b{x}_i^p)\big) - 4 \tan^2\!\big(\alpha\, d\big(\textbf{f}(\b{x}_i^a), \textbf{f}(\b{x}_i^c)\big)\big) \Big]_+.
\end{aligned}
\end{equation}
This loss reduces the distance of the anchor and positive points and increases the distance of anchor and $\b{x}_i^c$ and the upper bound $\alpha$. This increases the distance of the anchor and negative points for discrimination of dissimilar points. 

\subsubsection{SoftTriple Loss}

If we normalize the points to have unit length, Eq. (\ref{equation_triplet_loss_desire}) can be restated by using inner products:
\begin{align}\label{equation_triplet_loss_desire_innerProduct}
\textbf{f}(\b{x}_i^a)^\top \textbf{f}(\b{x}_i^p) + m \leq \textbf{f}(\b{x}_i^a)^\top \textbf{f}(\b{x}_i^n), 
\end{align}
whose margin is not exactly equal to the margin in Eq. (\ref{equation_triplet_loss_desire}). 
Consider a Siamese network whose last layer's weights are $\{\b{w}_l \in \mathbb{R}^p\}_{l=1}^c$ where $p$ is the dimensionality of the one-to-last layer and $c$ is the number of classes and the number of output neurons. 
We consider $k$ centers for the embedding of every class; hence, we define $\b{w}_l^j \in \mathbb{R}^p$ as $\b{w}_l$ for its $j$-th center.
It is shown in \cite{qian2019softtriple} that softmax loss results in Eq. (\ref{equation_triplet_loss_desire_innerProduct}). 
Therefore, we can use the SoftTriple loss for training a Siamese network \cite{qian2019softtriple}:
\begin{equation}\label{equation_SoftTriple_loss}
\begin{aligned}
&\underset{\theta} {\text{minimize}} ~-\! \sum_{i=1}^{b} \log \Big( \exp(\lambda(s_{i,y_i} - \delta)) \\
&~~~~~~~~~~ \times \big(\exp(\lambda(s_{i,y_i} - \delta)) + \sum_{l=1, l \neq y_i}^c \exp(\lambda s_{i,l}) \big)^{-1} \Big),
\end{aligned}
\end{equation}
where $\lambda,\delta>0$ are hyper-parameters, $y_i$ is the label of $\b{x}_i$, and:
\begin{align*}
s_{i,l} := \sum_{j=1}^k \frac{\exp\big(\textbf{f}(\b{x}_i)^\top \b{w}_l^j\big)}{\sum_{t=1}^k \exp\big(\textbf{f}(\b{x}_i)^\top \b{w}_l^t\big)} \textbf{f}(\b{x}_i)^\top \b{w}_l^k.
\end{align*}
This loss increases and decreases the intra-class and inter-class distances, respectively. 


\subsubsection{Fisher Siamese Losses}\label{section_Fisher_Siamese_losses}

Fisher Discriminant Analysis (FDA) \cite{fisher1936use,ghojogh2019fisher} decreases the intra-class variance and increases the inter-class variance by maximizing the Fisher criterion. This idea is very similar to the idea of loss functions for Siamese networks. Hence, we can combine the methods of FDA and Siamese loss functions. 

Consider a Siamese network whose last layer is denoted by the projection matrix $\b{U}$. We consider the features of the one-to-last layer in the mini-batch. The covariance matrices of similar points and dissimilar points (one-to-last layer features) in the mini-batch are denoted by $\b{S}_W$ and $\b{S}_B$. These covariances become $\b{U}^\top \b{S}_W \b{U}$ and $\b{U}^\top \b{S}_B \b{U}$, respectively, after the later layer's projection because of the quadratic characteristic of covariance. 
As in FDA, we can maximize the Fisher criterion or equivalently minimize the negative Fisher criterion:
\begin{equation*}
\begin{aligned}
\underset{\b{U}} {\text{minimize}} ~~~ \textbf{tr}(\b{U}^\top \b{S}_W \b{U}) - \textbf{tr}(\b{U}^\top \b{S}_B \b{U}).
\end{aligned}
\end{equation*}
This problem is ill-posed because it increases the total covariance of embedded data to increase the term $\textbf{tr}(\b{U}^\top \b{S}_B \b{U})$. 
Hence, we add minimization of the total covariance as the regularization term:
\begin{equation*}
\begin{aligned}
\underset{\b{U}} {\text{minimize}} ~~~ &\textbf{tr}(\b{U}^\top \b{S}_W \b{U}) - \textbf{tr}(\b{U}^\top \b{S}_B \b{U}) \\
&+ \epsilon \textbf{tr}(\b{U}^\top \b{S}_T \b{U}),
\end{aligned}
\end{equation*}
where $\epsilon \in (0,1)$ is the regularization parameter and $\b{S}_T$ is the covariance of all points of the mini-batch in the one-to-last layer.
The total scatter can be written as the summation of $\b{S}_W$ and $\b{S}_B$; hence:
\begin{align*}
& \textbf{tr}(\b{U}^\top \b{S}_W \b{U}) - \textbf{tr}(\b{U}^\top \b{S}_B \b{U}) + \epsilon \textbf{tr}(\b{U}^\top \b{S}_T \b{U}) \\
&= \textbf{tr}\big(\b{U}^\top (\b{S}_W - \b{S}_W + \epsilon \b{S}_W + \epsilon \b{S}_B) \b{U}\big) \\
&= (2-\lambda) \textbf{tr}(\b{U}^\top \b{S}_W \b{U}) - \lambda \textbf{tr}(\b{U}^\top \b{S}_B \b{U}),
\end{align*}
where $\lambda := 1-\epsilon$. 
Inspired by Eq. (\ref{equation_triplet_loss}), we can have the following loss, named the Fisher discriminant triplet loss \cite{ghojogh2020fisher}:
\begin{equation}\label{equation_Fisher_triplet_loss}
\begin{aligned}
\underset{\theta} {\text{minimize}}~~~~ \, &\Big[ (2-\lambda) \textbf{tr}(\b{U}^\top \b{S}_W \b{U}) \\
&~~~~~~~~ - \lambda \textbf{tr}(\b{U}^\top \b{S}_B \b{U}) + m \Big]_+, 
\end{aligned}
\end{equation}
where $m>0$ is the margin. 
Backpropagating the error of this loss can update both $\b{U}$ and other layers of network. 
Note that the summation over the mini-batch is integrated in the computation of covariance matrices $\b{S}_W$ and $\b{S}_B$. 
Inspired by Eq. (\ref{equation_constrastive_loss}), we can also have the Fisher discriminant contrastive loss \cite{ghojogh2020fisher}:
\begin{equation}\label{equation_Fisher_constrastive_loss}
\begin{aligned}
\underset{\theta} {\text{minimize}} ~~~ & (2-\lambda) \textbf{tr}(\b{U}^\top \b{S}_W \b{U}) \\
&+ \big[\! - \lambda \textbf{tr}(\b{U}^\top \b{S}_B \b{U}) + m \big]_+.
\end{aligned}
\end{equation}
Note that the variable $y_i$ used in the contrastive loss (see Eq. (\ref{equation_y_i_in_contrastive_loss})) is already used in computation of the covariances $\b{S}_W$ and $\b{S}_B$. 
There exist some other loss functions inspired by Fisher discriminant analysis but they are not used for Siamese networks. Those methods will be introduced in Section \ref{section_deep_discriminant_analysis}.

\subsubsection{Deep Adversarial Metric Learning}

In deep adversarial metric learning \cite{duan2018deep}, negative points are generated in an adversarial learning \cite{goodfellow2014generative,ghojogh2021generative}. In this method, we have a generator $G(.)$ which tries to generate negative points fooling the metric learning. Using triplet inputs $\{(\b{x}_i^a, \b{x}_i^p, \b{x}_i^n)\}_{i=1}^b$, the loss function of generator is \cite{duan2018deep}:
\begin{equation}
\begin{aligned}
&\mathcal{L}_G := \sum_{i=1}^b \Big( \|G(\b{x}_i^a, \b{x}_i^p, \b{x}_i^n) - \b{x}_i^a\|_2^2 \\
&~~~~~~~~ + \lambda_1 \|G(\b{x}_i^a, \b{x}_i^p, \b{x}_i^n) - \b{x}_i^n\|_2^2 \\
&~~~~~~~~ + \lambda_2 \big[ d(\textbf{f}(\b{x}_i^a), \textbf{f}(G(\b{x}_i^a, \b{x}_i^p, \b{x}_i^n))) \\
&~~~~~~~~~~~~~~~~~~~~~~~~~~~~~~~~~~~ - d(\textbf{f}(\b{x}_i^a), \textbf{f}(\b{x}_i^p)) + m \big]_+ \Big),
\end{aligned}
\end{equation}
where $\lambda_1, \lambda_2>0$ are the regularization parameters. 
This loss makes the generated negative point close to the real negative point (to be negative) and the anchor point (for fooling metric learning adversarially). The Hinge loss makes the generated negative point different from the anchor and positive points so it also acts like a real negative. 
If $\mathcal{L}_M$ denotes any loss function for Siamese network, such as the triplet loss, the total loss function in deep adversarial metric learning is minimizing $\mathcal{L}_G + \lambda_3 \mathcal{L}_M$ where $\lambda_3>0$ is the regularization parameter \cite{duan2018deep}. 
It is noteworthy that there exists another adversarial metric learning which is not for Siamese networks but for cross-modal data \cite{xu2019deep}. 


\subsubsection{Triplet Mining}\label{section_triplet_mining}

In every mini-batch containing data points from $c$ classes, we can select and use triplets of data points in different ways. For example, we can use all similar and dissimilar points for every anchor point as positive and negative points, respectively. Another approach is to only use some of the similar and dissimilar points within the mini-batch. These approaches for selecting and using triplets are called triplet mining \cite{sikaroudi2020offline}. 
In the following, we review some of the most important triplet mining methods. 
We use triplet mining methods for the triplet loss, i.e., Eq. (\ref{equation_triplet_loss}). 
Suppose $b$ is the mini-batch size, $c(\b{x}_i)$ is the class index of $\b{x}_i$, $\mathcal{X}_j$ denotes the points of the $j$-th class in the mini-batch, and $\mathcal{X}$ denotes the data points in the mini-batch. 

\hfill\break
\textbf{-- Batch-all:}
Batch-all triplet mining \cite{ding2015deep} considers every point in the mini-batch as an anchor point. 
All points in the mini-batch which are in the same class the anchor point are used as positive points. 
All points in the mini-batch which are in a different class from the class of anchor point are used as negative points:
\begin{equation}
\begin{aligned}
&\underset{\theta} {\text{minimize}}\,\,\, \sum_{i=1}^b \sum_{\b{x}_j \in \mathcal{X}_{c(\b{x}_i)}} \sum_{\b{x}_k \in \mathcal{X} \setminus \mathcal{X}_{c(\b{x}_i)}} \Big[ d\big(\textbf{f}(\b{x}_i), \textbf{f}(\b{x}_j)\big) \\
&~~~~~~~~~~~~~~~~~~~~~~~~~~~~~~~~~~~~~~~~~ - d\big(\textbf{f}(\b{x}_i), \textbf{f}(\b{x}_k)\big) + m \Big]_+.
\end{aligned}
\end{equation}
Batch-all mining makes use of all data points in the mini-batch to utilize all available information. 

\hfill\break
\textbf{-- Batch-hard:}
Batch-hard triplet mining \cite{hermans2017defense} considers every point in the mini-batch as an anchor point. 
The hardest positive, which is the farthest point from the anchor point in the same class, is used as the positive point. 
The hardest negative, which is the closest point to the anchor point from another class, is used as the negative point:
\begin{equation}
\begin{aligned}
&\underset{\theta} {\text{minimize}}\,\,\, \sum_{i=1}^b \Big[ \max_{\b{x}_j \in \mathcal{X}_{c(\b{x}_i)}} d\big(\textbf{f}(\b{x}_i), \textbf{f}(\b{x}_j)\big) \\
&~~~~~~~~~~~~~~~~~~~~~~~~~~ - \min_{\b{x}_k \in \mathcal{X} \setminus \mathcal{X}_{c(\b{x}_i)}} d\big(\textbf{f}(\b{x}_i), \textbf{f}(\b{x}_k)\big) + m \Big]_+.
\end{aligned}
\end{equation}
Bath-hard mining uses hardest points so that the network learns the hardest cases. By learning the hardest cases, other cases are expected to be learned properly. Learning the hardest cases can also be justified by the opposition-based learning \cite{tizhoosh2005opposition}.
Batch-hard mining has been used in many applications such as person re-identification \cite{wang2019improved}. 

\hfill\break
\textbf{-- Batch-semi-hard:}
Batch-semi-hard triplet mining \cite{schroff2015facenet} considers every point in the mini-batch as an anchor point. 
All points in the mini-batch which are in the same class the anchor point are used as positive points. 
The hardest negative (closest to the anchor point from another class), which is farther than the positive point, is used as the negative point:
\begin{equation}
\begin{aligned}
&\underset{\theta} {\text{minimize}}\,\,\, \sum_{i=1}^b \sum_{\b{x}_j \in \mathcal{X}_{c(\b{x}_i)}} \Big[ d\big(\textbf{f}(\b{x}_i), \textbf{f}(\b{x}_j)\big) \\
&~~~~~~~~~~~~~ - \min_{\b{x}_k \in \mathcal{X} \setminus \mathcal{X}_{c(\b{x}_i)}} \big\{d\big(\textbf{f}(\b{x}_i), \textbf{f}(\b{x}_k)\big)\, |\, \\
&~~~~~~~~~~~~~~ d\big(\textbf{f}(\b{x}_i), \textbf{f}(\b{x}_k)\big) > d\big(\textbf{f}(\b{x}_i), \textbf{f}(\b{x}_j)\big)\big\} + m \Big]_+.
\end{aligned}
\end{equation}

\textbf{-- Easy-positive:}
Easy-positive triplet mining \cite{xuan2020improved} considers every point in the mini-batch as an anchor point. 
The easiest positive (closest to the anchor point from the same class) is used as the positive point.
All points in the mini-batch which are in a different class from the class of anchor point are used as negative points:
\begin{equation}
\begin{aligned}
&\underset{\theta} {\text{minimize}}\,\,\, \sum_{i=1}^b \sum_{\b{x}_k \in \mathcal{X} \setminus \mathcal{X}_{c(\b{x}_i)}} \Big[ \min_{\b{x}_j \in \mathcal{X}_{c(\b{x}_i)}} d\big(\textbf{f}(\b{x}_i), \textbf{f}(\b{x}_j)\big) \\
&~~~~~~~~~~~~~~~~~~~~~~~~~~~~~~~~~ - d\big(\textbf{f}(\b{x}_i), \textbf{f}(\b{x}_k)\big) + m \Big]_+.
\end{aligned}
\end{equation}
We can use this triplet mining approach in NCA loss function such as in Eq. (\ref{equation_softmax_triplet_loss_innerproduct}). For example, we can have \cite{xuan2020improved}:
\begin{equation}
\begin{aligned}
&\underset{\theta} {\text{minimize}}\,\,\, \sum_{i=1}^b \bigg( \min_{\b{x}_j \in \mathcal{X}_{c(\b{x}_i)}} \exp\big(\textbf{f}(\b{x}_i)^\top \textbf{f}(\b{x}_j)\big) \\
&~~~~~~~ \times \Big( \min_{\b{x}_j \in \mathcal{X}_{c(\b{x}_i)}} \exp\big(\textbf{f}(\b{x}_i)^\top \textbf{f}(\b{x}_j)\big) \\
&~~~~~~~~~~~~ + \sum_{\b{x}_k \in \mathcal{X} \setminus \mathcal{X}_{c(\b{x}_i)}} \exp\big(\textbf{f}(\b{x}_i)^\top \textbf{f}(\b{x}_k)\big) \Big)^{-1} \bigg),
\end{aligned}
\end{equation}
where the embeddings for all points of the mini-batch are normalized to have length one. 


\hfill\break
\textbf{-- Lifted embedding loss:}
The lifted embedding loss \cite{oh2016deep} is related to the anchor-positive distance and the smallest (hardest) anchor-negative distance:
\begin{equation}
\begin{aligned}
\underset{\theta} {\text{minimize}} ~~~ &\sum_{i=1}^b \sum_{\b{x}_j \in \mathcal{X}_{c(\b{x}_i)}} \Big( \Big[
d(\textbf{f}(\b{x}_i), \textbf{f}(\b{x}_j)) \\
& + \max\Big( \max_{\b{x}_k \in \mathcal{X} \setminus \mathcal{X}_{c(\b{x}_i)}} \big\{m - d(\textbf{f}(\b{x}_i), \textbf{f}(\b{x}_k))\big\},  \\
& \max_{\b{x}_l \in \mathcal{X} \setminus \mathcal{X}_{c(\b{x}_j)}} \big\{m - d(\textbf{f}(\b{x}_j), \textbf{f}(\b{x}_l))\big\} \Big)
\Big]_+ \Big)^2, 
\end{aligned}
\end{equation}
This loss is using triplet mining because of using extreme distances. 
Alternatively, another version of this loss function uses logarithm and exponential operators \cite{oh2016deep}:
\begin{equation}
\begin{aligned}
\underset{\theta} {\text{minimize}}  &\sum_{i=1}^b \sum_{\b{x}_j \in \mathcal{X}_{c(\b{x}_i)}} \Big( \Big[
d(\textbf{f}(\b{x}_i), \textbf{f}(\b{x}_j)) \\
& + \log\Big( \sum_{\b{x}_k \in \mathcal{X} \setminus \mathcal{X}_{c(\b{x}_i)}} \exp\big(m - d(\textbf{f}(\b{x}_i), \textbf{f}(\b{x}_k))\big),  \\
& \sum_{\b{x}_l \in \mathcal{X} \setminus \mathcal{X}_{c(\b{x}_j)}} \exp\big(m - d(\textbf{f}(\b{x}_j), \textbf{f}(\b{x}_l))\big) \Big)
\Big]_+ \Big)^2.
\end{aligned}
\end{equation}

\hfill\break
\textbf{-- Hard mining center-triplet loss:}
Let the mini-batch contain data points from $c$ classes. 
Hard mining center–triplet loss \cite{lv2019novel} considers the mean of every class as an anchor point. The hardest (farthest) positive point and the hardest (closest) negative point are used in this loss as \cite{lv2019novel}:
\begin{equation}
\begin{aligned}
&\underset{\theta} {\text{minimize}}\,\,\, \sum_{l=1}^c \Big[ \max_{\b{x}_j \in \mathcal{X}_{c(\bar{\b{x}}^l)}} d\big(\textbf{f}(\bar{\b{x}}^l), \textbf{f}(\b{x}_j)\big) \\
&~~~~~~~~~~~~~~~~~~~~~~~~~~ - \min_{\b{x}_k \in \mathcal{X} \setminus \mathcal{X}_{c(\bar{\b{x}}^l)}} d\big(\textbf{f}(\bar{\b{x}}^l), \textbf{f}(\b{x}_k)\big) + m \Big]_+.
\end{aligned}
\end{equation}
where $\bar{\b{x}}^l$ denotes the mean of the $l$-th class. 

\hfill\break
\textbf{-- Triplet loss with cross-batch memory:}
A version of triplet loss can be \cite{wang2020cross}:
\begin{equation}
\begin{aligned}
&\underset{\theta} {\text{minimize}}\,\,\, \sum_{i=1}^b \bigg( -\sum_{\b{x}_j \in \mathcal{X}_{c(\b{x}_i)}} \textbf{f}(\b{x}_i)^\top \textbf{f}(\b{x}_j) \\
&~~~~~~~~~~~~~~~~~~~~~~ + \sum_{\b{x}_k \in \mathcal{X} \setminus \mathcal{X}_{c(\b{x}_i)}} \textbf{f}(\b{x}_i)^\top \textbf{f}(\b{x}_k) \bigg). 
\end{aligned}
\end{equation}
This triplet loss can use a cross-batch memory where we accumulate a few latest mini-batches.
Every coming mini-batch updates the memory. 
Let the capacity of the memory be $w$ points and the mini-batch size be $b$.
Let $\widetilde{\b{x}}_i$ denote the $i$-th data point in the memory. 
The triplet loss with cross-batch memory is defined as \cite{wang2020cross}:
\begin{equation}
\begin{aligned}
&\underset{\theta} {\text{minimize}}\,\,\, \sum_{i=1}^b \bigg( -\sum_{\widetilde{\b{x}}_j \in \mathcal{X}_{c(\b{x}_i)}} \textbf{f}(\b{x}_i)^\top \textbf{f}(\widetilde{\b{x}}_j) \\
&~~~~~~~~~~~~~~~~~~~~~~ + \sum_{\widetilde{\b{x}}_k \in \mathcal{X} \setminus \mathcal{X}_{c(\b{x}_i)}} \textbf{f}(\b{x}_i)^\top \textbf{f}(\widetilde{\b{x}}_k) \bigg),
\end{aligned}
\end{equation}
which takes the positive and negative points from the memory rather than from the coming mini-batch. 

\subsubsection{Triplet Sampling}

Rather than using the extreme (hardest or easiest) positive and negative points \cite{sikaroudi2020offline}, we can sample positive and negative points from the points in the mini-batch or from some distributions. There are several approaches for the positive and negative points to be sampled \cite{ghojogh2021data}:
\begin{itemize}\setlength\itemsep{0.1em}
\item Sampled by extreme distances of points,
\item Sampled randomly from classes,
\item Sampled by distribution but from existing points,
\item Sampled stochastically from distributions of classes.
\end{itemize}
These approaches are used for triplet sampling. 
The first approach was introduced in Section \ref{section_triplet_mining}. 
The first, second, and third approaches sample the positive and negative points from the set of points in the mini-batch. This type of sampling is called survey sampling \cite{ghojogh2020sampling}. The third and fourth approaches sample points from distributions stochastically.
In the following, we introduce some of the triplet sampling methods. 

\hfill\break
\textbf{-- Distance weighted sampling:}
Distance weighted sampling \cite{wu2017sampling} is a method in the third approach, i.e., sampling by distribution but from existing points.
The distribution of the pairwise distances is proportional to \cite{wu2017sampling}:
\begin{align*}
\mathbb{P}\big(d(\textbf{f}(\b{x}_i), \textbf{f}(\b{x}_j))\big) & \sim \big( d(\textbf{f}(\b{x}_i), \textbf{f}(\b{x}_j))\big)^{p-2} \times \nonumber \\
&\Big(1 - 0.25 \big( d(\textbf{f}(\b{x}_i), \textbf{f}(\b{x}_j))\big)^2\Big)^{(b-3)/2},
\end{align*}
where $b$ is the number of points in the mini-batch and $p$ is the dimensionality of embedding space (i.e., the number of neurons in the last layer of the Siamese network). 
In every mini-batch, we consider every point once as an anchor point. For an anchor point, we consider all points of the mini-batch which are in a different class as candidates for the negative point. 
We sample a negative point, denoted by $\b{x}_*^n$ from these candidates \cite{wu2017sampling}:
\begin{align*}
\b{x}_*^n \sim \min\Big(\lambda, \mathbb{P}^{-1}\big(d(\textbf{f}(\b{x}_i), \textbf{f}(\b{x}_j))\big)\Big), \quad \forall j \neq i,
\end{align*}
where $\lambda>0$ is a hyperparameter to ensure that all candidates have a chance to be chosen. 
This sampling is performed for every mini-batch. 
The loss function in distance weighted sampling is \cite{wu2017sampling}:
\begin{equation}
\begin{aligned}
&\underset{\theta} {\text{minimize}}\,\,\, \sum_{i=1}^b \sum_{\b{x}_j \in \mathcal{X}_{c(\b{x}_i)}} \Big[ d\big(\textbf{f}(\b{x}_i), \textbf{f}(\b{x}_j)\big) \\
&~~~~~~~~~~~~~~~~~~~~~~~~~~~~~~~~~ - d\big(\textbf{f}(\b{x}_i), \textbf{f}(\b{x}_*^n)\big) + m \Big]_+.
\end{aligned}
\end{equation}

\textbf{-- Sampling by Bayesian updating theorem:}
We can sample triplets from distributions of classes which is the forth approach of sampling, mentioned above. One method for this sampling is using the Bayesian updating theorem \cite{sikaroudi2021batch} which is updating the posterior by the Bayes' rule from some new data. 
In this method, we assume $p$-dimensional Gaussian distribution for every class in the embedding space where $p$ is the dimensionality of embedding space. We accumulate the embedded points for every class when the new mini-batches are introduced to the network. The distributions of classes are updated based on both the existing points available so far and the new-coming data points. 
It can be shown that the posterior of mean and covariance of a Gaussian distribution is a normal inverse Wishart distribution \cite{murphy2007conjugate}. 
The mean and covariance of a Gaussian distribution have a generalized Student-t distribution and inverse Wishart distribution, respectively \cite{murphy2007conjugate}. 
Let the so-far available data have sample size $n_0$, mean $\b{\mu}^0$, and covariance $\b{\Sigma}^0$. Also, let the newly coming data have sample size $n'$, mean $\b{\mu}'$, and covariance $\b{\Sigma}'$. 
We update the mean and covariance by expectation of these distributions \cite{sikaroudi2021batch}:
\begin{align*}
& \b{\mu}^{0} \leftarrow \mathbb{E}(\b{\mu}\, |\, \b{x}^{0}) = \frac{n' \b{\mu}' + n_0 \b{\mu}^{0}}{n' + n_0}, \\
& \b{\Sigma}^{0} \leftarrow \mathbb{E}(\b{\Sigma}\, |\, \b{x}^{0}) = \frac{\b{\Upsilon}^{-1}}{n'\!+\! n_0\! -\! p \!-\! 1},~~~ \forall\, n'\! +\! n_0\! >\! p\! +\! 1,
\end{align*}
where:
\begin{align*}
\mathbb{R}^{d \times d} \ni \b{\Upsilon} := &\, n' \b{\Sigma}' + n_0 \b{\Sigma}^0 \\
&+ \frac{n'_1 n_0}{n'_1 + n_0} (\b{\mu}^0 - \b{\mu}') (\b{\mu}^0 - \b{\mu}')^\top.
\end{align*}
The updated mean and covariance are used for Gaussian distributions of the classes. Then, we sample triplets from the distributions of classes rather than from the points of mini-batch. 
We consider every point of the new mini-batch as an anchor point and sample a positive point from the distribution of the same class. We sample $c-1$ negative points from the distributions of $c-1$ other classes. 
If this triplet sampling procedure is used with triplet and contrastive loss functions, the approach is named Bayesian Updating with Triplet loss (BUT) and Bayesian Updating with NCA loss (BUNCA) \cite{sikaroudi2021batch}.

\hfill\break
\textbf{-- Hard negative sampling:}
Let the anchor, positive, and negative points be denoted by $\b{x}^a$, $\b{x}^p$, and $\b{x}^n$, respectively. 
Consider the following distributions for the negative and positive points \cite{robinson2020contrastive}:
\begin{align*}
& \mathbb{P}(\b{x}^n) \propto \alpha \mathbb{P}_n(\b{x}^n) + (1 - \alpha) \mathbb{P}_p(\b{x}^n), \\
& \mathbb{P}_n(\b{x}) \propto \exp\big(\beta \textbf{f}(\b{x}^a)^\top \textbf{f}(\b{x})\big)\, \mathbb{P}(\b{x} | c(\b{x}) \neq c(\b{x}^a)), \\
& \mathbb{P}_p(\b{x}) \propto \exp\big(\beta \textbf{f}(\b{x}^a)^\top \textbf{f}(\b{x})\big)\, \mathbb{P}(\b{x} | c(\b{x}) = c(\b{x}^a)),
\end{align*}
where $\alpha \in (0,1)$ is a hyper-parameter.
The loss function with hard negative sampling is \cite{robinson2020contrastive}:
\begin{equation}
\begin{aligned}
&\underset{\theta} {\text{minimize}} ~-\! \sum_{i=1}^{b} \mathbb{E}_{\b{x}^p \sim \mathbb{P}_p(\b{x})} \log \bigg( \exp\big(\textbf{f}(\b{x}_i^a)^\top \textbf{f}(\b{x}^p)\big) \\
&~~~~~~~~~~ \Big( \exp\big(\textbf{f}(\b{x}_i^a)^\top \textbf{f}(\b{x}^p)\big) \\
&~~~~~~~~~~ + \mathbb{E}_{\b{x}^n \sim \mathbb{P}(\b{x}^n)}\big[ \exp\big(\textbf{f}(\b{x}_i^a)^\top \textbf{f}(\b{x}^n)\big) 
\big] \Big)^{-1} \bigg),
\end{aligned}
\end{equation}
where positive and negative points are sampled from positive and negative distributions defined above. The expectations can be estimated using the Monte Carlo approximation \cite{ghojogh2020sampling}. This time of triplet sampling is a method in the fourth type of triplet sampling, i.e., sampling stochastically from distributions of classes. 

\subsection{Deep Discriminant Analysis Metric Learning}\label{section_deep_discriminant_analysis}

Deep discriminant analysis metric learning methods use the idea of Fisher discriminant analysis \cite{fisher1936use,ghojogh2019fisher} in deep learning, for learning an embedding space which separates classes. 
Some of these methods are deep probabilistic discriminant analysis \cite{li2019discriminant}, discriminant analysis with virtual samples \cite{kim2021virtual}, Fisher Siamese losses \cite{ghojogh2020fisher}, and deep Fisher discriminant analysis \cite{diaz2017deep,diaz2019deep}.
The Fisher Siamese losses were already introduced in Section \ref{section_Fisher_Siamese_losses}.

\subsubsection{Deep Probabilistic Discriminant Analysis}

Deep probabilistic discriminant analysis \cite{li2019discriminant} minimizes the inverse Fisher criterion:
\begin{equation}
\begin{aligned}
&\underset{\theta} {\text{minimize}} ~~~ \frac{\mathbb{E}[\textbf{tr}(\text{cov}(\textbf{f}(\b{x})|y))]}{\textbf{tr}(\text{cov}(\mathbb{E}[\textbf{f}(\b{x})|y]))} = \frac{\sum_{i=1}^b \mathbb{E}[\text{var}(\textbf{f}(\b{x}_i)|y_i)]}{\sum_{i=1}^b \text{var}(\mathbb{E}[\textbf{f}(\b{x}_i)|y_i])} \\
&\overset{(a)}{=} \frac{\sum_{i=1}^b \mathbb{E}[\text{var}(\textbf{f}(\b{x}_i)|y_i)]}{\sum_{i=1}^b \big( \text{var}(\textbf{f}(\b{x}_i)) - \mathbb{E}[\text{var}(\textbf{f}(\b{x}_i)|y_i)] \big)} \\
&\overset{(b)}{=} \frac{\sum_{i=1}^b \sum_{l=1}^c \mathbb{P}(y=l) \text{var}(\textbf{f}(\b{x}_i)|y_i=l)}{\sum_{i=1}^b \big( \text{var}(\textbf{f}(\b{x}_i)) - \sum_{l=1}^c \mathbb{P}(y=l) \text{var}(\textbf{f}(\b{x}_i)|y_i=l) \big)},
\end{aligned}
\end{equation}
where $b$ is the mini-batch size, $c$ is the number of classes, $y_i$ is the class label of $\b{x}_i$, $\text{cov}(.)$ denotes covariance, $\text{var}(.)$ denotes variance, $\mathbb{P}(y=l)$ is the prior of the $l$-th class (estimated by the ratio of class population to the total number of points in the mini-batch), $(a)$ is because of the law of total variance, and $(b)$ is because of the definition of expectation. 
The numerator and denominator represent the intra-class and inter-class variances, respectively. 

\subsubsection{Discriminant Analysis with Virtual Samples}

In discriminant analysis metric learning with virtual samples \cite{kim2021virtual}, we consider any backbone network until the one-to-last layer of neural network and a last layer with linear activation function. Let the outputs of the one-to-last layer be denoted by $\{\textbf{f}'(\b{x}_i)\}_{i=1}^b$ and the weights of the last layer be $\b{U}$. We compute the intra-class scatter $\b{S}_W$ and inter-class scatter $\b{S}_B$ for the one-to-last layer's features $\{\textbf{f}'(\b{x}_i)\}_{i=1}^b$. If we see the last layer as a Fisher discriminant analysis model with projection matrix $\b{U}$, the solution is the eigenvalue problem \cite{ghojogh2019eigenvalue} for $\b{S}_W^{-1} \b{S}_B$. Let $\lambda_j$ denote the $j$-th eigenvalue of this problem. 

Assume $\mathcal{S}_b$ and $\mathcal{D}_b$ denote the similar and dissimilar points in the mini-batch where $|\mathcal{S}_b| = |\mathcal{D}_b| = q$. We define \cite{kim2021virtual}:
\begin{align*}
& \b{g}_p := [\exp(-\textbf{f}'(\b{x}_i)^\top \textbf{f}'(\b{x}_j))\, |\, (\b{x}_i, \b{x}_j) \in \mathcal{S}_b]^\top \in \mathbb{R}^q, \\
& \b{g}_n := [\exp(-\textbf{f}'(\b{x}_i)^\top \textbf{f}'(\b{x}_j))\, |\, (\b{x}_i, \b{x}_j) \in \mathcal{D}_b]^\top \in \mathbb{R}^q, \\
& s_{ctr} := \frac{1}{2q} \sum_{i=1}^q \big( \b{g}_p(i) + \b{g}_n(i) \big),
\end{align*}
where $\b{g}(i)$ is the $i$-th element of $\b{g}$.
We sample $q$ numbers, namely virtual samples, from the uniform distribution $U(s_{ctr} - \epsilon \bar{\lambda}, s_{ctr} + \epsilon \bar{\lambda})$ where $\epsilon$ is a small positive number and $\bar{\lambda}$ is the mean of eigenvalues $\lambda_j$'s. The $q$ virtual samples are put in a vector $\b{r} \in \mathbb{R}^q$. 

The loss function for discriminant analysis with virtual samples is \cite{kim2021virtual}:
\begin{equation}
\begin{aligned}
&\underset{\theta, \b{U}} {\text{minimize}} ~~~ \frac{1}{q} \sum_{i=1}^q \Big[ \frac{1}{q}\, \b{g}_p(i)\, \|\b{r}\|_1 - \frac{1}{q}\, \b{g}_n(i)\, \|\b{r}\|_1 + m \Big]_+ \\
&~~~~~~~~~~~~~~~~~~~~ - 10^{-5} \frac{\textbf{tr}(\b{U}^\top \b{S}_B \b{U})}{\textbf{tr}(\b{U}^\top \b{S}_W \b{U})},
\end{aligned}
\end{equation}
where $\|.\|_1$ is the $\ell_1$ norm, $[.]_+ := \max(.,0)$, $m>0$ is the margin, and the second term is maximization of the Fisher criterion. 

\subsubsection{Deep Fisher Discriminant Analysis}

It is shown in \cite{hart2000pattern} that the solution to the following least squares problem is equivalent to the solution of Fisher discriminant analysis:
\begin{equation}\label{equation_FDA_least_squares}
\begin{aligned}
&\underset{\b{w}_0 \in \mathbb{R}^{c}, \b{W} \in \mathbb{R}^{d \times c}} {\text{minimize}} ~~~ \frac{1}{2} \|\b{Y} - \b{1}_{n \times 1} \b{w}_0^\top - \b{X} \b{W}\|_F^2,
\end{aligned}
\end{equation}
where $\|.\|_F$ is the Frobenius norm, $\b{X} \in \mathbb{R}^{n \times d}$ is the row-wise stack of data points, $\b{Y} := \b{H} \b{E} \b{\Pi}^{-(1/2)} \in \mathbb{R}^{n \times c}$ where $\b{H} := \b{I} - (1/n) \b{1} \b{1}^\top \in \mathbb{R}^{n \times n}$ is the centering matrix, $\b{E} \in \{0,1\}^{n \times c}$ is the one-hot-encoded labels stacked row-wise, $\b{\Pi} \in \mathbb{R}^{c \times c}$ is the diagonal matrix whose $(l,l)$-th element is the cardinality of the $l$-th class. 

Deep Fisher discriminant analysis \cite{diaz2017deep,diaz2019deep} implements Eq. (\ref{equation_FDA_least_squares}) by a nonlinear neural network with loss function:
\begin{equation}
\begin{aligned}
&\underset{\theta}{\text{minimize}} ~~~ \frac{1}{2} \|\b{Y} - \textbf{f}(\b{X}; \theta)\|_F^2, 
\end{aligned}
\end{equation}
where $\theta$ is the weights of network, $\b{X} \in \mathbb{R}^{n \times d}$ denotes the row-wise stack of points in the mini-batch of size $b$, $\b{Y} := \b{H} \b{E} \b{\Pi}^{-(1/2)} \in \mathbb{R}^{b \times c}$ is computed in every mini-batch, and $\textbf{f}(.) \in \mathbb{R}^{b \times c}$ is the row-wise stack of output embeddings of the network. After training, the output $\textbf{f}(\b{x})$ is the embedding for the input point $\b{x}$. 

\subsection{Multi-Modal Deep Metric Learning}

Data has several modals where a separate set of features is available for every modality of data. In other words, we can have several features for every data point. Note that the dimensionality of features may differ. 
Multi-modal deep metric learning \cite{roostaiyan2017multi} addresses this problem in metric learning. 
Let $m$ denote the number of modalities. 
Consider $m$ stacked autoencoders each of which is for one of the modalities. The $l$-th autoencoder gets the $l$-th modality of the $i$-th data point, denoted by $\b{x}_i^l$, and reconstructs it as output, denoted by $\widehat{\b{x}}_i^l$. 
The embedding layer, or the layer between encoder and decoder, is shared between all $m$ autoencoders. We denote the output of this shared embedding layer by $\textbf{f}(\b{x}_i)$. 
The loss function for training the $m$ stacked autoencoders with the shared embedding layer can be \cite{roostaiyan2017multi}:
\begin{equation}
\begin{aligned}
&\underset{\theta} {\text{minimize}} ~~~ \sum_{i=1}^b \sum_{l=1}^m \|\b{x}_i^l - \widehat{\b{x}}_i^l\|_2^2 \\
&+ \lambda_1 \sum_{i=1}^b \sum_{\b{x}_j \in \mathcal{X}_{c(\b{x}_i)}} \big[d(\textbf{f}(\b{x}_i), \textbf{f}(\b{x}_j)) - m_1\big]_+ \\
&+ \lambda_2 \sum_{i=1}^b \sum_{\b{x}_j \in \mathcal{X} \setminus \mathcal{X}_{c(\b{x}_i)}} \big[-d(\textbf{f}(\b{x}_i), \textbf{f}(\b{x}_j)) + m_2\big]_+,
\end{aligned}
\end{equation}
where $\lambda_1, \lambda_2>0$ are the regularization parameters and $m_1, m_2>0$ are the margins. 
The first term is the reconstruction loss and the second and third terms are for metric learning which collapses each class to a margin $m_1$ and discriminates classes by a margin $m_2$. 
This loss function is optimized in a stacked autoencoder setup \cite{hinton2006reducing,wang2014effective}. Then, it is fine-tuned by backpropagation \cite{ghojogh2021restricted}. 
After training, the embedding layer can be used for embedding data points. 
Note that another there exists another multi-modal deep metric learning, which is \cite{xu2019deep}.


\subsection{Geometric Metric Learning by Neural Network}

There exist some works, such as \cite{huang2017riemannian}, \cite{hauser2017principles}, and \cite{hajiabadi2019layered}, which have implemented neural networks on the Riemannian manifolds.
Layered geometric learning \cite{hajiabadi2019layered} implements Geometric Mean Metric Learning (GMML) \cite{zadeh2016geometric} (recall Section \ref{section_geometric_mean_metric_learning}) in a neural network framework. 
In this method, every layer of network is a metric layer which projects the output of its previous layer onto the subspace of its own metric (see Proposition \ref{proposition_metric_learning_projection} and Proposition \ref{equation_metric_learning_projection}). 

For the $l$-th layer of network, we denote the weight matrix (i.e., the projection matrix of metric) and the output of layer for the $i$-th data point by $\b{U}_l$ and $\b{x}_{i,l}$, respectively. Hence, the metric in the $l$-th layer models $\|\b{x}_{i,l} - \b{x}_{j,l}\|_{\b{U}_l\b{U}_l^\top}$. 
Consider the dataset of $n$ points $\b{X} \in \mathbb{R}^{d \times n}$. We denote the output of the $l$-th layer by $\b{X}_l \in \mathbb{R}^{d \times n}$. 
The projection of a layer onto its metric subspace is $\b{X}_l = \b{U}_l^\top \b{X}_{l-1}$. 

Every layer solves the optimization problem of GMML \cite{zadeh2016geometric}, i.e., Eq. (\ref{equation_GMML_optimization}). 
For this, we start from the first layer and proceed to the last layer by feed-propagation. The $l$-th layer computes $\b{\Sigma}_{\mathcal{S}}$ and $\b{\Sigma}_{\mathcal{D}}$ for $\b{X}_{l-1}$ by Eq. (\ref{equation_spectral_ML_first_method_Sigma_S}). Then, the solution of optimization (\ref{equation_GMML_optimization}) is computed which is the Eq. (\ref{equation_GMML_solution_1_1}), i.e., $\b{W}_l = \b{\Sigma}_{\mathcal{S}}^{-1} \sharp_{(1/2)} \b{\Sigma}_{\mathcal{D}} = \b{\Sigma}_{\mathcal{S}}^{(-1/2)} \big(\b{\Sigma}_{\mathcal{S}}^{(1/2)} \b{\Sigma}_{\mathcal{D}} \b{\Sigma}_{\mathcal{S}}^{(1/2)}\big)^{(1/2)} \b{\Sigma}_{\mathcal{S}}^{(-1/2)}$. 
Then, using Eq. (\ref{equation_W_U_UT}), we decompose the obtained $\b{W}_l$ to find $\b{U}_l$. Then, data points are projected onto the metric subspace as $\b{X}_l = \b{U}_l^\top \b{X}_{l-1}$. 

If we want the output of layers lie on the positive semi-definite manifold, the activation function of every layer can be projection onto the positive semi-definite cone \cite{ghojogh2021kkt}:
\begin{align*}
\b{X}_l := \b{V}\, \textbf{diag}(\max(\lambda_1, 0), \dots, \max(\lambda_d, 0))\, \b{V}^\top, 
\end{align*}
where $\b{V}$ and $\{\lambda_1, \dots, \lambda_d\}$ are the eigenvectors and eigenvalues of $\b{X}_l$, respectively. 
This activation function is called the eigenvalue rectification layer in \cite{huang2017riemannian}. 
Finally, it is noteworthy that there is another work, named backprojection \cite{ghojogh2020backprojection}, which has similar idea but in the Euclidean and Hilbert spaces and not in the Riemannian space.


\subsection{Few-shot Metric Learning}

Few-shot learning refers to learning from a few data points rather than from a large enough dataset. 
Few-shot learning is used for domain generalization to be able to use for unseen data in the test phase \cite{wang2020generalizing}. 
The training phase of few-shot learning is episodic where in every iteration or so-called episode of training, we have a support set and a query set. In other words, the training dataset is divided into mini-batches where every mini-batch contains a support set and a query set \cite{triantafillou2020meta}. 
Consider a training dataset with $c_\text{tr}$ classes and a test dataset with $c_\text{te}$ classes. As mentioned before, test and training datasets are usually disjoint in few-shot learning so it is useful for domain generalization. 
In every episode, also called the task or the mini-batch, we train using some (and not all) training classes by randomly sampling from classes.

The support set is $\mathcal{S}_s := \{(\b{x}_{s,i}, y_{s,i})\}_{i=1}^{|\mathcal{S}_s|}$ where $\b{x}$ and $y$ denote the data point and its label, respectively. 
The query set is $\mathcal{S}_q := \{(\b{x}_{q,i}, y_{q,i})\}_{i=1}^{|\mathcal{S}_q|}$. 
The training data of every episode (mini-batch) is the union of the support and query sets. 
At every episode, we randomly sample $c_s$ classes out of the total $c_\text{tr}$ classes of training dataset, where we usually have $c_s \ll c_\text{tr}$.
Then, we sample $k_s$ training data points from these $c_s$ selected classes. These $c_s \times k_s = |\mathcal{S}_s|$ form the support set. This few-shot setup is called $c_s$-way, $k_s$-shot in which the support set contains $c_s$ classes and $k_s$ points in every class. 
The number of classes and every class's points in the query set of every episode may or may not be the same as in the support set. 

In every episode of the training phase of few-shot learning, we update the network weights by back-propagating error using the support set. These updated weights are not finalized yet. We feed the query set to the network with the updated weights and back-propagate error using the query set. This second back-propagation with the query set updates the weights of network finally at the end of episode. In other words, the query set is used to evaluate how good the update by support set are. This learning procedure for few-shot learning is called meta-learning \cite{finn2017model}. 

There are several family of methods for few-shot learning, one of which is some deep metric learning methods. Various metric learning methods have been proposed for learning from few-shot data. For example, Siamese network, introduced in Section \ref{section_metric_learning_Siamese}, has been used for few-shot learning \cite{koch2015siamese,li2020revisiting}. In the following, we introduce two metric learning methods for few-shot learning. 

\subsubsection{Multi-scale Metric Learning}

Multi-scale metric learning \cite{jiang2020multi} learns the embedding space by learning multiple scales of middle features in the training process. 
It has several steps. 
In the first step, we use a pre-trained network with multiple output layers which produce several different scales of features for both the support and query sets. 
In the second step, within every scale of support set, we take average of the $k_s$ features in every class. This gives us $c_s$ features for every scale in the support set. This and the features of the query set are fed to the third step. 
In the third step, we feed every scale to a sub-network where larger scales are fed to sub-networks with more number of layers as they contain more information to process. These sub-networks are concatenated to give a scalar output for every data point with multiple scales of features. Hence, we obtain a scalar score for every data point in the support and query sets. Finally, a combination of a classification loss function, such as the cross-entropy loss (see Eq. (\ref{equation_cross_entropy_loss})), and triplet loss (see Eq. \ref{equation_triplet_loss}) is used in the support-query setup explained before.

\subsubsection{Metric Learning with Continuous Similarity Scores}

Another few-shot metric learning is \cite{xu2019zero} which takes pairs of data points as the input support and query sets. 
For the pair $(\b{x}_i, \b{x}_j)$, consider binary similarity score, $y_{ij}$, defined as:
\begin{align}\label{equation_y_ij_deep}
y_{ij} := 
\left\{
    \begin{array}{ll}
        1 & \mbox{if } (\b{x}_i, \b{x}_j) \in \mathcal{S} \\
        0 & \mbox{if } (\b{x}_i, \b{x}_j) \in \mathcal{D}.
    \end{array}
\right.
\end{align}
where $\mathcal{S}$ and $\mathcal{D}$ denote the sets of similar and dissimilar points, respectively. 
We can define continuous similarity score, $y'_{ij}$, as \cite{xu2019zero}:
\begin{align}\label{equation_y_ij_deep_continuous}
y'_{ij} := 
\left\{
    \begin{array}{ll}
        (\beta-1) d(\b{x}_i, \b{x}_j) + 1 & \mbox{if } (\b{x}_i, \b{x}_j) \in \mathcal{S} \\
        -\alpha d(\b{x}_i, \b{x}_j) + \alpha & \mbox{if } (\b{x}_i, \b{x}_j) \in \mathcal{D},
    \end{array}
\right.
\end{align}
where $0<\alpha<\beta<1$ and $d(\b{x}_i, \b{x}_j)$ is the normalized squared Euclidean distance (we normalize distances within every mini-batch). The ranges of these continuous similarities are:
\begin{align*}
y'_{ij} \in
\left\{
    \begin{array}{ll}
        \, [\beta, 1] & \mbox{if } (\b{x}_i, \b{x}_j) \in \mathcal{S} \\
        \, [0, \alpha] & \mbox{if } (\b{x}_i, \b{x}_j) \in \mathcal{D}.
    \end{array}
\right.
\end{align*}
In every episode (mini-batch), the pairs are fed to a network with several feature vector outputs. For every pair $(\b{x}_i, \b{x}_j)$, these feature vectors are fed to another network which outputs a scalar similarity score $s_{ij}$. 
The loss function of metric learning in this method is \cite{xu2019zero}:
\begin{equation}
\begin{aligned}
& \underset{\theta}{\text{maximize}}
& & \sum_{(\b{x}_i, \b{x}_j) \in \mathcal{X}} (1 + \lambda) (s_{ij} - y'_{ij})^2, \\ 
& \text{subject to}
& & \beta \leq s_{ij}, y'_{ij} \leq 1 \quad \text{ if } \quad y_{ij} = 1, \\
& & & 0 \leq s_{ij}, y'_{ij} \leq \alpha \quad \text{ if } \quad y_{ij} = 0,
\end{aligned}
\end{equation}
where $\lambda>0$ is the regularization parameter and $\mathcal{X}$ is the mini-batch of the support or query set depending on whether it is the phase of support or query. 

\section{Conclusion}\label{section_conclusion}

This was a tutorial and survey on spectral, probabilistic, and deep metric learning. We started with defining distance metric. In spectral methods, we covered methods using scatters of data, methods using Hinge loss, locally linear metric adaptation, kernel methods, geometric methods, and adversarial metric learning. In probabilistic category, we covered collapsing classes, neighborhood component analysis, Bayesian metric learning, information theoretic methods, and empirical risk minimization approaches. In deep learning methods, we explain reconstruction autoencoders, supervised loss functions, Siamese networks, deep discriminant analysis methods, multi-modal learning, geometric deep metric learning, and few-shot metric learning. 




\bibliography{References}
\bibliographystyle{icml2016}

\onecolumn
\tableofcontents


\twocolumn

\end{document}